\documentclass[11pt]{article}
\usepackage[margin=1in]{geometry}
\usepackage{mathpazo}
\usepackage{mathtools}

\usepackage[style=alphabetic,natbib=true,maxcitenames=2,minalphanames=3,minbibnames=12,maxbibnames=12]{biblatex}
\usepackage[utf8]{inputenc} %
\usepackage[T1]{fontenc}    %
\usepackage[colorlinks,linkcolor=blue,citecolor=blue]{hyperref}       %
\usepackage{url}            %
\usepackage{booktabs}       %
\usepackage{amsfonts}       %
\usepackage{nicefrac}       %
\usepackage{microtype}      %
\usepackage{color}         %
\usepackage{wrapfig}
\usepackage{xcolor}
\usepackage{subcaption}
\usepackage{graphicx}
\usepackage{multirow}
\usepackage{adjustbox}
\usepackage{amsmath,amssymb}
\usepackage{mathtools}
\usepackage{algorithm}
\usepackage{algpseudocode}
\usepackage{bbm}
\usepackage{float}
\usepackage{euscript}
\usepackage{amsthm}
\usepackage{cleveref}
\usepackage{thmtools}
\usepackage{thm-restate}
\usepackage{caption}
\usepackage{xspace}
\usepackage{comment}

\usepackage[toc,page,header]{appendix}
\usepackage{minitoc}

\crefname{figure}{Figure}{Figures}
\crefname{theorem}{Theorem}{Theorems}
\crefname{appendix}{Appendix}{Appendices}

\usepackage[normalem]{ulem}
\usepackage{xcolor}

\newtoggle{showcomments}

\newcommand{\roy}[1]{\iftoggle{showcomments}{\textcolor{blue}{[Roy: #1]}}{}}
\newcommand{\sam}[1]{\iftoggle{showcomments}{\textcolor{DarkGreen}{[Sam: #1]}}{}}

\newcommand{\scrcmd}{\textsuperscript}
\newcommand{\scr}[1]{\scrcmd{#1}}

\newtheorem{theorem}{Theorem}

\newtheorem{definition}{Definition}

\usepackage[T1]{fontenc}

\usepackage{amsmath,amsfonts,bm}

\def\1{\bm{1}}

\DeclareMathAlphabet{\mathsfit}{\encodingdefault}{\sfdefault}{m}{sl}
\SetMathAlphabet{\mathsfit}{bold}{\encodingdefault}{\sfdefault}{bx}{n}

\newcommand{\E}{\mathbb{E}}

\newcommand{\R}{\mathbb{R}}

\graphicspath{ {./figs/} }

\newcommand{\algo}{\mathcal{A}}
\newcommand{\graddesc}{\texttt{\graddesc}}

\newcommand{\om}{\textsc{om}\xspace}
\newcommand{\sgd}{\textsc{sgd}\xspace}
\newcommand{\gd}{\textsc{gd}\xspace}

\newcommand{\gda}{\textsc{gda}\xspace}
\newcommand{\omrs}{\textsc{omrs}\xspace}
\newcommand{\dmm}{\textsc{dmm}\xspace}
\newcommand{\trak}{\textsc{trak}\xspace}
\newcommand{\klom}{\texttt{KLoM}\xspace}

\newcommand{\dmdirect}{\textsc{dm-direct}\xspace}

\newcommand{\GD}{\textsc{GD}\xspace}

\newcommand{\ULIRA}{\texttt{U-LiRA}\xspace}
\newcommand{\ulira}{\texttt{U-LiRA}\xspace}
\newcommand{\EfficientULIRA}{Efficient-ULIRA\xspace}

\newcommand{\bbE}[1]{\mathbb{E}\left[#1\right]}

\definecolor{DarkGreen}{rgb}{0.1,0.5,0.1}

\title{Attribute-to-Delete: \\ Machine Unlearning via Datamodel Matching}
\author{Kristian Georgiev\footnote{Equal contribution} $^{1}$, 
Roy Rinberg\footnotemark[1] $^{2,3}$, Sung Min Park\footnotemark[1] $^{4}$\footnote{Work done primarily at MIT EECS.}\ , Shivam Garg\footnotemark[1] $^{5}$\footnote{Work done primarily at Harvard Business School.} \\[.2em]
Andrew Ilyas$^{6}$\footnotemark[2]\ ,
Aleksander M\k{a}dry$^{1}$, \
Seth Neel$^{2}$ \\[1em]
{\normalsize $^1$MIT EECS \ \ \
$^2$Harvard Business School \ \ \
$^3$Harvard SEAS \ \ \
$^4$Stanford CS} \\
{\normalsize $^5$Microsoft Research  \ \ \
$^6$Stanford Statistics}}
\date{}

\begin{document}
\maketitle

\setcounter{tocdepth}{2}
\doparttoc %
\renewcommand\ptctitle{}
\faketableofcontents %

\begin{abstract}
Machine unlearning---efficiently removing the effect of a small "forget set" of training data on a pre-trained machine learning model---has recently attracted significant research interest. Despite this interest, however, recent work shows that existing machine unlearning techniques do not hold up to thorough evaluation in non-convex settings. In this work, we introduce a new machine unlearning technique that exhibits strong empirical performance even in such challenging settings. Our starting point is the perspective that the goal of unlearning is to produce a model whose outputs are {\em statistically indistinguishable} from those of a model re-trained on all but the forget set.  This perspective naturally suggests a reduction from the unlearning problem to that of {\em data attribution}, where the goal is to predict the effect of changing the training set on a model's outputs. Thus motivated, we propose the following meta-algorithm, which we call Datamodel Matching (\dmm): given a trained model, we (a) use data attribution to {\em predict} the output of the model if it were re-trained on all but the forget set points; then (b) {\em fine-tune} the pre-trained model to match these predicted outputs.
In a simple convex setting, we show how this approach provably outperforms a variety of iterative unlearning algorithms. Empirically, we use a combination of existing evaluations and a new metric based on the KL-divergence to show that even in non-convex settings, \dmm achieves strong unlearning performance relative to existing algorithms. An added benefit of \dmm is that it is a meta-algorithm, in the sense that future advances in data attribution translate directly into better unlearning algorithms, pointing to a clear direction for future progress in unlearning.
\end{abstract}

\section{Introduction}
\label{sec:intro}
The goal of machine \emph{unlearning} is to remove 
(or ``unlearn'') the impact of a specific collection 
of training examples from a trained machine learning model. 
Initially spurred by regulations such as the EU's {\em Right to be Forgotten} \citep{ginart2019making},  
machine unlearning has found a variety of recent applications including:
removing the effect of toxic, outdated, or poisoned data
\citep{pawelczyk2024machine,
goel2024correctivemachineunlearning}; rectifying copyright
infringement in generative models \citep{liu2024unlearning, dou2024avoidingcopyrightinfringementmachine, triantafillou_llm_copyright_unlearning}; 
and aligning LLMs \citep{li2024wmdpbenchmarkmeasuringreducing, yao2024largelanguagemodelunlearning}. 

This plethora of potential applications has prompted a growing 
line of research into better {\em unlearning algorithms}.
An unlearning algorithm takes as input a model $\theta$
(trained on a dataset $S$) and a ``forget set'' $S_F \subset S$,
and outputs a model $\theta_{UL}$ that ``looks like'' it was trained on 
the so-called ``retain set'' $S_R := S \setminus S_F$.
Of course, one valid unlearning algorithm simply
ignores the trained model $\theta$
and trains a new model $\theta_{UL}$ from scratch on the retain set $S_R$.
This algorithm clearly succeeds at the task of unlearning, 
since the generated $\theta_{UL}$ really {\em is} trained only on the retain set.
But as model and dataset sizes continue to increase, or unlearning requests become more frequent, 
this approach becomes infeasible. 
The goal of unlearning is thus to {\em approximate} this naive retraining algorithm 
while imposing a much lower computational burden.

For simple (e.g., convex) models, 
there are fast unlearning algorithms that also enjoy provable guarantees
\citep{neel2021deletion, graves2021amnesiac,approximate_data_deletion,
mahadevan2021certifiablemachineunlearninglinear, suriyakumar2022algorithms}.
For large neural networks, however---where efficient unlearning is 
arguably most relevant, given the cost of training from scratch---the situation
is considerably murkier. The only methods that obtain provable guarantees
tend to significantly degrade accuracy
and/or require significant changes to the training pipeline
\citep{sisa_unlearning, li2022largelanguagemodelsstrong}. 

As a result, unlearning algorithms for neural networks
typically rely on heuristic approaches that finetune
an initial model $\theta$ into an ``empirically unlearned'' 
model $\smash{\widehat{\theta_{UL}}}$.
These approaches, however, have not yet led to consistently reliable 
unlearning algorithms, as evidenced by a variety of empirical evaluations 
and benchmarks \citep{hayes2024inexact,kurmanji2023towards,pawelczyk2023incontext}.

A pervasive challenge (identified by \citet{hayes2024inexact}) for
fine-tuning-based approaches is what we refer to 
as the {\em missing targets} problem.
The problem can be summarized as follows: in order to unlearn a
forget set point $x \in S_F$, fine-tuning-based methods typically employ 
some version of {\em gradient ascent} on $x$, starting from $\theta$, and gradient descent on the retain set $S_R$ in order to maintain performance.
If left unrestricted, gradient ascent will continue to make the loss 
on $x$ arbitrarily high---what we want, however, is to increase 
the loss only until it reaches its {\em counterfactual value}, i.e., 
the loss on $x$ of a model trained on the retain set $S_R$. 
Ideally, we could terminate the algorithm when the model's loss on $x$ 
reaches this ``target'' value, but the problem is that
(a) we do not have access to the target;
and (b) the optimal ``stopping time'' might be different for different points $x \in S_F$.
The result is the well-documented phenomenon of unlearning algorithms
``undershooting'' and ``overshooting'' the loss on different examples \citep{hayes2024inexact}.

\paragraph{This work.} 
In this paper, we present a new unlearning algorithm that sidesteps the
issue discussed above, and (empirically) achieves state-of-the-art unlearning 
performance.
Our algorithm resembles prior techniques in that we rely on fine-tuning 
the trained model $\theta$. 
We deviate from prior work, however, through two main ideas:
\begin{enumerate}
    \item \textbf{Oracle Matching (OM).} 
    Consider the following thought experiment: what if we could access 
    the {\em outputs} (but not the parameters) of a 
    model trained on the retain set $S_R$?
    We show that such ``oracle'' access directly enables an efficient,
    fine-tuning-based unlearning algorithm.
    Rather than minimizing/maximizing loss on the retain/forget sets,
    this algorithm samples a subset of the entire dataset that includes the forget set, 
    and directly minimizes the difference between model outputs and oracle outputs on this subset. 
    Conceptually, this algorithm is ``stable'' in that upon convergence,
    the model agrees with the oracle on all the fine-tuning points, 
    including those in the forget set---thus sidestepping the aforementioned
    ``missing targets'' problem.
    Empirically, we find that the fine-tuned model also {\em generalizes} beyond 
    the fine-tuning points, and in some way ``distills'' the target model
    into parameters $\theta'$.
    \item \textbf{Oracle simulation.} $\om$ 
    on its own is not an unlearning algorithm---it relies on 
    the very ``oracle model'' that it aims to replicate. 
    Observe, however, that implementing $\om$ does not require access to the weights of an oracle model, 
    but only to its outputs on a fixed number of inputs. 
    Thus, $\om$ can be implemented efficiently given access to an efficient routine for computing such outputs.
    Such a routine is precisely the target of {\em predictive data attribution}
    methods \citep{ilyas2024mltutorial}, where the goal is exactly to 
    predict how a model's outputs would change if its training dataset were modified.
    This leads to our second idea: instead of fine-tuning on ``oracle'' outputs, 
    we fine-tune on {\em simulated} outputs from a predictive data attribution method. 
    We show that despite these methods being imperfect, applying our $\om$
    algorithm to {\em simulated} oracle outputs works nearly as well as using
    the true oracle outputs.
\end{enumerate}
The resulting algorithm, {\em Datamodel Matching} (\dmm), not only achieves current
state-of-the-art performance (\Cref{fig:headline}), but also introduces a reduction from unlearning
to data attribution, 
allowing us to translate future improvements in the latter field to 
better algorithms for the former.

\begin{figure}[!htbp]
    \centering
        \includegraphics[width=1.0\linewidth]{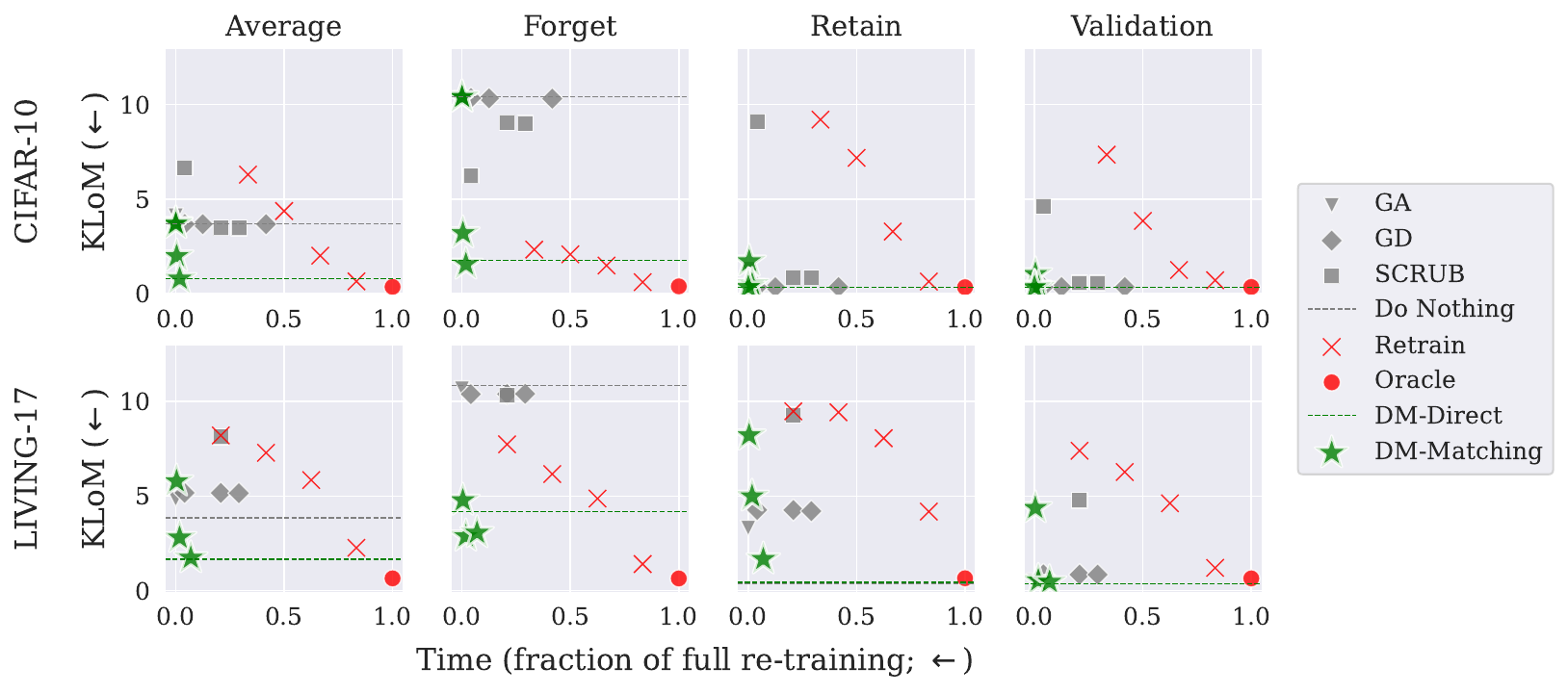}
                \caption{{\bf Effective unlearning via predictive data attribution.} We apply different approximate unlearning  methods to trained DNNs to unlearn selected forget sets from CIFAR-10 and ImageNet-Living-17. \klom scores (y-axis) measure the quality of unlearning by computing the distributional distance between unlearned predictions and oracle predictions (e.g., 0 means perfect unlearning). To contextualize each method's efficiency, we also show the amount of compute relative to full re-training (x-axis). We evaluate \klom~values over points in the forget, retain, and validation sets to ensure that unlearning is effective across all datapoints, and report the 95th percentile in each group; we also report their average (1st column). %
               Our new methods leveraging data attribution (\dmdirect and \dmm) dominate the pareto frontier of existing unlearning methods, and approach the unlearning quality of oracle models (full re-training) at a much smaller fraction of the cost. 
}
    \label{fig:headline}
\end{figure}

The rest of our paper proceeds as follows.
In Section \ref{sec:background}, we formally introduce the unlearning problem,
as well as the field of (predictive) data attribution.
In Section \ref{sec:define_unlearning}, we strengthen existing 
unlearning evaluation by introducing a new metric called {\em KL Divergence of Margins} (\klom).
\klom directly adapts a formal definition of unlearning \citep{neel2021deletion}
to be computationally and statistically tractable to estimate,
and addresses some challenges faced by existing unlearning metrics.
Then, in \Cref{sec:oracle_matching}, we combine the two insights above (Oracle Matching
and Oracle Simulation) to derive {\em Datamodel Matching} (\dmm).
Our algorithm achieves state-of-the-art performance across a suite of empirical evaluations.
Finally, in Section \ref{sec:linear_models}, we provide some theoretical justification 
for our algorithm using a case study of underdetermined ridge (linear) regression.
In particular, we show that in this simple setting, the Oracle Matching (\om) primitive 
can provably lead to faster convergence.
We conclude with a discussion of limitations and directions for future work.

\section{Preliminaries}
\label{sec:background}
\label{sec:prelim} 
In this section, we introduce some preliminary notation, 
definitions, and results. Throughout the section, 
we will let $S \in \mathcal{X}^n$ be a fixed dataset 
drawn from an example space $\mathcal{X}$, and we define a 
{\em learning algorithm} $\algo: \mathcal{X}^* \to \Theta$
as a (potentially random) function mapping
from datasets to machine learning models $\theta$.
Finally, for an example $x \in \mathcal{X}$,
we use $f_x: \Theta \to \mathbb{R}^k$ to denote a {\em model evaluation} on the example $x$
(for example, this may be the $k$-dimensional per-class probabilities).

\subsection{Machine unlearning}
Consider a machine learning model $\theta \sim \algo(S)$ trained on a dataset $S$.
Given a ``forget set'' $S_F \subset S$, and 
a corresponding ``retain set'' 
$S_R = S \setminus S_F$, the goal of an  
{\em exact} unlearning algorithm is to compute a sample from 
$\algo(S_R)$ starting from the trained model $\theta$:
\begin{definition}[Exact unlearning \citep{ginart2019making}]
    An \underline{\smash{unlearning algorithm}} $\mathcal{U}: \Theta \times 2^{|S|} \to \Theta$ is said
to be an \underline{\smash{exact}} unlearning algorithm if, for all $S_F\subset S$,
$\smash{\mathcal{U}(\algo(S), S_F) \overset{d}{=} \algo(S_R),}$
where $\smash{\overset{d}{=}}$ represents equality in distribution over models. 
\end{definition}
Though compelling in theory, exact unlearning tends to be  
too stringent a criterion when applied to deep learning,
often leading to computational infeasibility or degradations in 
accuracy \citep{liu2024unlearning}.
This motivates a look at {\em approximate unlearning}, 
which asks only for the distribution over unlearned models 
to be $(\epsilon, \delta)$-indistinguishable from re-training: 
\begin{definition}[$(\varepsilon, \delta)$-unlearning \citep{neel2021deletion}]
    \label{def:unlearn}
    $\mathcal{U}$ is an $(\epsilon, \delta)$-approximate unlearning algorithm if, for all $\mathcal{O} \subset \Theta, S_F \subset S$ we have that
    \begin{align}
        \label{eq:approx_unlearning}
    \Pr{[\mathcal{U}(\algo(S), S_F) \in \mathcal{O}]} &\leq e^{\epsilon}\Pr{[\algo(S_R) \in \mathcal{O}]} + \delta, \\
    \Pr{[\algo(S_R) \in \mathcal{O}]} &\leq e^{\epsilon}\Pr{[\mathcal{U}(\algo(S), S_F) \in \mathcal{O}]} + \delta
    \nonumber
    \end{align}
\end{definition}
This definition (intentionally) resembles differential privacy, 
and asks for the distribution of unlearned models to be statistically close 
to the distribution of re-trained ``oracle'' models. In particular,
this condition guarantees than an adversary who observes the model 
returned by the unlearning algorithm $\mathcal{U}$
cannot draw any inferences with accuracy that is much 
higher than if the model was fully re-trained. 

While unlearning
algorithms achieving Definition~\ref{def:unlearn} exist for convex models
\citep{neel2021deletion, approximate_data_deletion, guo2019certified}, and for non-convex models when
the training process is altered or under stylized optimization conditions
\citep{sisa_unlearning, chien2024stochastic, Gupta2021AdaptiveMU}, the bulk of ongoing work
in unlearning evaluates Definition~\ref{def:unlearn} \emph{empirically}, rather
than as a provable property. 
We return to the problem of evaluating unlearning algorithms 
more carefully in \Cref{sec:define_unlearning}.

\subsection{Predictive data attribution (Datamodeling)}
Our work also draws on a separate line of work in 
machine learning called {\em data attribution} 
\citep{koh2017understanding,hammoudeh2024training,ilyas2024mltutorial}.
Broadly, data attribution is an area concerned with connecting training data samples to the
predictions of the corresponding ML models. 
Of particular relevance to our work is 
a particular type of data attribution 
called {\em predictive data attribution}
(also known as datamodeling \citep{datamodels, TRAK}).

In predictive data attribution, the goal is to produce an
estimator (or {\em datamodel}) that takes as input a training set,
and as output accurately predicts the behavior of a 
machine learning model trained on that training set.
Using our existing notation: for an example $x \in \mathcal{X}$, a datamodel for $x$ is a function
$\smash{\hat{f}: 2^S \to \mathbb{R}}$ such that, for any $S' \subset S$,
\begin{equation}
    \label{eq:datamodel_acc}
    \hat{f}(S') \approx f_x(\algo(S')).
\end{equation}
In other words, $\hat{f}(S')$ directly predicts the result of applying 
the training algorithm $\algo$ to the dataset $S'$, and 
evaluating the function $f_x$ on the resulting model.\footnote{For notational simplicity we take $f_x$ to be a scalar function, but more generally we can apply this approach to each coordinate of vector-valued $f_x$.} 
Despite the complexity of modern training algorithms $\algo$ (e.g., training deep neural networks with stochastic gradient descent),
\citet{datamodels} show that 
{\em linear} datamodels often suffice 
to accurately predict model behavior.
In other words, for an example $x$, one can compute a 
vector  $\beta \in \mathbb{R}^{|S|}$ such that,
for subsets $S' \subset S$,
\[
    \hat{f}(S') := \sum_{z_i \in S'} \beta_i \approx f_x(\algo(S')).
\]
To compute these coefficients, \citet{datamodels} sample
a variety of subsets $S_1,\ldots, S_k$ at random from $S$,
and then solve the (regularized) regression problem 
\begin{align}
    \label{eq:testset_datamodels}
    \beta = \min_{w \in \mathbb{R}^n}
    \frac{1}{m} \sum_{i=1}^m (w^\top \bm{1}_{S_i} - f_x(\mathcal{A}(S_i)))^2 + \lambda \|w\|_1.
\end{align}
They show that despite the datamodel being constructed using 
random subsets $S_i \subset S$, the function $\hat{f}$ remains 
remarkably accurate on non-random datasets (see \cite{datamodels}
for a full evaluation).
Linear datamodels are particularly appealing for two reasons. 
First, the coefficients $\beta_i$ have an intuitive interpretation
as the influence of the $i$-th training example on a model's 
prediction on $x$ \citep{koh2017understanding, feldmanmemorization}.
Second, they establish a connection to a class of statistical 
techniques relating to influence functions, 
which has unlocked a suite of tools for estimating the 
coefficients more effectively \citep{TRAK, grosse2023studying}.

\label{sec:problem_setup}
\section{Empirically evaluating unlearning}
\label{sec:define_unlearning}
In \Cref{sec:background}, we introduced the unlearning problem,
culminating in a formal definition of the problem (\Cref{def:unlearn}).
As stated,
evaluating whether \Cref{def:unlearn} holds for a given unlearning algorithm
is a difficult problem for several reasons.
First, given the overparameterized nature of large-scale models,
fully satisfying Definition \ref{def:unlearn} is likely impossible, and verifying it involves comparing distributions in a
space with millions or billions of dimensions.
Secondly,
the definition generally needs to hold over arbitrary forget sets $S_F$,
or at least across a range of forget sets $S_F$
likely to occur in practice.

\paragraph{The current evaluation paradigm.}
To deal with these challenges,
one typically evaluates unlearning by
using model {\em outputs} $f_x$\footnote{A common choice of model output used in prior work, which we will also use, is the {\em margin} of the classifier.} rather than parameters,
and testing for the {\em implications} of \Cref{def:unlearn}
rather than for the definition directly.
In particular, the strongest existing unlearning evaluation for supervised learning,
called \ulira \citep{hayes2024inexact},
takes inspiration from {\em membership inference attacks} (MIAs)
\citep{carlini2021membership} and measures the ability of an adversary
to distinguish between the distribution of outputs of an 
unlearned model on (a) validation examples and (b) unlearned examples.

Omitting a few details, the main idea is as follows: consider an unlearning
algorithm $\mathcal{U}$ and a training algorithm $\algo$.
Starting from a training set $S$, first randomly hold out a small validation set $S_V$.
From what remains, fix a forget set $S_F \subset S$ and sample datasets $S_i \subset S$.
Train a model $\theta^* \sim \algo(S_i)$ on each dataset,
and then unlearn $S_F$ to produce
$\theta^*_{UL} \sim \mathcal{U}(\theta^*, S_F)$.
By \Cref{def:unlearn}, these models should be indistinguishable
from models that were not trained on $S_F$, and thus should
behave (nearly) identically on points $x \in S_F$ and on points $x \in S_V$.

Operationalizing this intuition, \ulira with probability
$\frac{1}{2}$ draws either $x \in S_F$ or $x \in S_V$, and
measures the output $y = f_x(\theta^*_{UL})$.
An (optimal) adversary observes $y$, and tries to guess
whether the corresponding $x$ was an unlearned point or a validation point.
If $\mathcal{U}$ is an $(\epsilon, \delta)$-unlearning
algorithm, the two distributions $y|x\in S_F$ and $y|x\in S_V$ would be
$(\epsilon, \delta)$-indistinguishable by post-processing, and so the adversary
could have accuracy greater than $\frac{1}{2}e^{\epsilon} + \delta$.
By the Neyman-Pearson lemma the optimal adversary performs a likelihood ratio test, and
$\ulira$ is based on approximating this likelihood by estimating the
distribution of unlearned and test margins.

For a more detailed description of \ulira and overview of similar MIA-based approaches, we refer the reader to
\citep{hayes2024inexact}, and we include the pseudocode for the computationally efficient
implementation of this evaluation \EfficientULIRA in Appendix \ref{appendix:ulira}.

\paragraph{A more direct evaluation.}

Recall from Section \ref{sec:prelim} that traditionally the target of unlearning algorithms has been
$(\varepsilon, \delta)$-approximate unlearning (Definition \ref{def:unlearn}).
Note that in essence, Definition~\ref{def:unlearn} simply asks for the distribution
induced by the unlearning algorithm be ``close'' to a
distribution of models that have never been trained on $S_F$. In particular, we can view
it as a special case of the condition
\begin{equation}
    \label{eq:general_unlearn}
    \Delta_\delta (\mathcal{U}(\mathcal{A}(S), S_F), \text{safe}(S_F)) \leq \epsilon,
\end{equation}
where $\Delta_\delta$ is a statistical divergence measure (in this case the $\delta$-approximate max divergence) parameterized by $\delta > 0$,
and $\text{safe}(S_F)$ is a distribution of ``safe'' models (i.e., models that
have not been trained on $S_F$).
We can recover \Cref{def:unlearn} exactly by
letting $\text{safe}(S_F)$ be exactly $\mathcal{A}(S \setminus S_F)$, i.e., the distribution of models
trained on all but the forget set and
setting $\Delta_\delta$ appropriately.

We make two observations about \eqref{eq:general_unlearn}.
First, the choice of $\text{safe}(S_F)$ is somewhat arbitrary,
and in particular {\em any} distribution that does not depend on $S_F$, and produces a useful model would suffice. This includes, for example, distributions $\mathcal{A}'(S \setminus S_F)$
for learning algorithms $\mathcal{A}' \neq \mathcal{A}$, or distributions of {\em ensembles}
of models $\mathcal{A}(S \setminus S_F)$.
Second,
while the $\Delta_\delta$ used in \Cref{def:unlearn} has an appealing privacy interpretation,
it is sensible (especially given our focus on empirical evaluation)
to consider other divergences that are easier to estimate.
These two observations inspire a metric that we call \klom for
empirical unlearning evaluation. $\klom$ corresponds to Definition~\ref{def:unlearn} where we
(a) use $\Delta =$ KL divergence,
(b) allow for an arbitrary ``reference distribution'' $\text{safe}(S_F)$, and
(c) as in $\ulira$, study distributions of model {\em outputs} $f_x$ rather than
 parameters.

\begin{definition}[KL divergence of margins (\klom)]
    For an unlearning algorithm $\mathcal{U}$,
    reference distribution $\text{safe}(S_F)$,
    and input $x$,
    the KL divergence of margins (\klom) is given by
    \[
        \text{\klom}(\mathcal{U})\!:=\!
        D_{KL}\!\left(
            \text{safe}(S_F),\
            f_x(\mathcal{U}(\algo(S), S_F))
        \right).
    \]
\end{definition}
\noindent We note that we could choose other divergence measures such as the $\alpha$-Renyi divergence, which is also tractable to compute and has the benefit that it can be directly converted into a bound on the $\delta$-approximate max divergence in \Cref{def:unlearn} \cite{renyi}. %

Despite the arbitrariness of $\text{safe}(S_F)$, unless otherwise noted we will
mirror \Cref{def:unlearn} and take $\text{safe}(S_F) := \algo(S \setminus S_F).$
Throughout the rest of this work, we primarily evaluate unlearning algorithms via computing
$\klom$ for different inputs $x$ from the forget set, retain set, and
validation set. We also evaluate our algorithms with $\ulira$, and defer these
results to \Cref{app:eval_details}.

Compared to \ulira, \klom is simpler to implement,
has a natural correspondence with our original \Cref{def:unlearn},
and importantly, does not suffer from {\em catastrophic unlearning}:
observe that an unlearning algorithm $\mathcal{U}$ that
transforms its input into a random classifier will pass an \ulira
evaluation, as the random classifier will treat unlearned points and
validation points identically.
In contrast, by forcing us to explicitly specify $\text{safe}(S_F)$,
\klom explicitly compares unlearned models to a baseline whose performance
we know {\em a priori}.
We explain this further in Appendix
\ref{appendix:ulira_vs_klom}, and provide empirical evidence in Figure
\ref{fig:ulira_vs_acc_plot}. Crucially, both $\klom$ and $\ulira$ evaluate unlearning algorithms using \emph{point-specific} distributional estimates, which as observed in \cite{hayes2024inexact} makes these evaluations far more stringent than prior approaches.

\section{\dmm: Unlearning by matching on simulated oracles}
\label{sec:method}
Having defined an evaluation apparatus, 
we now introduce our algorithm for machine unlearning.
We first motivate the algorithm by observing a common challenge
in existing methods. We then, in \Cref{subsec:oracle_matching}, propose an effective
hypothetical algorithm for unlearning, under the unrealistic assumption
that we have access to outputs of the ``oracle'' model.
In \Cref{subsec:dd}, we show how to accurately simulate such oracle outputs using data attribution methods. %
Finally, in \Cref{subsec:dmm}, we combine these insights and present our final algorithm, Datamodel Matching (\dmm), and demonstrate its effectiveness and efficiency.

\label{sec:oracle_matching}

\subsection{Motivation: the missing targets problem}
Recall that the goal of unlearning is to approximate an {\em oracle} model,
i.e., a model that was never trained on a given ``forget set'' of data.
In strongly convex settings, this oracle model is unique,
since it corresponds to the minimizer of a strongly convex loss function
over the complement of the forget set (the {\em retain set}).
Thus, running gradient descent ($\GD$) on the retain set loss yields a
provable (and in some cases, efficient \citep{neel2021deletion}) unlearning
algorithm.

In the context of deep neural networks, however, $\GD$ alone is insufficient.
In these settings, the loss function and training data alone do not fully
specify the final model.
In particular, once we have already minimized loss on the forget set,
applying \GD on the retain set does not significantly alter
forget set predictions, preventing us from recovering the oracle model.
Many unlearning methods for deep neural networks
\citep{neurips2023unlearning, kurmanji2023towards} thus
actively {\em increase} loss on
the forget set (e.g., via gradient ascent)
while {\em maintaining} performance on the retain set.

This general approach comes with a significant set of drawbacks,
which we collectively refer to as the {\em missing targets} problem.
First, the assumption that forget set points will increase in loss
after unlearning and retain set points will not is not necessarily correct. 
For example, if there are semantically similar points across the forget and retain sets,
then loss on points in the retain set can also increase after re-training solely on the retain set; conversely, the loss may not noticeably increase for any of the points if the model can generalize sufficiently well from similar points in the retain set. 
Second, even for a forget set point whose loss {\em does} increase under the
oracle model, our goal is not to increase loss arbitrarily,
but instead only until it reaches its ``target value''---the expected loss under a perfectly retrained model.
Since we lack access to these values, it is challenging to know when a given forget set point has been ``unlearned.''
Prior work tries to deal with this problem by devising heuristic
regularization schemes, e.g., via early stopping, but nevertheless
often overshoot or undershoot the target loss for a given data point.
\Cref{fig:individual_point_unlearning__scrub} illustrates this phenomenon
for a popular unlearning algorithm called SCRUB: over
iterations of the algorithm, different points are unlearned
(and then subsequently ``overshot'') at different points in time
\citep{hayes2024inexact}.

\begin{figure}[ht]
    \centering
    \includegraphics[width=1.0\linewidth,trim={0 0 0 2.2em},clip]{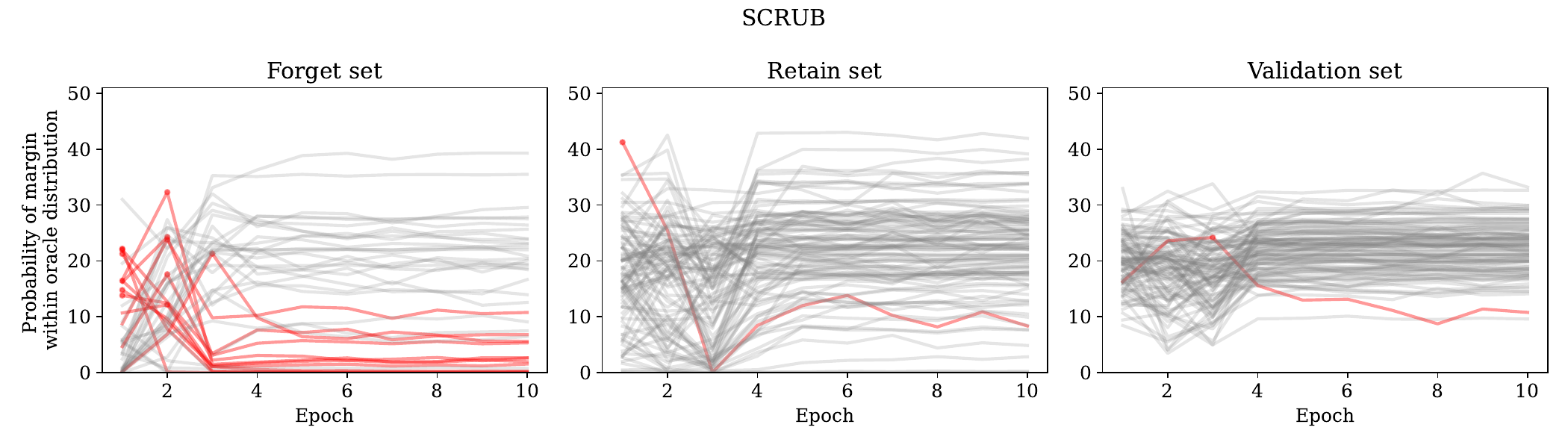}
    \caption{{\bf The missing targets problem.}
    We apply the SCRUB \citep{kurmanji2023towards} algorithm to unlearn a forget set of CIFAR-10,
    and measure how well different (random) points are unlearned over time.
    To quantify how well a given point $x$ is unlearned, we fit a Gaussian
    distribution to the outputs of oracle models at $x$, and compute
    the likelihood of the average outputs from unlearned models under
    this distribution.
    We track this likelihood (y-axis) for random points across the duration of unlearning algorithm (x-axis).
    For many examples in the forget set (shown in red), unlearning quality is hurt by training
    for too long as we lack access to oracle targets. %
    }
    \label{fig:individual_point_unlearning__scrub}
\end{figure}

\subsection{The oracle matching algorithm}
\label{subsec:oracle_matching}
In essence, the underlying challenge is that we do not know the oracle model's behavior a priori. But what if we did have access to its predictions? In particular,
\begin{center}
{\em Given access to sample outputs from the oracle model (re-trained without the forget set), can we efficiently fine-tune an existing model (trained on the full dataset) to match the outputs out of sample?}
\end{center}

\noindent While assuming access to oracle outputs is unreasonable---since our goal is to produce an oracle model in the first place---later in \Cref{subsec:dmm}, we
 will replace oracle access with an efficient proxy using data attribution. For now, we simply assume we have direct access to oracle predictions, and focus on understanding whether
 gradient-based optimization can match predictions of the oracle.
 Even in this idealized setting, it is not clear how fast (if at all) gradient descent can converge to an oracle model. For example, whether we can do this efficiently with a small sample is unclear; it is possible that fine-tuning the trained model can match the oracle predictions on the sampled points, but fail to generalize when evaluated on held-out points.

Formally, we assume access to predictions of an oracle model $f^{\text{oracle}}(x) := f_x(\mathcal{A}(S_R))$, where again $S_R$ is the retain set and $f_x$ is the evaluation of the model on input $x$ (e.g., in classification settings one can take $f$ to be the logits of the neural network).
The {\em Oracle Matching} (\om) algorithm runs gradient descent\footnote{In practice we use the Adam optimizer.} to minimize the MSE between the output logits from the model $f_x(\theta)$ and oracle predictions $f^{\text{oracle}}(x)$ on samples $x$ from the
forget and retain sets; see the pseudocode in Algorithm~\ref{algo:oracle_matching}.

\begin{algorithm} 
\caption{Oracle Matching (\om)}
\begin{algorithmic}[1] 
    \State \textbf{Input:} Trained model $\theta$; oracle predictions $f^{\text{oracle}}(x)$; fine-tuning retain set size $r$
    \State \textbf{Output:} Unlearned model $\theta_{UL}$
    \State Initialize $\theta_0 = \theta$
    \For{$t = \{0,...,T-1\}$} \Comment{$T$ epochs}
        \State $S_R' \leftarrow S \setminus  S_F$  \Comment{Sub-sample $r$ points from retain set}
        \State $S_{\text{fine-tune}} = S_F \bigcup S_R'$ \Comment{Combine with forget set}
        \For{$x \sim S_{\text{fine-tune}}$} %
            \State $L(\theta_t) = \|f_x(\theta_t) - f^{\text{oracle}}(x)\|^2$ \Comment{Compute MSE loss}
            \State $\theta_{t+1} = \theta_{t} - \eta_t \cdot \nabla_\theta L(\theta_{t})$ \Comment{Perform update with gradient}
        \EndFor
    \EndFor
    \State \textbf{Return} Model $\theta_{UL} = \theta_T$
\end{algorithmic} \label{algo:oracle_matching}
\end{algorithm}

\paragraph{Evaluating oracle matching.}
We evaluate \om on various forget sets on two image classification tasks:
ResNet-9 models trained on CIFAR-10 and ResNet-18 models trained on an ImageNet
subset Living-17~\citep{santurkar2020breeds}.
We compare \om to the following unlearning baselines: gradient ascent (GA) on forget set, gradient
descent (GD) on retain set, SCRUB \citep{kurmanji2023towards}, the trivial ``Do Nothing'' baseline that returns the original trained model, and re-training fully or for a fraction of the time.  %
We evaluate all methods using \klom (\Cref{sec:define_unlearning}) over distribution of $100$ unlearned 
 (method-specific) and $100$ re-trained models (see \Cref{app:eval_details} for more details).

\begin{figure*}[!htbp]
    \centering
        \includegraphics[width=1.0\linewidth]{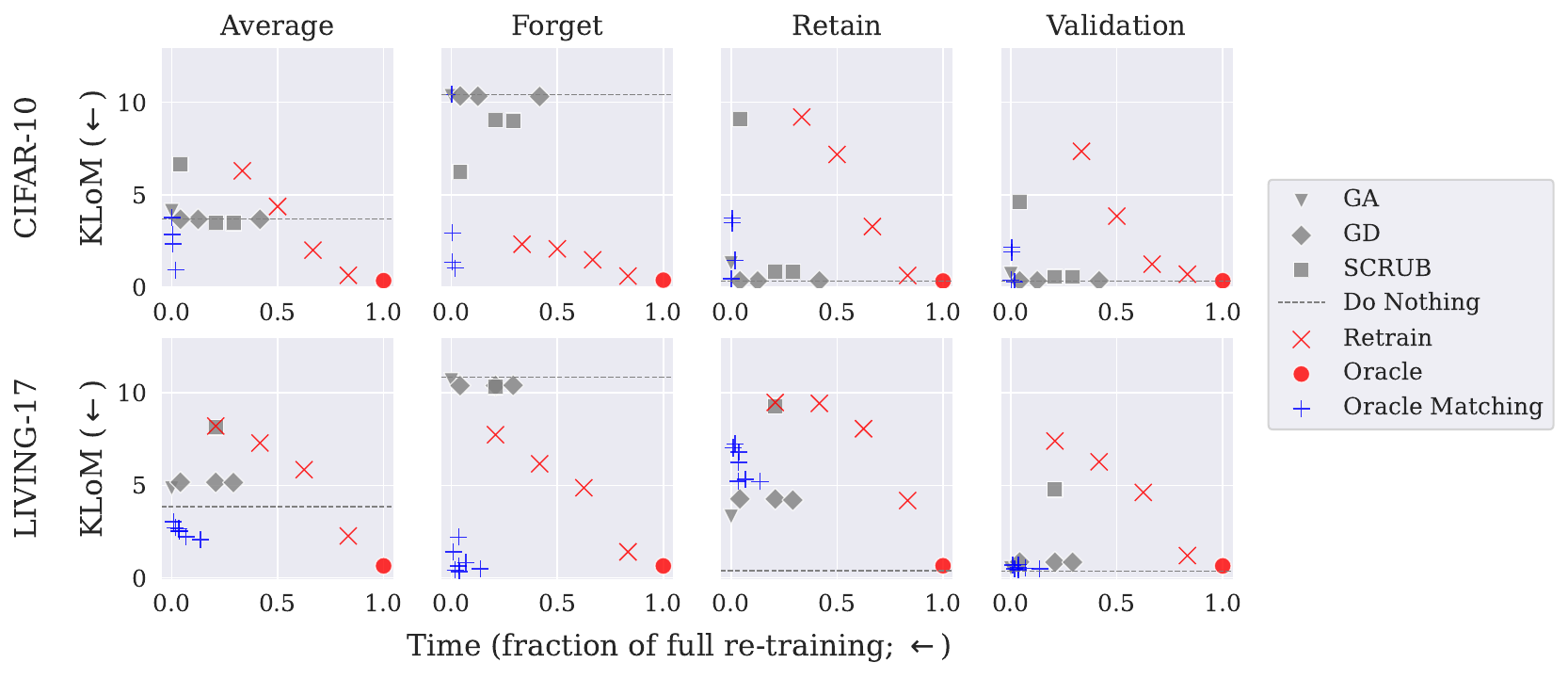}
                \caption{{\bf Oracle matching can efficiently approximate re-training.} %
                The \klom metric (y-axis) measures the distributional difference between unlearned predictions and oracle predictions (0 being perfect). We also show the amount of compute relative to full re-training (x-axis). We evaluate \klom~values over points in the forget, retain, and validation sets and report the $95$th percentile in each group; we also report the average across groups (1st column).                %
}
    \label{fig:headline_OM}
\end{figure*}

The results (\Cref{fig:headline_OM}) demonstrate that \om is able to efficiently match the predictions of the oracle.
Models unlearned with \om closely match the oracle distribution---as measured by \klom scores---across all splits of the dataset (forget, retain, and validation sets),
significantly outperforming all of the prior gradient-based approaches.
Importantly, \om achieves effective unlearning while using {\em less than 5\%} of the compute of full-retraining. In contrast, matching the performance of $\om$ on the forget set by retraining requires spending more than $60\%$ of full retraining time. The success of $\om$ implies that for any given trained model $\theta$ and the forget sets we studied: i) there exists another model $\theta'$ close in parameter space that yields similar predictions as an oracle retrained without the forget set; and ii) $\theta$ can be fine-tuned to quickly converge to $\theta'$ with a sufficient sample of oracle outputs. We find that using a sufficiently high ratio of forget points in the fine-tuning set and a sufficient fraction of retain points (but still much smaller than the full train set) is able to provide enough guidance (see \Cref{sec:understanding} for exact details). 

\subsection{An efficient proxy for oracles: datamodels}
\label{subsec:dd}
We saw that \om is effective at unlearning, but \om is not a practical algorithm as it assumes access to oracle outputs. To now turn this into a practical algorithm, we leverage methods for predictive data attribution (introduced in Section~\ref{sec:prelim}) to {\em simulate} oracle outputs. That is, for each input $x$ rather than computing $f^{\text{oracle}}(x)$ in Line $7$ of \om, given access to a datamodel $\hat{f}_x$, we can swap in an estimate of $f^{\text{oracle}}(x)$ constructed using the datamodel: 

Recall that a {\em datamodel} $\hat{f}_x$ predicts the counterfactual output of the model on input $x$ when trained on an arbitrary subset 
$S \setminus S_F$: $\smash{\hat{f}_x(S \setminus S_F) \approx f_x(\mathcal{A}(S \setminus S_F))} = f^{\text{oracle}}(x)$.
In the case of linear datamodels, we can parameterize the datamodel with a vector $\beta(x)$ so that $\hat{f}_x(S \setminus S_F) \coloneqq \beta_0 + \sum_{i \in S \setminus S_F}\beta_i(x)$.\footnote{$\beta_0$ is a constant, which we do not need to estimate as we show next.}
Leveraging linearity, we can re-write this as $(\beta_0 + \sum_{i \in S}\beta_i(x)) - \sum_{i \in S_F}\beta_i(x)$, and we also replace the first term with the starting model output $f_x(\theta)$.
Our general algorithm, \dmdirect (\Cref{appendix:dm_direct}), thus simulates the oracle outputs as $h(x) \coloneqq f_x(\theta) - \sum_{i \in S_F} \beta_i(x)$.

\begin{figure}[!hbtp]
    \centering
    \includegraphics[width=1.\linewidth]{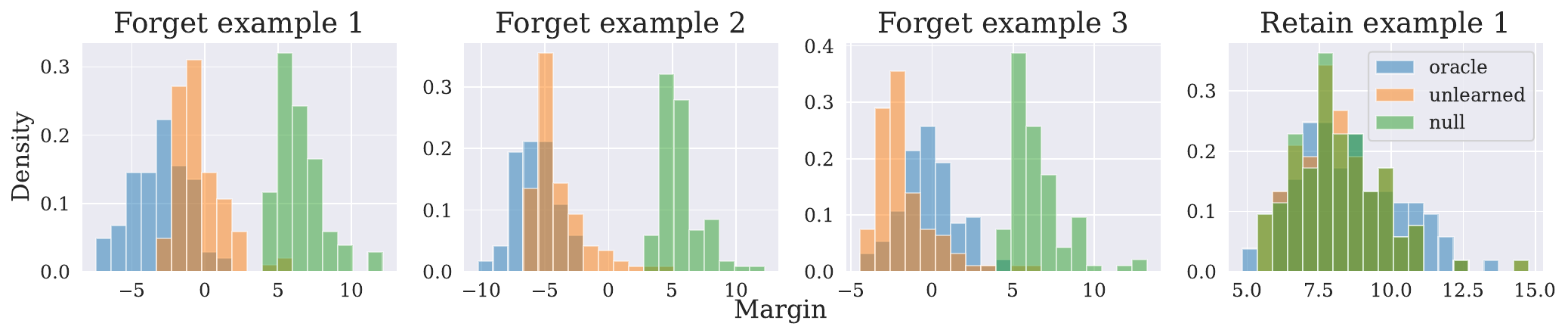}%
    \caption{{\bf Datamodels predict oracle outputs.} We examine the accuracy of datamodel predictions for unlearning a CIFAR-10 forget set (ID 5). For random samples from the forget and retain sets, we compare the distribution (across multiple runs) of margins when evaluated on that example across three settings: i) null (model on full dataset); ii) oracle (model re-trained without forget set); and iii) unlearned (using \dmdirect, applied to instances of null models). In every case, the predicted outputs (orange) closely match the ground-truth (oracle), demonstrating the effectiveness of datamodels as a proxy for oracle outputs.} %
    \label{fig:margin_histograms}
\end{figure}

\paragraph{Estimating datamodels.}
To estimate datamodels, we follow the approach in \citep{datamodels}: we train models random subsamples of the full training set and use sparse linear regression to fit datamodel vectors $\beta(x)$
(later in \Cref{subsec:dm_science} we explore alternative estimators).

\paragraph{Evaluating \dmdirect.}
In \Cref{fig:margin_histograms}, we compare model outputs on random forget and retain examples; the histograms show that the unlearned outputs from \dmdirect closely approximate the true oracle outputs.
\klom evaluations (\Cref{fig:headline}) show that \dmdirect (green line) produces outputs close in distribution to that from oracle re-training for almost all points in the data distribution. 

\paragraph{\dmdirect as a practical unlearning method.} This shows that if we only care about unlearning in prediction space rather than in model weights, we can use \dmdirect as a practical unlearning algorithm. While this may be suitable for some applications of unlearning, in practice we generally still need to update the model weights themselves since (i) privacy reasons may require deletion of the full model $\theta_0$; (ii) deploying \dmdirect for inference would require computing a new datamodel for every new input $z$, slowing down inference speed; and (iii) we may want to share the model directly, for example for fine-tuning on other downstream tasks. In the next Section we show how to use \dmdirect to unlearn in weight space.

\subsection{Oracle matching with Datamodels}
\label{subsec:dmm}
Now that we have an efficient proxy for oracle outputs via datamodels, we revisit the \om algorithm from earlier.
Our final algorithm, Datamodel Matching (\dmm), first uses datamodels to generate approximations of oracle predictions on a subset of retain points and the forget points, and then runs \om on the datamodel predictions: %

\begin{algorithm}[!hbtp]
\caption{Datamodel Matching (\dmm)}
\begin{algorithmic}[1]
    \State \textbf{Input:} Trained model $\theta_0$; datamodels $\beta(\cdot)$; fine-tuning set size $r$
    \State \textbf{Output:} Unlearned model $\theta$
    \State $S_R' \leftarrow S\setminus  S_F$  \Comment{Sub-sample $r$ points from retain set}
    \State $S_{\text{fine-tune}} = S_F \bigcup S_R'$
    \State $h \leftarrow \dmdirect(\theta_0, \beta, S_f)$ \Comment{Simulate oracles with datamodels}
    \For{$t = \{1,...,T\}$} \Comment{$T$ epochs}
        \For{$x \sim S_{\text{fine-tune}}$} \Comment{mini-batch}
            \State $L(\theta_t) = \|f_x(\theta_t) - h(x)\|^2$ \Comment{Compute loss}
            \State $\theta_{t+1} = \theta_{t} - \eta_t \cdot \nabla_\theta L(\theta_{t})$ \Comment{Perform update with gradient}
        \EndFor
    \EndFor
    \State \textbf{Return} Model $\theta = \theta_T$
\end{algorithmic} \label{algo:datamodel_matching}
\end{algorithm}

In Figure \ref{fig:headline},\footnote{Plots in the main paper highlight forget set ID 5 for CIFAR-10 and ID 1 for Living-17; see \Cref{fig:klom-full,fig:klom-full-living} for results on all forget sets.} we contextualize the performance of \dmm against baselines and \dmdirect from earlier.
\dmm achieves levels of unlearning similar to that of fully retraining the model (as measured by \klom scores), while using significantly less compute. %
Using datamodels allow us to recover the performance of $\om$ and outperforms all prior gradient-based approaches. Importantly, \dmm does not degrade accuracy on the retain and validation sets (\Cref{appendix:model-accuracy}), a common failure mode in prior methods. In contrast, previous baselines often perform worse overall than {\em doing nothing}.

\dmm is significantly more effective than partially re-training given the same computational budget for fine-tuning.
Here, we do not include the cost of computing datamodels as this is a {\em one-time} cost and hence amortized over many unlearning requests.\footnote{The practice of not including pre-computation costs is standard in the unlearning literature, e.g., \citet{approximate_data_deletion}.} This is possible because once a we estimate an accurate datamodel (either via re-sampling as done here or influence function-like approximations, which we explore in \Cref{subsec:dm_science}), the datamodel generalizes well to new forget sets in practice.

We note that \dmm is not perfect: on the tail of the retain and validation sets, \om still deviates from the oracle distribution. This deviation is biggest on the retain set, which we suspect is due to gradient updates ``reversing'' the overfitting that occurred during original training.

\section{Understanding effectiveness of datamodel matching}
\label{sec:understanding}
We saw that datamodel matching is an effective and efficient algorithm for unlearning. Here, we aim to better understand the effectiveness of Datamodel Matching (\dmm). Since \dmm consists of i) oracle matching (the fine-tuning algorithm) and ii) estimating datamodels (approximating oracle outputs), we study each component separately. First, in \Cref{subsec:om_science}, we analyze the stability of \om across time and to different choices of hyperparameters and show:
\begin{itemize}
    \item {\bf \om is stable across time}: Once $\om$ unlearns an example, the example generally stays unlearned after further iterations, unlike prior gradient-based approaches. As a result, \om is also much more stable than prior methods with respect to the choice of optimization hyperparameters.
    \item {\bf \om generalizes from a small sample}: Though \om introduces additional design parameters (sampling ratios for forget and retain sets), we find that \om is effective as long as it uses a sufficiently large sample of both. In particular, it only requires fine-tuning on a small fraction of the retain set, making \om efficient.
\end{itemize}

\noindent Next, in \Cref{subsec:dm_science}, we ablate different components of the datamodel estimators to better understand necessary ingredients for \dmm and show:
\begin{itemize}
    \item {\bf Necessity of modeling interactions between datapoints}: Datamodels (linearly) model the effect of different training examples on other inputs. We show that using only the ``diagonal'' entries (i.e., modeling only the self-influence, or the effect of including training example on itself) is much less effective and we thus need to model inter-example influences.
    \item {\bf Scaling with datamodel estimation cost}: We study how both datamodel predictiveness and unlearning performance scale with computational budget and show that the latter scales more favorably.
    \item {\bf Effectiveness of fast approximate data attribution methods}: We show that replacing regression-estimated datamodels with much faster alternatives like \trak still yield effective unlearning algorithms, albeit with worse performance.
\end{itemize}

\subsection{Unlearning stability of \om} %
\label{subsec:om_science}
To motivate our approach, we demonstrated earlier that existing fine-tuning approaches suffer from the problem of different unlearning rates due to missing targets. %
Figure \ref{fig:individual_point_unlearning__all} (cf. \Cref{fig:individual_point_unlearning__scrub}) shows that \om no longer suffers from the same problem; unlearning quality generally only improves over time (there is no risk of overshooting), even if points are still unlearned at different rates. Since points generally stay unlearned over time, $\om$ can be stopped after sufficiently many epochs and maintain good unlearning performance across nearly all examples. As a result, $\om$ is much more robust to the choice of optimization parameters such as learning rate and number of epochs compared to prior gradient ascent based methods (see \Cref{appendix:hyperparam-sensitivity} for more analysis).

\begin{figure}[H]
    \centering
    \includegraphics[width=1.\linewidth]{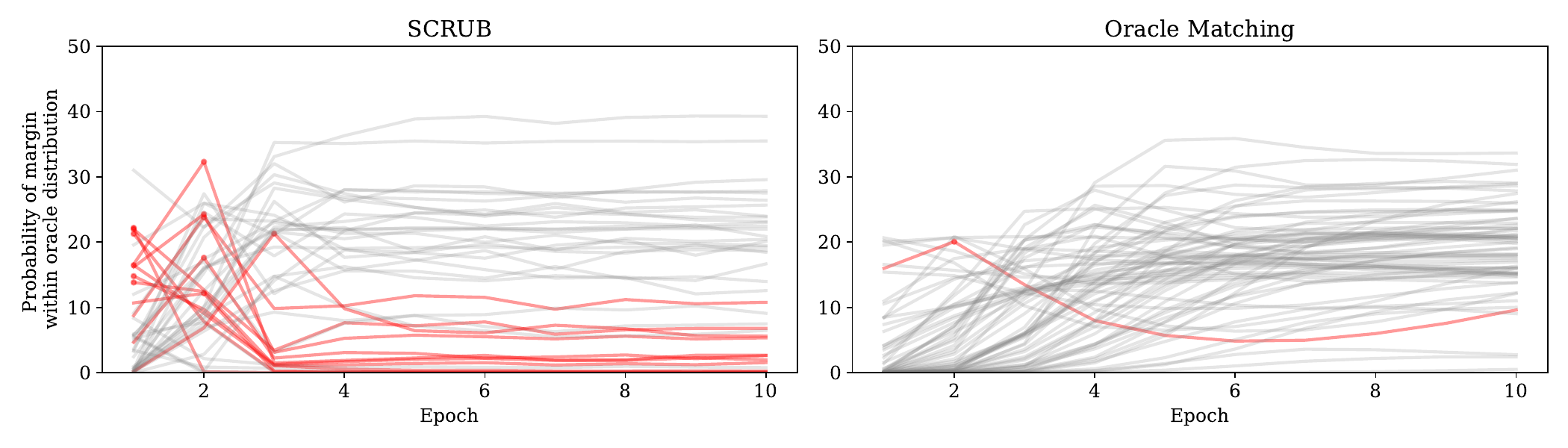}

\caption{{\bf Oracle matching circumvents the stopping time problem.}
    We revisit the earlier analysis for SCRUB (left) and apply the same analysis to Oracle Matching (right). The red lines highlight examples in the forget set whose unlearning quality is hurt by training longer. This ``overshooting'' happens frequently with SCRUB, but only rarely with Oracle Matching.}
    \label{fig:individual_point_unlearning__all}
\end{figure}
\label{fig:stopping_time_revisited}

\paragraph{Generalization of \om.}
Oracle matching fine-tunes on samples from both the forget and retain sets.
The efficiency of \om hinges on whether it can generalize quickly from a small sample.
Ablations show that %
i) \om succeeds as long as the ratio of forget set points in the fine-tuning set is sufficiently high (\Cref{fig:om_ablation_forget})
and ii) a small fraction ($\ge 0.04$) of retain set suffices to guide \om to converge towards an oracle on most retain set samples (\Cref{fig:om_ablation_retain}).
That is, \om is able to effectively {\em generalize} from a small sample of oracle outputs.
Since approximating each new oracle prediction requires estimating a separate datamodel, the fact that we only need to estimate a small number of datamodels is crucial for the efficiency of the \om algorithm.

\begin{figure}[htbp]
    \centering
    \begin{minipage}[b]{0.71\textwidth}
      \centering
      \raisebox{3mm}{
        \includegraphics[width=\textwidth]{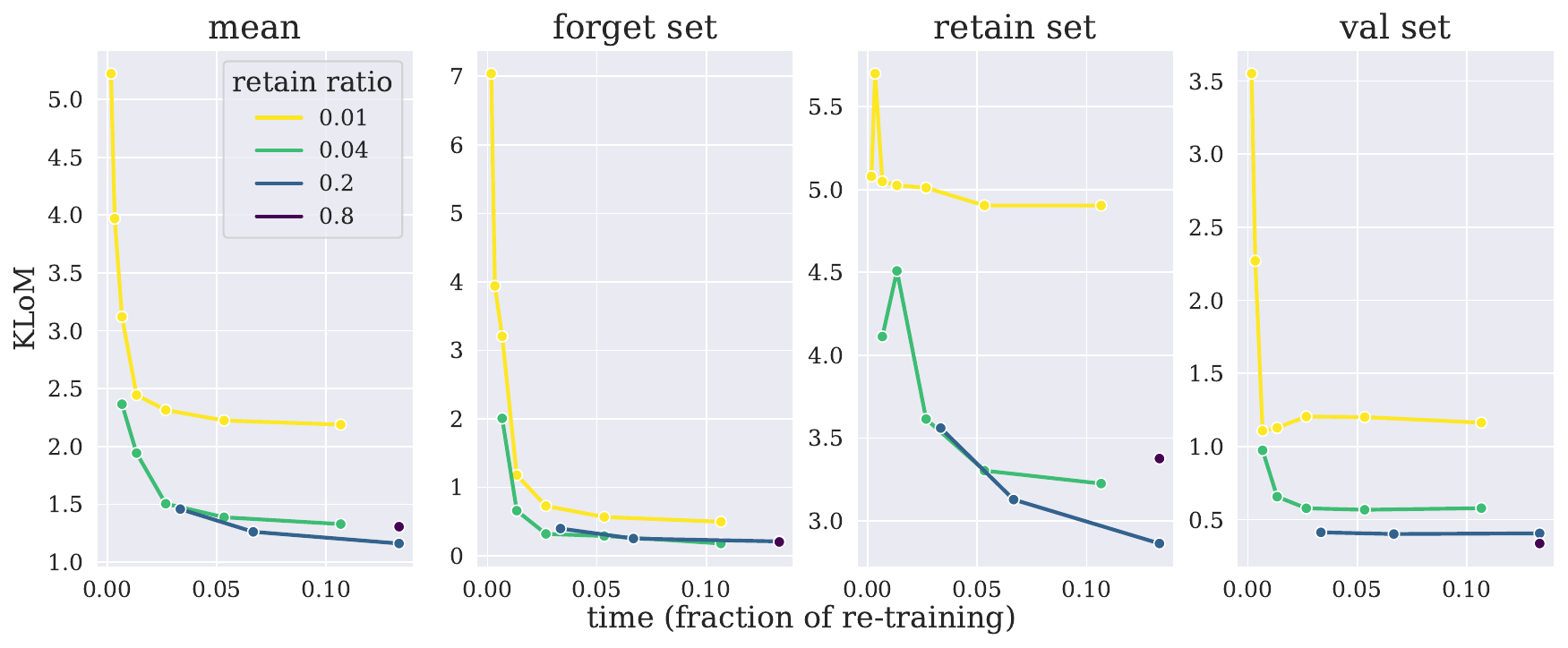}
      }
      \caption{Varying the fraction of retain set (of the full dataset) sampled for oracle matching. A sufficiently large fraction  ($\ge$ 0.04) appears to be sufficient in enabling \om to generalize to out-of-sample. In other words, \om can ``distill`` the oracle model using a small subset of the training data.}
      \label{fig:om_ablation_retain}
    \end{minipage}
    \hfill
    \begin{minipage}[b]{0.24\textwidth}
      \centering
      \includegraphics[width=\textwidth]{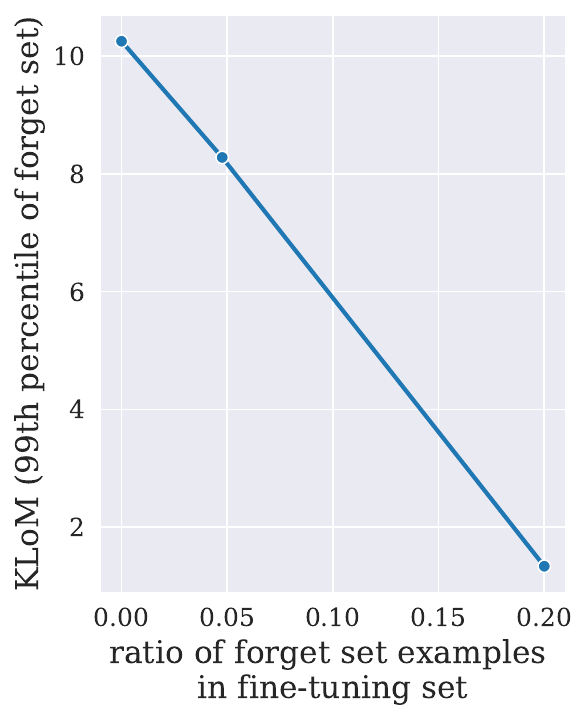}
      \caption{Higher ratio of forget set samples in the fine-tuning set for \om improves performance on the forget set.}
      \label{fig:om_ablation_forget}
    \end{minipage}
  \end{figure}

\subsection{Datamodel ablations}
\label{subsec:dm_science}

Given the stability and efficiency of \om, the success of \dmm essentially only depends on the fidelity of the datamodel-based approximations to oracle outputs. We now study aspects of datamodel estimation in more detail.

\paragraph{Necessity of modeling interactions between datapoints.} %
Datamodels model the effect of different training examples on a given model output as a {\em linear} function. Despite their simple form, we were able to accurately simulate the oracle model outputs in \Cref{subsec:dd}.
Inspecting the weights of linear datamodels show that the highest-magnitude entry for the datamodel of a training input corresponds to the example itself (i.e., excluding that training example has the largest negative effect on the model prediction on itself).
The corresponding weight is what is also known as the {\em memorization score} in prior work \citep{feldmanmemorization}. Could memorization scores---without modeling interactions between examples---suffice to linearly model oracle outputs? 
\begin{figure}[!hb]
    \centering
    \includegraphics[width=0.85\linewidth]{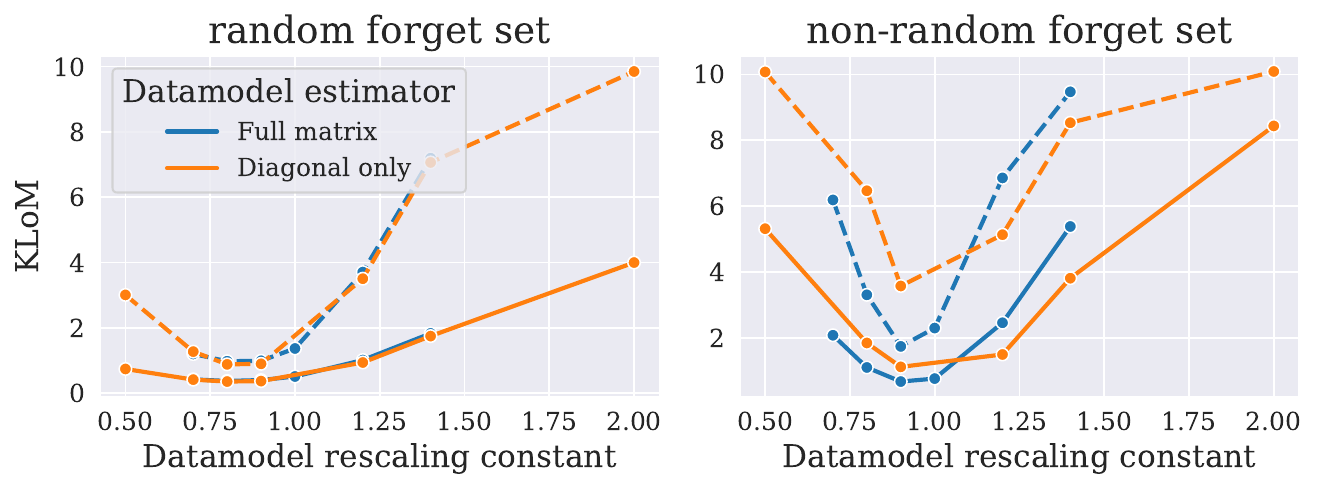}
    \caption{{\bf The effect of off-diagonal entries and re-scaling on the unlearning effectiveness.} We evaluate how well \dmdirect approximates oracle outputs for two different types of forget set on CIFAR-10 (left is random; right is non-random). Solid and dotted lines correspond to the mean and 95\%-percentile \klom scores. Diagonal-only indicates that we only use the memorization scores (the ``diagonal'' of the datamodel matrix $\{\beta_{j}(x_i)\}$).
    }
    \label{fig:mem_scores}
\end{figure}

In \Cref{fig:mem_scores}, we evaluate the quality of oracle predictions when using the full datamodel vector vs. only the memorization scores, and find the following:
\begin{itemize}
\item {\bf Insufficiency of memorization scores for unlearning non-random forget sets:} Excluding off-diagonal entries has no effect for random forget sets (left);  this makes sense intuitively since two random examples are in general unrelated. In contrast, using only the diagonal entries hurts unlearning quality for non-random forget sets (right), particularly for examples in the tail (dotted line). Intuitively, this is because the ``off-diagonal'' weights in datamodels %
capture important cross-example correlations (e.g., similar examples should have reinforcing effect on one another).  Moreover, globally scaling the memorization scores (orange lines) cannot make up for the missing off-diagonal entries. Hence, to effectively unlearn forget sets that arise in practice (i.e., non-random) via our approach, it seems necessary to model interactions between different datapoints.
\item {\bf Consistency of best scaling:} As an artifact, we also find that scaling down the datamodel weights globally by a factor ($\approx$0.9 here) improves unlearning quality marginally. Conveniently, the scale seems consistent across different types of forget sets; one could calibrate this scale using a ``held-out'' forget as part of a pre-computation stage, and subsequently apply to all forget sets.
\end{itemize}

\paragraph{Scaling with estimation cost.}
Though we excluded the cost of estimating datamodels in our analysis in \Cref{subsec:dmm} (as it is a one-time cost),
practically we need to account for them as estimating predictive datamodels is computationally expensive.
But estimating them is not all or nothing: we can tradeoff the computational cost and the datamodel predictiveness
by varying the number of re-trained models.
Here, we investigate how varying the computational resources affects both datamodel predictiveness (LDS) and unlearning performance (as measured by \klom). 
In \Cref{fig:dm_scaling}, we show the result of varying the computational cost by orders of magnitude: we can observe that while the datamodel predictiveness (as measured by the {\em linear datamodeling score} \citep{TRAK}) continues increase at the same rate, \klom---averaged across various forget sets---begins to saturate.

\begin{figure}
    \centering
    \includegraphics[width=0.8\linewidth]{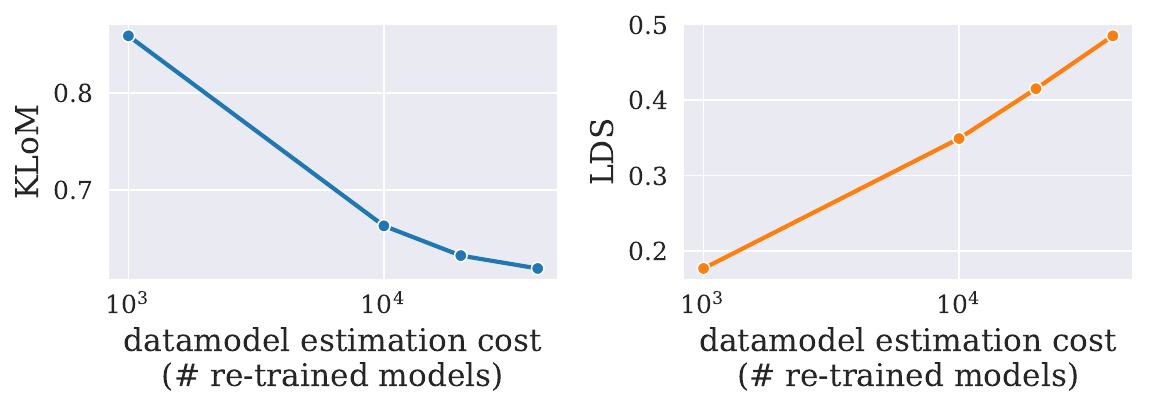}%
    \caption{{\bf The effect of estimation cost on unlearning performance and datamodel predictiveness.} On CIFAR-10, we show how unlearning performance of \dmdirect (measured by \klom; lower is better) and datamodel predictiveness (measured by the linear datamodeling score; higher is better) scales with datamodel estimation cost (number of re-trained models, in $\{ 10^3, 10^4, 2\times 10^4, 4\times10^4\}$).
    \klom is averaged over different forget sets. }
    \label{fig:dm_scaling}
\end{figure}

We hypothesize that this is due to the following difference in the distribution of counterfactual subsets being evaluated: while achieving high LDS requires predicting model outputs on the same distribution of random subsets of the training data that the datamodels were trained on, achieving high \klom scores requires predicting over small non-random forget sets. In practice, this suggests that for the purposes of unlearning, cheaper alternatives (either by using the same regression-based estimator with reduced computational resources or by using a different estimator) may perform nearly as well as computationally expensive methods.

\begin{figure}[!ht]
    \centering
    \includegraphics[width=1.\linewidth]{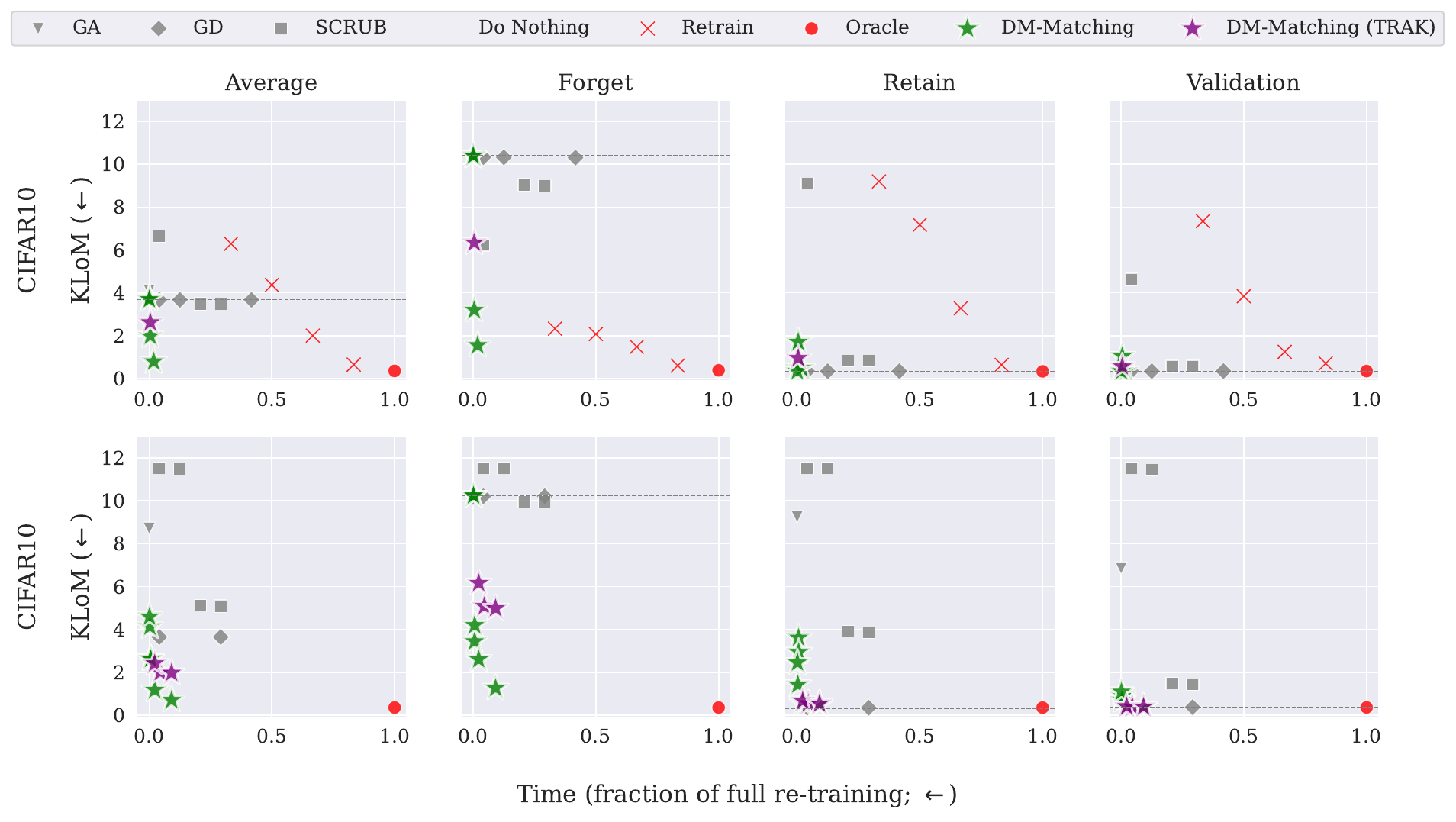}%
    \caption{{\bf Datamodel-Matching with more efficient estimators.} In the same set up as in \Cref{fig:headline}, we evaluate the effectivness of \dmdirect and \dmm when a different, more efficient estimator (TRAK) for datamodels. Though the equality of unlearning degrades, our algorithms still outperform prior methods given the same fine-tuning budget. Note that the computational savings from using \trak is not reflected on the x-axis, as computation of datamodels is considered separately from the fine-tuning cost of \dmm.}
    \label{fig:headline_TRAK}
\end{figure}
\paragraph{Efficient unlearning with TRAK.}
Even with the above considerations, estimating datamodels using the regression approach is still prohibitively expensive except in smaller scale settings. %
Since the work of \citet{datamodels}, follow-up works \citep{TRAK,grosse2023studying} have shown that efficient alternative methods can be comparably effective with substantially lower computational costs. %

Next, we investigate whether \dmm is still effective when datamodels are estimated with \trak, which is based on a particular approximation to the influence function. In \Cref{fig:headline_TRAK}, we run $\om$ with predictions generated by \trak estimators with x1000 less compute than the regression-based datamodels. As expected $\dmm$ with TRAK performs worse in terms of $\klom$ than when datamodels are used---as TRAK is worse at the underlying task of predicting counterfactuals, but we find that this drastically cheaper alternative to \dmm still outperforms all prior methods significantly. These results highlight that improving the accuracy and efficiency of alternative data attribution methods %
can help us obtain even more efficient and effective unlearning algorithms.

\nopagebreak
\section{Oracle Matching for Linear Models}
\label{sec:linear_models}
We have seen that empirically $\om$ outperforms standard gradient-based unlearning methods, and we have highlighted the missing targets problem as one possible explanation. %
Are there other factors that contribute to the success of $\om$ relative to prior methods gradient-based methods, and can we better understand what settings we expect $\om$ to perform well? This motivates studying a setting where the missing targets problem is neutralized: when the objective is strongly convex. In this setting, the unlearned model is the unique empirical risk minimizer on $X_R$, and $\gd$ initialized at the current model is a provably effective unlearning algorithm \citep{neel2021descent}. 
Even in the setting when GD on its own can converge, does providing ``guidance'' from an oracle help?

We answer this affirmatively: First, in ~\Cref{subsec:simple_exp}, we empirically identify two factors that influence whether $\om$ outperforms \gd: the strength of regularization and stochasticity in optimization. In ~\Cref{subsec:theory}, we theoretically characterize the exact convergence rates of full batch $\om$ and $\gd$ to the unlearned model in terms of the degree of regularization and the relative eigenmass on the forget and retain sets. Unlike in the full-batch setting where both algorithms converge at a linear rate, in the stochastic setting we show that $\om$ converges exponentially faster than $\sgd$, which sheds light on the superior performance of stochastic $\om$ in our empirical results. 

\paragraph{Setting.} We consider the following ridge regression algorithm, given by
\begin{align}
\label{eq:ridge_objective}
    \mathcal{A}(S) :=
    \arg\min_{\theta}  \sum_{(\bm{x_i}, y_i) \in S} \left(
        \theta^\top \bm{x_i} - y_i
    \right)^2 + \lambda \| \theta \|_2^2,
\end{align}
where $\bm{x_i} \in \mathbb{R}^d$ are the training inputs,
$y_i \in \mathbb{R}$ are the corresponding labels,
and the setting is overparameterized, so $d > |S|$.
Given a model $\theta_{\text{full}} = \mathcal{A}(S)$
trained on a full dataset $S$,
our goal is to unlearn the forget set $S_F \subset S$
by obtaining a model that minimizes the objective on the
retain set $S_R = S \setminus S_F$.
For convenience, we use $X$, $X_R$, and $X_F$ to denote the
covariate matrices for the full dataset $S$,
the retain set $S_R$,
and the forget set $S_F$ respectively. We choose the under-determined ridge regression for three reasons:
(a) The objective~(\ref{eq:ridge_objective}) is strongly convex, and so $\gd$ on $X_R$ is guaranteed to compute the (unique) unlearned model if ran for sufficiently many iterations;
(b) the least-squares objective is amenable to theoretical analysis; and 
(c) the over-parameterized setting is most
relevant to modern deep learning models where $ d \gg n$.

\newcommand{\ultarg}{\ensuremath{\theta_*}}
\paragraph{Unlearning algorithms.} Let $\ultarg = \mathcal{A}(S_R)$
be the minimizer of the ridge regression objective on the retain set
(i.e., the unlearning target).
Starting from $\theta_{\text{full}} = \mathcal{A}(S)$, we evaluate several
iterative first-order unlearning algorithms in terms of their ability to recover $\ultarg$:
\begin{itemize}
    \item[(i)] {\bf Gradient Descent (\gd)} minimizes the ridge regression objective on the retain set $S_R$ using gradient descent
    with constant step size, starting from $\theta_{\text{full}}$;
    \item[(ii)] {\bf Gradient Descent + Ascent (\gda)} incorporates forget set points in the gradient descent updates, combining gradient descent on the retain set with gradient
    ascent on the forget set;
    \item[(iii)] {\bf Oracle Matching (\om)} assumes query access to an unlearned model $\theta^*$, and uses gradient descent
    (with constant step size)
    to minimize squared error with respect to ``oracle'' predictions
    $\bm{x_i}^\top \theta^*$ on the full dataset, aiming to minimize $||X\theta  - X\theta_{*}||^2$;
    \item[(iv)] {\bf Oracle Matching on Retain Set (\omrs)}  performs gradient descent on
    the squared error from oracle predictions but only on the retain set, thereby minimizing the objective $\|X_R\theta^* - X_R \theta\|^2$.
\end{itemize}
We analyze both the full-batch and stochastic versions of these methods.
\Cref{sec:linear_exp_details} contains more details on the exact setup of
the algorithms we evaluate.

\begin{figure}[!ht]
 \centering
    \includegraphics[width=0.96\textwidth, trim = 25 25 25 25, clip]{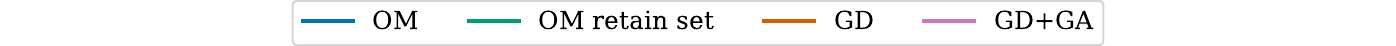}

    \begin{tabular}{cc}
        \begin{subfigure}[b]{0.35\textwidth}
        \captionsetup{ skip=0pt} %
            \includegraphics[width=\textwidth]{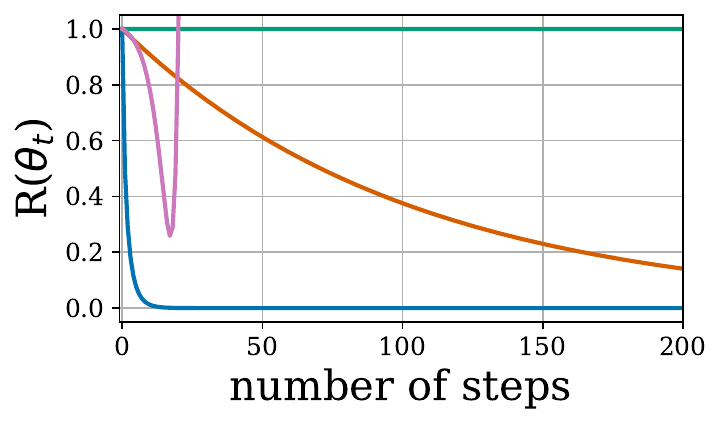}
            \caption{ small $\lambda$, full-batch}
            \label{fig:linear_main_small_lambda_fb}
        \end{subfigure} &
        \begin{subfigure}[b]{0.35\textwidth}
         \captionsetup{ skip=0pt} %
            \includegraphics[width=\textwidth]{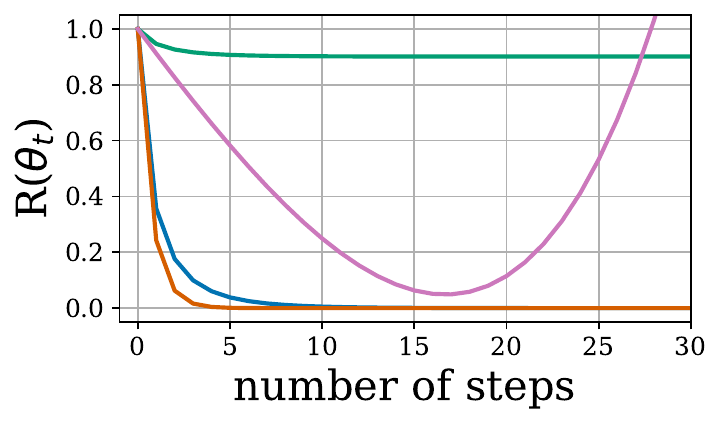}
            \caption{large $\lambda$, full-batch}
            \label{fig:linear_main_large_lambda_fb}
        \end{subfigure} 
        \\
        \begin{subfigure}[b]{0.35\textwidth}
        \captionsetup{skip=0pt} %
            \includegraphics[width=\textwidth]{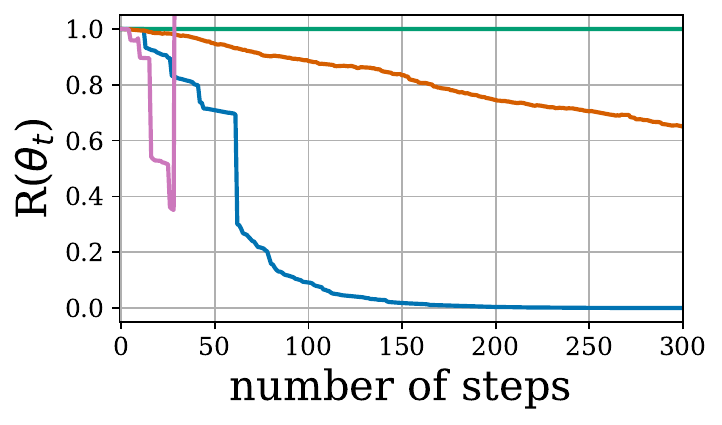}
            \caption{small $\lambda$, stochastic}
            \label{fig:linear_main_small_lambda_stochastic}
        \end{subfigure} &
        \begin{subfigure}[b]{0.35\textwidth}
         \captionsetup{skip=0pt} %
            \includegraphics[width=\textwidth]{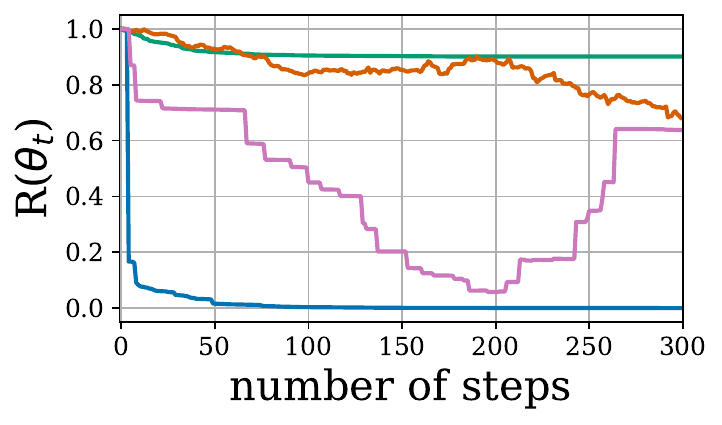}
            \caption{large $\lambda$, stochastic}
            \label{fig:linear_main_large_lambda_stochastic}
        \end{subfigure} \\
        [0.0em] %

    \end{tabular}
    \caption{Comparing unlearning methods for a linear model with $d = 400, n = 100, |S_F| = 5$. The y-axis
    shows the relative squared distance to the optimal unlearned model $R(\theta_t) = \frac{\|\ultarg - \theta_t\|}{\|\ultarg - \theta_{\text{full}}\|}$,
    where $\theta_t$ is the iterate at time $t$,
    $\theta_{\text{full}}$ is the model trained on all data, and $\theta_*$ is
    the optimal unlearned model.}
    \label{fig:linear_main}
    \end{figure}

\subsection{Experimental results}
\label{subsec:simple_exp}
For our experimental analysis, we construct a synthetic setting with $n = 100$ 
datapoints of dimension $d = 400$. 
We first fix a ``ground-truth'' weight vector $\theta \in \mathbb{R}^{400}$ 
by sampling each coordinate $\theta_i \sim \mathcal{N}(0, 1)$.
We then generate a dataset $S$ consisting of 100 training points $(\bm{x_i}, y_i)$,
where each $\bm{x_i} \sim N(0, \mathbf{I}_{400})$ is sampled from a
400-dimensional standard Gaussian, and each
$y_i \sim \mathcal{N}(\theta^\top x_i, 4)$.
We choose a forget set $S_F$ consisting of 5 random points from $S$.

We obtain the starting point for unlearning $\theta_{\text{full}}$
by minimizing the ridge regression objective \eqref{eq:ridge_objective}
with $\lambda = 4$. We then apply the iterative unlearning algorithms above,
starting from $\theta_0 := \theta_{\text{full}};$ after each iteration $t$, we
measure the relative distance from the current unlearning iterate $\theta_t$
to the unlearning target $\ultarg$, i.e.,
\[
    R(\theta_t) := \frac{\|\ultarg - \theta_t\|}{\|\ultarg - \theta_{\text{full}}\|}
\]

\paragraph{Standard setting ($\lambda = 4$, full-batch updates).}
Figure \ref{fig:linear_main_small_lambda_fb} depicts the performance of the four
unlearning algorithms we consider.
We observe that \om converges to the unlearned solution much faster \gd,
while \gda and \omrs both fail to converge even as $t \to \infty$.
\om differs from \gd primarily in two aspects: (i) it
minimizes the squared error with respect to the target model predictions
(rather than the training data labels),
and (ii) its updates involve forget set points along with retain set points
(rather than just retain set points).
\omrs helps us isolate the impact of these differences---applying oracle matching
to only the retained points makes negligible progress.
This suggests that the inclusion of forget set points in the update is the
primary driver of superior performance of \om.

Investigating further, we observe that during unlearning, the model parameters
change the most in directions orthogonal to the retain set.
Indeed, we find that despite the fact that $24\%$ of the mass of the
forget set points lies in the span of the retain set, this span actually
captures less than $0.01\%$ of the mass of the ground-truth update
$\ultarg - \theta_{\text{full}}$.
This explains the importance of including forget set points in the unlearning
objective: \omrs cannot make any
progress in directions orthogonal to the retain set, and
\gd only makes progress in those directions
due to the the $\ell_2$ regularization term.
On the other hand, \om makes rapid progress in these directions
due to the inclusion of forget set points.

As for \gda, this algorithm also uses the forget set points,
but since it does not have access to oracle predictions on these points,
it must instead introduce them into the objective in a heuristic way.
As a result, the algorithm initially makes rapid progress towards the
target solution $\ultarg$ (in this initial stage, the update induced
by the oracle labels matches that induced by gradient ascent)
but then suffers from the ``missing targets'' problem discussed in the previous
parts and diverges.
In the convex setting, we can mitigate this issue somewhat by using the
gradient norm on the retain set as a stopping criterion, but there is
no obvious way to transport this criterion to the more relevant
non-convex case.

\paragraph{Large-regularization case ($\lambda = 400$)} In Figure \ref{fig:linear_main_large_lambda_fb},
we consider the same setting as above but set $\lambda$ to a much larger value
of 400.
Here, \gd converges slightly faster than \om due to the stronger $\ell_2$
regularization, which aids \gd in converging along directions orthogonal to the
retain set. Thus, \om converges faster than \gd when $\lambda$ is moderate but can
be slower with large $\lambda$. Again, \omrs and \gda do not successfully converge,
but \gda---guided by the heuristic use of the forget set points---initially makes
significant progress towards $\ultarg$ before eventually diverging.

\paragraph{Stochastic case.} In Figures
\ref{fig:linear_main_small_lambda_stochastic} and
\ref{fig:linear_main_large_lambda_stochastic}, we replicate the experiments
above with stochastic variants of the unlearning algorithms. As in the
non-stochastic case, we see the \om on the retain set fails to make any progress,
and that \gda makes quick progress but then diverges.
However, unlike in the full-batch setting, SOM outperforms \sgd in both the large
and small $\lambda$ settings (see Appendix \ref{sec:linear_exp_details} for a discussion). 
This surprising finding is characterized in Theorem~\ref{thm:som_sep}
below, where we show that SOM converges exponentially faster than \sgd.

\subsection{Convergence theory}
\label{subsec:theory}
We now turn to studying the algorithms above theoretically,
aiming to formalize some of our empirical observations.
We start with the full-batch case, corresponding to
\cref{fig:linear_main_small_lambda_fb,fig:linear_main_large_lambda_fb} above;
we then proceed to the stochastic/minibatched case
(\cref{fig:linear_main_small_lambda_stochastic,fig:linear_main_large_lambda_stochastic}).
In all cases, we will focus on the two convergent algorithms above: oracle matching (\om)
and ridge gradient descent (\gd).

\paragraph{Full-batch case.}
In Theorem \ref{thm:convergence_gd_om}, we provide a theoretical analysis of the
convergence rates of \gd and \om (see Appendix \ref{sec:thm_proof} for the proof).
The key takeaway from this result is that the convergence rate for both algorithms
depends on both (a) the relative eigenmass of the forget and retain sets; and (b)
the strength of the ridge regularization.

\begin{restatable}[Proof in \cref{sec:thm_proof}]{theorem}{thmconvg}
\label{thm:convergence_gd_om}
    Let $S$ and $S_R$ be the full training set and the retain set respectively,
    with input matrices $X$ and $X_R$ and corresponding labels $y$ and $y_R$.
    Additionally, let $\theta_{\text{full}}$
    and $\ultarg$ denote the optima of the ridge objective \eqref{eq:ridge_objective}
    for the full data $S$ and retain set $S_R$ respectively.
    After $t$ iterations of unlearning starting from $\theta_{\text{full}}$,
    the iterate $\theta_t$ satisfies
    \[
     \theta_t - \theta_{*} = \begin{cases}
     \left(I - 2\eta\lambda \right)^t
     \left( I - \frac{2\eta}{1 - 2\eta\lambda}X_R^\top X_R \right)^{t}
     \left(\theta_\text{full} - \theta_{*}\right)
     &\text{ for ridge gradient descent (\gd).} \\
    \left( I - 2  \eta X^\top X \right)^{t} \left(\theta_\text{full} - \theta_{*}\right)
     &\text{ for oracle matching (\om).}
     \end{cases}
    \]
\end{restatable}
Theorem \ref{thm:convergence_gd_om} shows that both (full-batch) \om and \gd
exhibit linear convergence, albeit at different rates.
Indeed, in directions orthogonal to the retain set,
the middle term in the \gd convergence rate disappears, and so
the rate depends only on $\eta \lambda$. Thus, as long as the
learning rate $\eta$ is set high enough (i.e., not to cancel out
the $\lambda$) higher regularization will cause \gd to converge faster.

\paragraph{Stochastic case.} In our experimental analysis, we saw that unlike the
full-batch case, the stochastic version of oracle matching was {\em consistently}
more effective than that of gradient descent.
In Theorem~\ref{thm:som_sep} we show that, at least in the setting of
under-determined ridge regression we consider here, this observation is strongly supported by theory. 
In particular, we show that while \om converges at a linear rate 
(i.e., exponentially fast in $t$), we can show a $\Omega(\frac{1}{t^2})$ 
lower bound on the convergence of \sgd, 
giving a strong separation between the two methods.

\begin{restatable}[Proof in \cref{app:som_separation}]{theorem}{somsgdseparation}
\label{thm:som_sep}
    Consider the setting of \cref{thm:convergence_gd_om},
    where $\theta_t$ is the iterate after $t$ steps of unlearning
    initialized at $\theta_{full}$.
    Further let $\gamma_{\min} = \frac{1}{n}\lambda_{\min}^+(X^TX)$ be the minimum non-zero eigenvalue, and assume that for each $x_i, \|x_i\|_2^2 \leq \beta$.
    Then, as long as the learning rate $\eta \in (0, \frac{2}{5(\beta + \lambda)})$
    and $\text{rank}(X) > 1$, we have:
    \begin{equation}
        \frac{\mathbb{E}\left[\|\theta_t - \theta^*\|^2\right]}{\|\theta_\text{full} - \ultarg\|^2} \in
        \begin{cases}
            O\left((1-\gamma_{min} \eta)^t\right) &\text{ for oracle matching (\om).} \\
            \Omega\left(\frac{1}{t^2}\right) &\text{ for ridge gradient descent (\gd).}
        \end{cases}
    \end{equation}
\end{restatable}

The intuition is as follows.
In each update step, stochastic \om updates the model parameters only in the span
of the random subset of points used for that update. In contrast, stochastic \gd
decays the model parameters in other directions due to the regularization term.
(Recall that this is what led to \gd's improved convergence in the high-regularization
setting.)
While this shrinkage is beneficial in the subspace orthogonal to the retain set,
stochastic \gd also decays the parameters along the directions
spanned by retain set points that are not in the current batch.

Mathematically, this is formalized by defining the \emph{gradient disagreement} 
$$\sigma^2_f := \mathbb{E}_i\left[
            \left\|\nabla f_i(\ultarg) - \mathbb{E}[\nabla f_i(\ultarg)]\right\|^2
        \right] = \mathbb{E}_i\left[
            \left\|\nabla f_i(\ultarg)\right\|^2
        \right],$$ 
which measures for a smooth function $f(x) = \sum_{i}f_i(x)$ that is the sum of smooth functions, a measure of the expected extent to which $\theta^*$ differs from the minimizer of each $f_i$.
       
Note that for the $\om$ objective, $\theta^*$ minimizes each $f_i$ and so $\sigma^2_f = 0$, whereas for $\gd$ when $\text{rk}(X) > 1, \sigma^2_f > 0$. Standard results \cite{garrigos2023handbook} show that if $\theta_t$ is the iterate after $t$ steps of SGD, then
\begin{align*}
        \bbE{ \|\theta_t - \ultarg\|^2 } 
        &\geq \left(1 - 2\eta\beta -\eta^2\beta^2\right)\|\theta_{t-1} - \ultarg\|^2 
                + \frac{\eta^2\sigma^2_f}{2}
\end{align*}
The fact that this second term is non-zero for \gd is what forces the $\Omega(1/t^2)$ lower bound. A full proof is deferred to \cref{app:som_separation}.

\section{Conclusion and Discussion}

In this work, we presented a general framework for reducing unlearning 
to the related problem of predictive data attribution. %
Our fine-grained evaluations using \klom---which directly measures the quality of unlearning in terms of the difference in the distributions of the unlearned model's outputs from oracle counterparts---demonstrate that \dmm significantly outperforms  prior gradient-based unlearning methods and approaches the performance of oracle re-training. To conclude, we discuss some limitations and promising directions for future work:

\paragraph{Extending techniques and evaluations beyond classification.}
While our methods perform well in classification settings with few classes, extending them to work in settings with more classes (e.g., full ImageNet or language-modeling) would make them more practical. Extending our techniques directly (i.e.,  attributing and estimating all class logits)  would incur a heavy computational cost, so  additional techniques will be necessary to make the algorithms more scalable.

\paragraph{Improving and understanding oracle matching.}
As we saw in \Cref{sec:method}, \om and \dmm can sometimes cause a mismatch on the retain set due to ``reversing the overfitting.'' \roy{todo- provide a reference to where this is defined. } Better understanding the dynamics of the \om algorithm and leveraging other insights (e.g., from the model distillation literature) would be valuable for making the matching part of our framework more stable.
One potential direction is to understand when and how better sampling strategies (instead of random subsampling of the retain set) can improve general matching algorithms. 
\sam{+ more clever sampling strategies}

\paragraph{Reducing computational costs.}
Even the most efficient data attribution methods require a non-trivial computational cost (at least on the same order as training the original model). Can we design other cheaper alternatives for data attribution that we can still leverage for---and are possibly tailored to---practical unlearning scenarios without the full computational cost?

\paragraph{Applying to more practical scenarios.}
In our analysis, we only considered single unlearning requests (i.e., removing one forget set). A natural way of extending them to multiple unlearning requests is to apply \dmm sequentially. However, it is plausible that after too many unlearning updates with \dmm, the model diverges far enough so that we need to ``recalibrate`` the pre-computed datamodels.
Analyzing how well existing unlearning algorithms and \dmm compose under multiple unlearning requests and other practical scenarios can be valuable for understanding the practicality and failure modes of these methods.

\section*{Acknowledgements}

We thank Jamie Hayes and Ilia Shumailov for discussions on machine unlearning and ULIRA. 
We thank Salil Vadhan for some useful discussions throughout the paper. Supported in part by NSF grant BCS-2218803. Additionally, we acknowledge Harvard SEAS and MIT CSAIL for providing computational resources.

\printbibliography

@inproceedings{golatkar2020eternal,
  title={Eternal sunshine of the spotless net: Selective forgetting in deep networks},
  author={Golatkar, Aditya and Achille, Alessandro and Soatto, Stefano},
  booktitle={Proceedings of the IEEE/CVF Conference on Computer Vision and Pattern Recognition},
  pages={9304--9312},
  year={2020}
}

@inproceedings{renyi,
  author       = {Ilya Mironov},
  title        = {R{\'{e}}nyi Differential Privacy},
  booktitle    = {{IEEE} Computer Security Foundations Symposium ({CSF})},
  year         = {2017},
}

@article{pawelczyk2024machine,
  title={Machine unlearning fails to remove data poisoning attacks},
  author={Pawelczyk, Martin and Di, Jimmy Z and Lu, Yiwei and Kamath, Gautam and Sekhari, Ayush and Neel, Seth},
  journal={arXiv preprint arXiv:2406.17216},
  year={2024}
}

@inproceedings{thudi2022necessity,
  title={On the necessity of auditable algorithmic definitions for machine unlearning},
  author={Thudi, Anvith and Jia, Hengrui and Shumailov, Ilia and Papernot, Nicolas},
  booktitle={USENIX Security Symposium},
  year={2022}
}

@article{hammoudeh2024training,
  title={Training data influence analysis and estimation: A survey},
  author={Hammoudeh, Zayd and Lowd, Daniel},
  journal={Machine Learning},
  volume={113},
  number={5},
  pages={2351--2403},
  year={2024},
  publisher={Springer}
}

@article{eldan2023s,
  title={Who's Harry Potter? Approximate Unlearning in LLMs},
  author={Eldan, Ronen and Russinovich, Mark},
  journal={arXiv preprint arXiv:2310.02238},
  year={2023}
}

@inproceedings{zhong2023mquake,
  title={Mquake: Assessing knowledge editing in language models via multi-hop questions},
  author={Zhong, Zexuan and Wu, Zhengxuan and Manning, Christopher D and Potts, Christopher and Chen, Danqi},
  booktitle={Conference on Empirical Methods in Natural Language Processing (EMNLP)},
  year={2023}
}

@inproceedings{kumari2023ablating,
  title={Ablating concepts in text-to-image diffusion models},
  author={Kumari, Nupur and Zhang, Bingliang and Wang, Sheng-Yu and Shechtman, Eli and Zhang, Richard and Zhu, Jun-Yan},
  booktitle={IEEE/CVF International Conference on Computer Vision (ICCV)},
  year={2023}
}

@inproceedings{sisa_unlearning,
  title={Machine unlearning},
  author={Bourtoule, Lucas and Chandrasekaran, Varun and Choquette-Choo, Christopher A and Jia, Hengrui and Travers, Adelin and Zhang, Baiwu and Lie, David and Papernot, Nicolas},
  booktitle={IEEE Symposium on Security and Privacy (SP)},
  year={2021},
}

@inproceedings{neel2021descent,
  title={Descent-to-delete: Gradient-based methods for machine unlearning},
  author={Neel, Seth and Roth, Aaron and Sharifi-Malvajerdi, Saeed},
  booktitle={Algorithmic Learning Theory (ALT)},
  year={2021},
}

@inproceedings{wu2020deltagrad,
  title={Deltagrad: Rapid retraining of machine learning models},
  author={Wu, Yinjun and Dobriban, Edgar and Davidson, Susan},
  booktitle={International Conference on Machine Learning (ICML)},
  year={2020},
}

@inproceedings{cao2015towards,
  title={Towards making systems forget with machine unlearning},
  author={Cao, Yinzhi and Yang, Junfeng},
  booktitle={IEEE Symposium on Security and Privacy (SP)},
  year={2015}
}

@misc{ilyas2024mltutorial,
  author = {Andrew Ilyas and Kristian Georgiev and Logan Engstrom and Sung Min (Sam) Park},
  title = {Data Attribution at Scale: ICML 2024 Tutorial},
  year = {2024},
  url = {https://ml-data-tutorial.org/},
  note = {Accessed: 2024-09-24}
}

@article{bae2024training,
  title={Training Data Attribution via Approximate Unrolled Differentation},
  author={Bae, Juhan and Lin, Wu and Lorraine, Jonathan and Grosse, Roger},
  journal={arXiv preprint arXiv:2405.12186},
  year={2024}
}

@article{choe2024your,
  title={What is Your Data Worth to GPT? LLM-Scale Data Valuation with Influence Functions},
  author={Choe, Sang Keun and Ahn, Hwijeen and Bae, Juhan and Zhao, Kewen and Kang, Minsoo and Chung, Youngseog and Pratapa, Adithya and Neiswanger, Willie and Strubell, Emma and Mitamura, Teruko and others},
  journal={arXiv preprint arXiv:2405.13954},
  year={2024}
}

@inproceedings{pregibon1981logistic,
  title={Logistic regression diagnostics},
  author={Pregibon, Daryl},
  booktitle={The Annals of Statistics},
  year={1981}
}

@book{jaeckel1972infinitesimal,
  title={The infinitesimal jackknife},
  author={Jaeckel, Louis A},
  year={1972},
  publisher={Bell Telephone Laboratories}
}

@article{hampel1974influence,
  title={The influence curve and its role in robust estimation},
  author={Hampel, Frank R},
  journal={Journal of the american statistical association},
  volume={69},
  number={346},
  pages={383--393},
  year={1974},
  publisher={Taylor \& Francis}
}

@article{suriyakumar2022algorithms,
  title={Algorithms that approximate data removal: New results and limitations},
  author={Suriyakumar, Vinith and Wilson, Ashia C},
  journal={Advances in Neural Information Processing Systems (NeurIPS)},
  year={2022}
}

@misc{jia2024modelsparsitysimplifymachine,
      title={Model Sparsity Can Simplify Machine Unlearning}, 
      author={Jinghan Jia and Jiancheng Liu and Parikshit Ram and Yuguang Yao and Gaowen Liu and Yang Liu and Pranay Sharma and Sijia Liu},
      year={2024},
      eprint={2304.04934},
      archivePrefix={arXiv},
      primaryClass={cs.LG},
      url={https://arxiv.org/abs/2304.04934}, 
}

@inproceedings{feldmanmemorization,
      title={Does Learning Require Memorization? A Short Tale about a Long Tail}, 
      author={Vitaly Feldman},
      year={2020},
      booktitle={ACM Symposium on Theory of Computing (STOC)},
}

@misc{mahadevan2021certifiablemachineunlearninglinear,
      title={Certifiable Machine Unlearning for Linear Models}, 
      author={Ananth Mahadevan and Michael Mathioudakis},
      year={2021},
      eprint={2106.15093},
      archivePrefix={arXiv},
      primaryClass={cs.LG},
}

@inproceedings{li2022largelanguagemodelsstrong,
      title={Large Language Models Can Be Strong Differentially Private Learners}, 
      author={Xuechen Li and Florian Tramèr and Percy Liang and Tatsunori Hashimoto},
      year={2022},
      booktitle={International Conference on Learning Representations (ICLR)},
}

@article{grosse2023studying,
  title={Studying large language model generalization with influence functions},
  author={Grosse, Roger and Bae, Juhan and Anil, Cem and Elhage, Nelson and Tamkin, Alex and Tajdini, Amirhossein and Steiner, Benoit and Li, Dustin and Durmus, Esin and Perez, Ethan and others},
  journal={arXiv preprint arXiv:2308.03296},
  year={2023}
}

@InProceedings{santurkar2020breeds,
        title={BREEDS: Benchmarks for Subpopulation Shift},
        author={Shibani Santurkar and Dimitris Tsipras and Aleksander Madry},
        year={2021},
        booktitle={International Conference on Learning Representations (ICLR)},
}

@inproceedings{koh2017understanding,
  title={Understanding black-box predictions via influence functions},
  author={Koh, Pang Wei and Liang, Percy},
  booktitle={International conference on machine learning},
  pages={1885--1894},
  year={2017},
  organization={PMLR}
}

@misc{neurips2023unlearning,
  author       = {Eleni Triantafillou and Fabian Pedregosa and Isabelle Guyon and Sergio Escalera and Julio C. S. Jacques Junior and Gintare Karolina Dziugaite and Peter Triantafillou and Vincent Dumoulin and Ioannis Mitliagkas and Lisheng Sun Hosoya and Meghdad Kurmanji and Kairan Zhao and Jun Wan and Peter Kairouz},
  title        = {NeurIPS 2023 Machine Unlearning Challenge},
  howpublished = {\url{https://unlearning-challenge.github.io}},
  year         = {2023},
  note         = {Accessed: 2024-05-29}
}

@inproceedings{engstrom2024dsdm,
      title={DsDm: Model-Aware Dataset Selection with Datamodels}, 
      author={Logan Engstrom and Axel Feldmann and Aleksander Madry},
      booktitle={International Conference on Machine Learning (ICML)},
      year={2024},
}

@misc{pawelczyk2023incontext,
      title={In-Context Unlearning: Language Models as Few Shot Unlearners}, 
      author={Martin Pawelczyk and Seth Neel and Himabindu Lakkaraju},
      year={2023},
      eprint={2310.07579},
      archivePrefix={arXiv},
      primaryClass={cs.LG}
}

@misc{hayes2024inexact,
      title={Inexact Unlearning Needs More Careful Evaluations to Avoid a False Sense of Privacy}, 
      author={Jamie Hayes and Ilia Shumailov and Eleni Triantafillou and Amr Khalifa and Nicolas Papernot},
      year={2024},
      eprint={2403.01218},
      archivePrefix={arXiv},
      primaryClass={cs.LG}
}

@misc{liu2024unlearning,
  title   = {Machine unlearning in 2024},
  author  = {Liu, Ken Ziyu},
  journal = {Ken Ziyu Liu - Stanford Computer Science},
  year    = {2024},
  month   = {Apr},
  url     = {https://ai.stanford.edu/~kzliu/blog/unlearning},
}

@inproceedings{guo2019certified,
title={Certified data removal from machine learning models},
author={Guo, Chuan and Goldstein, Tom and Hannun, Awni and Van Der Maaten, Laurens},
booktitle={International Conference on Machine Learing (ICML)},
year={2019}
}

@InProceedings{neel2021deletion,
title = 	 {Descent-to-Delete: Gradient-Based Methods for Machine Unlearning},
author = {Neel, Seth and Roth, Aaron and Sharifi-Malvajerdi, Saeed},
booktitle = {Proceedings of the 32nd International Conference on Algorithmic Learning Theory (ALT)},
year = 	 {2021},
}

@inproceedings{ginart2019making,
  title={Making AI Forget You: Data Deletion in Machine Learning}, 
  author={Antonio Ginart and Melody Y. Guan and Gregory Valiant and James Zou},
  year={2019},
  booktitle={Proceedings of the 33rd Conference on Neural Information Processing Systems (NeurIPS 2019)}
}

@article{goel2022towards,
title={Towards adversarial evaluations for inexact machine unlearning},
author={Goel, Shashwat and Prabhu, Ameya and Sanyal, Amartya and Lim, Ser-Nam and Torr, Philip and Kumaraguru, Ponnurangam},
journal={arXiv preprint arXiv:2201.06640},
year={2022}
}

@inproceedings{kurmanji2023towards,
  title={Towards Unbounded Machine Unlearning},
  author={Kurmanji, Meghdad and Triantafillou, Peter and Triantafillou, Eleni},
  booktitle={Advances in Neural Information Processing Systems (NeurIPS)},
  year={2024}
}

@inproceedings{carlini2021membership,
title={Membership inference attacks from first principles},
author={Carlini, Nicholas and Chien, Steve and Nasr, Milad and Song, Shuang and Terzis, Andreas and Tramer, Florian},
booktitle={2022 IEEE Symposium on Security and Privacy (SP)},
pages={1897--1914},
year={2022},
organization={IEEE}
}

@inproceedings{graves2021amnesiac,
title={Amnesiac machine learning},
author={Graves, Laura and Nagisetty, Vineel and Ganesh, Vijay},
booktitle={Proceedings of the AAAI Conference on Artificial Intelligence},
year={2021}
}

@article{Gupta2021AdaptiveMU,
  title={Adaptive Machine Unlearning},
  author={Varun Gupta and Christopher Jung and Seth Neel and Aaron Roth and Saeed Sharifi-Malvajerdi and Chris Waites},
  journal={ArXiv},
  year={2021},
  volume={abs/2106.04378},
}

@misc{chien2024stochastic,
      title={Stochastic Gradient Langevin Unlearning}, 
      author={Eli Chien and Haoyu Wang and Ziang Chen and Pan Li},
      year={2024},
      eprint={2403.17105},
      archivePrefix={arXiv},
      primaryClass={cs.LG}
}

@inproceedings{ravfogel2022linear,
  title={Linear adversarial concept erasure},
  author={Ravfogel, Shauli and Twiton, Michael and Goldberg, Yoav and Cotterell, Ryan D},
  booktitle={International Conference on Machine Learning (ICML)},
  year={2022},
}

@article{datamodels,
  author       = {Andrew Ilyas and
                  Sung Min Park and
                  Logan Engstrom and
                  Guillaume Leclerc and
                  Aleksander Madry},
  title        = {Datamodels: Predicting Predictions from Training Data},
  booktitle      = {International Conference on Machine Learning (ICML)},
  year         = {2022},
}

@InProceedings{approximate_data_deletion,
title = {Approximate Data Deletion from Machine Learning Models },
author = {Izzo, Zachary and Anne Smart, Mary and Chaudhuri, Kamalika and Zou, James},
booktitle = {Proceedings of The 24th International Conference on Artificial Intelligence and Statistics (AISTATS)},
year = 	 {2021}
}

@misc{triantafillou_llm_copyright_unlearning,
      title={To Each (Textual Sequence) Its Own: Improving Memorized-Data Unlearning in Large Language Models}, 
      author={George-Octavian Barbulescu and Peter Triantafillou},
      year={2024},
      eprint={2405.03097},
      archivePrefix={arXiv},
      primaryClass={cs.LG},
      url={https://arxiv.org/abs/2405.03097}, 
}

@inproceedings{TRAK,
      title={TRAK: Attributing Model Behavior at Scale}, 
      author={Sung Min Park and Kristian Georgiev and Andrew Ilyas and Guillaume Leclerc and Aleksander Madry},
      year={2023},
      booktitle={International Conference on Machine Learning (ICML)},
}

@inproceedings{li2024wmdpbenchmarkmeasuringreducing,
      title={The WMDP Benchmark: Measuring and Reducing Malicious Use With Unlearning}, 
      author={Nathaniel Li and Alexander Pan and Anjali Gopal and Summer Yue and Daniel Berrios and Alice Gatti and Justin D. Li and Ann-Kathrin Dombrowski and Shashwat Goel and Long Phan and Gabriel Mukobi and Nathan Helm-Burger and Rassin Lababidi and Lennart Justen and Andrew B. Liu and Michael Chen and Isabelle Barrass and Oliver Zhang and Xiaoyuan Zhu and Rishub Tamirisa and Bhrugu Bharathi and Adam Khoja and Zhenqi Zhao and Ariel Herbert-Voss and Cort B. Breuer and Samuel Marks and Oam Patel and Andy Zou and Mantas Mazeika and Zifan Wang and Palash Oswal and Weiran Lin and Adam A. Hunt and Justin Tienken-Harder and Kevin Y. Shih and Kemper Talley and John Guan and Russell Kaplan and Ian Steneker and David Campbell and Brad Jokubaitis and Alex Levinson and Jean Wang and William Qian and Kallol Krishna Karmakar and Steven Basart and Stephen Fitz and Mindy Levine and Ponnurangam Kumaraguru and Uday Tupakula and Vijay Varadharajan and Ruoyu Wang and Yan Shoshitaishvili and Jimmy Ba and Kevin M. Esvelt and Alexandr Wang and Dan Hendrycks},
      year={2024},
      booktitle={International Conference on Machine Learning (ICML)},
}

@misc{goel2024correctivemachineunlearning,
      title={Corrective Machine Unlearning}, 
      author={Shashwat Goel and Ameya Prabhu and Philip Torr and Ponnurangam Kumaraguru and Amartya Sanyal},
      year={2024},
      eprint={2402.14015},
      archivePrefix={arXiv},
      primaryClass={cs.LG},
}

@misc{dou2024avoidingcopyrightinfringementmachine,
      title={Avoiding Copyright Infringement via Machine Unlearning}, 
      author={Guangyao Dou and Zheyuan Liu and Qing Lyu and Kaize Ding and Eric Wong},
      year={2024},
      eprint={2406.10952},
      archivePrefix={arXiv},
      primaryClass={cs.CL},
}

@misc{yao2024largelanguagemodelunlearning,
      title={Large Language Model Unlearning}, 
      author={Yuanshun Yao and Xiaojun Xu and Yang Liu},
      year={2024},
      eprint={2310.10683},
      archivePrefix={arXiv},
      primaryClass={cs.CL},
      url={https://arxiv.org/abs/2310.10683}, 
}

@article{garrigos2023handbook,
  title={Handbook of convergence theorems for (stochastic) gradient methods},
  author={Garrigos, Guillaume and Gower, Robert M},
  journal={arXiv preprint arXiv:2301.11235},
  year={2023}
}

@article{warnecke2021machine,
  title={Machine unlearning of features and labels},
  author={Warnecke, Alexander and Pirch, Lukas and Wressnegger, Christian and Rieck, Konrad},
  journal={arXiv preprint arXiv:2108.11577},
  year={2021}
}

@inproceedings{li2024fast,
  title={Fast-NTK: Parameter-Efficient Unlearning for Large-Scale Models},
  author={Li, Guihong and Hsu, Hsiang and Chen, Chun-Fu and Marculescu, Radu},
  booktitle={Proceedings of the IEEE/CVF Conference on Computer Vision and Pattern Recognition},
  pages={227--234},
  year={2024}
}

@article{hinton2015distilling,
  title={Distilling the Knowledge in a Neural Network},
  author={Hinton, Geoffrey},
  journal={arXiv preprint arXiv:1503.02531},
  year={2015}
}

\newpage
\appendix

\addcontentsline{toc}{section}{Appendix} %
\renewcommand\ptctitle{Appendices}
\part{}
\parttoc
\clearpage

\counterwithin{figure}{section}
\counterwithin{table}{section}
\counterwithin{algorithm}{section}

\section{Pseudocode}

\subsection{Oracle Matching} \label{appendix:oracle_matching}
    \begin{algorithm} 
    \caption{Oracle Matching (\om)}
    \begin{algorithmic}[1] 
        \State \textbf{Input:} Trained model $\theta$; oracle predictions $f^{\text{oracle}}(x)$; fine-tuning set size $r$
        \State \textbf{Output:} Unlearned model $\theta_{UL}$
        \State Initialize $\theta_0 = \theta$
        \For{$t = \{0,...,T-1\}$} \Comment{$T$ epochs}
            \State $S_R' \leftarrow S \setminus  S_F$  \Comment{Sub-sample $r$ points from retain set}
            \State $S_{\text{fine-tune}} = S_F \bigcup S_R'$
            \For{$x \sim S_{\text{fine-tune}}$} \Comment{mini-batch}
                \State $L(\theta_t) = \|f_x(\theta_t) - f^{\text{oracle}}(x)\|^2$ \Comment{Compute loss}
                \State $\theta_{t+1} = \theta_{t} - \eta_t \cdot \nabla_\theta L(\theta_{t})$ \Comment{Perform update with gradient}
            \EndFor
        \EndFor
        \State \textbf{Return} Model $\theta_{UL} = \theta_T$
    \end{algorithmic}
    \end{algorithm}

\subsection{Datamodel Direct} \label{appendix:dm_direct}
\begin{algorithm}
    \caption{\dmdirect }
    \begin{algorithmic}[1]
        \State \textbf{Input:} Trained model $\theta$; datamodels $\beta(x)$ for each $x \in S$; forget set $S_F$
        \State \textbf{Output:} A predictor $h(\cdot): S \mapsto \mathbb{R}^k$
        \State $h(x) := f_x(\theta_0) - \sum\limits_{i \in S_F} \beta_i(x)$
        \State \textbf{End}
    \end{algorithmic}
\end{algorithm}
    
\subsection{Datamodel Matching} \label{appendix:datamodel_matching}

\begin{algorithm}
\caption{Datamodel Matching (\dmm)}
\begin{algorithmic}[1]
    \State \textbf{Input:} Trained model $\theta$; datamodels $\beta(\cdot)$; fine-tuning set size $r$
    \State \textbf{Output:} Unlearned model $\theta_{UL}$
    \State $S_R' \leftarrow S\setminus  S_F$; $S_{\text{fine-tune}} = S_F \bigcup S_R'$
    \State $h \leftarrow \dmdirect(\theta, \beta, S_f)$ \Comment{Simulate oracles with datamodels}
    \For{$t = \{0,...,T-1\}$} \Comment{$T$ epochs}
        \For{$x \sim S_{\text{fine-tune}}$} \Comment{mini-batch}
            \State $L(\theta_t) = \|f_x(\theta_t) - h(x)\|^2$ \Comment{Compute loss}
            \State $\theta_{t+1} = \theta_{t} - \eta_t \cdot \nabla_\theta L(\theta_{t})$ \Comment{Perform update with gradient}
        \EndFor
    \EndFor
    \State \textbf{Return} Model $\theta_{UL} = \theta_T$
\end{algorithmic}
\end{algorithm}

\section{Related work}

We provide a high level overview of prior works most relevant to our setting.
For a more extensive survey of unlearning, see \cite{liu2024unlearning}.

\paragraph{Machine unlearning: goals and evaluations.}
While other lines of work also study unlearning ``concepts'' or ``knowledge''~\citep{ravfogel2022linear,eldan2023s,kumari2023ablating,zhong2023mquake}, we focus on the data-driven notion of unlearning \citep{cao2015towards,ginart2019making, wu2020deltagrad, neel2021deletion, sisa_unlearning}.
Our focus is on approximate unlearning methods. However, other works (e.g., \cite{sisa_unlearning}) aim for {\em exact} unlearning (e.g., by careful data partitioning and ensembling or leveraging differential privacy).
While these approaches come with provable guarantees, they often come at the cost of accuracy~\citep{sisa_unlearning}, so most unlearning algorithms for deep learning are approximate. This approximate nature necessitates empirical evaluations in lieu of a provable guarantee. One line of work adapts membership inference attacks (MIAs) to evaluate machine unlearning~\citep{golatkar2020eternal,goel2022towards,hayes2024inexact}.
Complementary to that, \citet{pawelczyk2024machine} evaluate machine unlearning methods' ability to remove backdoor attacks from the training set. More broadly, evaluation in these setting can be nuanced:  \citet{thudi2022necessity} argue that due to the stochastic nature of deep learning optimization, approximate evaluation of machine unlearning is only well-defined on an algorithmic level, and not on an individual model instance level.

\paragraph{Prior unlearning approaches in deep learning.}
Due to challenges of developing rigorous unlearning methods in non-convex settings, typical approaches involve some form of gradient-based optimization. Strategies include: partial fine-tuning \citep{goel2022towards}, combinations of gradient descent and ascent \citep{kurmanji2023towards}, and sparsity-regularized fine-tuning \citep{jia2024modelsparsitysimplifymachine} among others.
Our approach also employs fine-tuning, but differs primarily in that we address the common problem of ``missing targets.''  
Other approaches employ parameter updates based on a local quadratic approximation \citep{golatkar2020eternal,li2024fast} or influence function approximations \citep{warnecke2021machine}; these can be interpreted as also leveraging different forms of data attribution.
However, our approach is unique in its use of predictive data attribution only as {\em guidance} and still employing fine-tuning, which is a more flexible and robust strategy than direct parameter updates.

The primary method we compare against is SCRUB (and SCRUB+R) as it achieves the current state-of-the-art on strong
unlearning evaluations \cite{hayes2024inexact}. At a high level, SCRUB finetunes the original model to: i) maximize the KL divergence between the probabilities of the original model and the new model on forget points; ii) minimize the KL divergence on retain points; and iii) also minimize test loss. Despite their alternative design choices from other fine-tuning approaches (e.g., use of KL divergence), our analyses suggest that it suffers from similar underlying challenges.

\paragraph{Data attribution.}
Key to our framework is a reduction to the problem of predictive data attribution~\citep{datamodels, TRAK}. More broadly, the problem of attributing model predictions back to training data has been extensively studied in recent machine learning literature \citep{koh2017understanding,TRAK,engstrom2024dsdm,grosse2023studying,choe2024your,bae2024training}, with some of the ideas originating from statistics \citep{hampel1974influence,
jaeckel1972infinitesimal, pregibon1981logistic}. 
For an extensive survey on the topic, refer to~\citet{hammoudeh2024training, ilyas2024mltutorial}.

\paragraph{Model distillation.}
Our approach of fine-tuning on (simulated) oracle predictions has some similarity to a different line of work on {\em knowledge distillation} \citep{hinton2015distilling}, e.g., distilling an existing ``teacher'' model (possibly an ensemble) to a ``student'' model (often smaller). We can cast oracle matching as distilling an oracle model into the current model. The main difference, however, is that in our setting, the model we fine-tune is already trained on the full dataset, and our goal is only to apply a small update to this model.

\section{Experimental setup}
\label{app:exp_details}

\subsection{Training setup}
\paragraph{CIFAR-10. } We use the ResNet-9 architecture.\footnote{\url{https://github.com/wbaek/torchskeleton/blob/master/bin/dawnbench/cifar10.py}}
Models are trained for 24 epochs with SGD with the following configuration: initial learning rate 0.4, cyclic learning rate schedule (with peak at epoch 5), batch size 512, momentum 0.9, and weight decay 5e-4. %

\paragraph{ImageNet Living-17.}
We use the standard ResNet18 architecture from the torchvision library.
Models are trained for 25 epochs using SGD with the following configuration: initial learning rate
0.6, cyclic learning rate schedule (with peak at epoch 12),  batch size 1024, momentum 0.9, weight
decay 5e-4, and label smoothing (with smoothing hyperparameter 0.1).

\subsection{Constructing forget sets}\label{appendix:forget-set-specifics}

We evaluate methods across various types and sizes of forget sets to test the robustness of unlearning. Our selection of unlearning scenarios span both random and non-random forgets of different sizes; that said, we view the non-random sets as practically more interesting. 
Compared to prior work, our target sets are harder to unlearn as we remove a small coherent subpopulation as opposed to an entire class.

\paragraph{CIFAR-10.} We use 9 different forget sets: sets 1,2,3 are random forget sets of sizes 10,100,1000 respectively; sets 4-9 correspond to semantically coherent subpopulations of examples (e.g., all dogs facing a similar direction) identified using clustering methods. 
Specifically, we take a $n \times n$ datamodel matrix constructed by concatenating ``train x train`` datamodels ($n=50,000$). Next, we compute the top principal components (PCs) of the influence matrix and construct the following forget sets:
\begin{enumerate}
    \item Forget set 1: 10 random samples
    \item Forget set 2: 100 random samples
    \item Forget set 3: 500 random samples
    \item Forget set 4: 10 samples with the highest projection onto the 1st PC
    \item Forget set 5: 100 samples with the highest projection onto the 1st PC 
    \item Forget set 6: 250 samples with the highest projection onto the 1st PC and 250 with lowest projection
    \item Forget set 7: 10 samples with the highest projection onto the 2nd PC
    \item Forget set 8: 100 samples with the highest projection onto the 2nd PC
    \item Forget set 9: 250 samples with the highest projection onto the 2nd PC and 250 with the lowest projection.
\end{enumerate}

\paragraph{ImageNet Living-17.} We use three different forget sets: set 1 is random of size 500; sets 2 and 3 correspond to 200 examples from a certain subpopulation (corresponding to a single original ImageNet class) within the Living-17 superclass.

\subsection{Hyperparameters of unlearning algorithms}

Most unlearning algorithms are highly sensitive to the choice of forget set; thus, so for each of the baseline unlearning algorithms, for each forget set, we evaluate over a grid of hyperparameters and report the best \klom score relative to compute time. For \om and \dmm, after the initial grid search, we fix the same hyperparameters for all forget sets and only vary the size of the sampled retain set and the number of epochs in our analysis. Below we report the hyperparameter grid that we search over for each of the unlearning algorithms. 
\begin{enumerate}
    \item \textbf{Gradient Ascent}: Optimized with SGD. Learning rates: $\{$1e-5, 1e-3, 1e-2$\}$. Epochs: $\{1, 3, 5, 7, 10\}$. %
    \item \textbf{Gradient Descent}: Optimized with SGD. Learning rates: $\{$1e-5, 1e-3, 1e-2$\}$. Epochs: $\{1, 3, 5, 7, 10\}$.
    \item \textbf{SCRUB}: Optimized with SGD. Forget batch size: $\{32,64\}$; retain batch size is set to 64. Learning rates: $\{$5e-3, 1e-3, 5e-4, 5e-5$\}$. Maximization epochs (number of epochs in initial GA phase): $\{3,5\}$. Total epochs $\{5,7,10\}$. 
    \item \textbf{Oracle Matching / Datamodel Matching}: Optimized with default Adam\footnote{We also tried SGD, but it made no noticeable difference.} using batch size 512. Retain multipliers (size of the retain set included in fine-tuning relative to forget set size): $\{5, 20, 100, 400\}$. Learning rates: $\{$ 1e-3, 1e-4, 1e-5, 1e-6$\}$. Epochs: $\{1, 2, 4, 8\}$. %
\end{enumerate}

\subsection{Datamodel estimation}
\paragraph{Regression-based.}
We re-train 20,000 models on random 50\% subsets of the full train dataset.
We use the sparse linear regression based solvers from \citet{datamodels} to estimate each datamodel vector, for each target example and target class (logit).
Though our main results are computed with 20,000 models, we find that using just 1,000 models suffice effective unlearning with \dmm (see \Cref{subsec:dm_science}.

\paragraph{TRAK.}
We compute TRAK scores using 300 model checkpoints and 16328 projection dimensions using the code provided in \citet{TRAK}.

\pagebreak
\section{Linear model analysis}

\subsection{Proof of Theorem \ref{thm:convergence_gd_om}}
We restate Theorem \ref{thm:convergence_gd_om} below.
\label{sec:thm_proof}

\thmconvg*

\begin{proof}
    Note that by using the Taylor expansion of the ridge gradient descent
    objective around $\theta_*$, we can rewrite it as follows:
    \begin{align}
        &\|X_R \theta - y_R\|_2^2 + \lambda \|\theta\|^2\\
        &= \|X_R \ultarg - y_{R}\|^2 + \lambda \|\ultarg\|_2^2 
        + \left(\theta - \ultarg\right)^\top \left(X_{R}^\top X_{R} + \lambda I\right) \left(\theta - \ultarg\right) \\
        &= c + \left(\theta - \ultarg\right)^\top \left(X_R^\top X_R + \lambda I\right) \left(\theta - \ultarg\right),
    \end{align}
    where $c = \|X_R \ultarg - y_R\|^2 + \lambda \|\ultarg\|^2$  is a constant independent of $\theta$. 
    Similarly, we can write oracle matching objective as
    \begin{align}
        \|X \theta - X \ultarg\|^2 = \left(\theta - \ultarg\right)^\top 
            \left(X^\top X\right) \left(\theta - \ultarg\right).
    \end{align}
    Note that both the ridge gradient descent and the oracle matching objectives can be written as
    \begin{align}
    \label{eq:gen_objective}
  \left(\theta - \theta_*\right)^T \left(Z^T \ Z\right) \left(\theta - \theta_*\right) + {c} 
    \end{align}
     for some PSD matrix $Z^TZ$ and some constant $c$. 
    Gradient descent update on objective \ref{eq:gen_objective} can be written as:
    \begin{align}
        \theta_t = \theta_{t-1} - 2\eta \ Z^\top Z (\theta_{t-1} - \theta_{*}).
    \end{align}
    Subtracting $\theta_{*}$ from both sides,
    \begin{align}
        \theta_t - \theta_{*} &=  (I - 2\eta Z^\top Z) (\theta_{t-1} - \theta_{*}).
    \end{align}
    Unrolling the recursion, squaring both sides, and simplifying then yields the desired result.
\end{proof}

\subsection{Proof of \cref{thm:som_sep}}
\label{app:som_separation}
\somsgdseparation*
\begin{proof}
    We will begin with a more general setting than the theorem. 
    In particular, consider an arbitrary convex optimization 
    problem of the form
    \[
        \min_\theta f(\theta), \qquad \text{ where } 
        f(\theta) = \frac{1}{n} \sum_{i=1}^n f_i(\theta),
    \]
    where the function $f$ is $\alpha$-strongly convex, 
    and each $f_i$ is $\beta$-smooth.  
    In other words, we have that for any $\theta$ and $\theta'$,
    \begin{align*}
        \langle \nabla f(\theta) - \nabla f(\theta'), \theta - \theta' \rangle
            &\geq \alpha \|\theta - \theta'\| &\text{($f$ is $\alpha$-strongly convex)} \\
        \|\nabla f_i(\theta) - \nabla f_i(\theta')\| 
            &\leq \beta \|\theta - \theta'\|\ \text{ for all } i \in [n]. 
            &\text{(each $f_i$ is $\beta$-smooth)} 
    \end{align*}
    Note that both \gd and \om are both $\alpha$-strongly convex for
    $\alpha_{GD} = \frac{1}{n}(\gamma_{min}(X_R^TX_R)) + \lambda$ and $\alpha_{OM} = \frac{1}{n}\gamma_{min}^+(X^TX)$ respectively. In OM each $f_i(\theta) = \frac{1}{2}(\theta^Tx_i-\theta_*^Tx_i)$ is $\beta$-smooth for $\beta_{OM} = \max_{x}\|x\|_2^2$, since $\nabla^2f_i(\theta) = x_ix_i^{T}$, which has maximum eigenvalue $\|x_i\|_2^2$. Similarly, in GD each $f_i(\theta) = \frac{1}{2}(\theta^Tx_i-y_i)^2 + \frac{\lambda}{2} \|\theta \|_2^2$ is $\beta_{GD} = \max_{x}\|x\|_2^2 + \lambda$ smooth. Finally, we note that if each $f_i(\theta)$ is $\beta$-smooth, then $f(\theta)$ is also $\beta$-smooth. 

    We further define a quantity $\sigma^2_f$ called {\em gradient disagreement},
    measured as 
    \[
        \sigma^2_f := \mathbb{E}_i\left[
            \left\|\nabla f_i(\ultarg) - \mathbb{E}[\nabla f_i(\ultarg)]\right\|^2
        \right] = \mathbb{E}_i\left[
            \left\|\nabla f_i(\ultarg)\right\|^2
        \right],
    \]
    where $\ultarg$ is the optimum of $f$.

    \paragraph{Upper bound for OM.} 
    With these quantities defined, let us begin with the OM upper bound. 
    In fact, this follows directly from 
    a standard SGD convergence proof, e.g., Theorem 5.8 of \citep{garrigos2023handbook},
    restated below:
    \begin{theorem}[Theorem 5.8 of \citep{garrigos2023handbook}]
    \label{gar}
        Suppose $f$ is a $\alpha$-strongly convex 
        sum of $\beta$-smooth convex functions. Consider 
        the sequence of iterates $\{\theta_t\}_{t \in \mathbb{N}}$ generated 
        by stochastic gradient descent with a fixed step size 
        $\eta \in (0, \frac{1}{2\beta})$. For $t \geq 0$, 
        \[
        \mathbb{E}\left[\|\theta_t - \ultarg\|^2\right] \leq 
        (1 - \eta \alpha)^t \|\theta_0 - \ultarg\|^2 + \frac{2\eta}{\alpha} \cdot \sigma^2_f.
        \]
    \end{theorem}
    A few observations conclude the proof. First, any step size 
    $\eta \leq \frac{2}{5\beta_{GD}}$ also satisfies $\eta \leq \frac{1}{2\beta_{OM}}$
    and we can thus apply Theorem~\ref{gar} to both algorithms when  $\eta \leq \frac{2}{5\beta_{GD}}$.
    Second, for oracle matching, we have that 
    \[
        \nabla f_i(\theta) = 2 \left(\bm{x_i}^\top \theta - \bm{x_i}^\top \ultarg\right) \bm{x_i},
    \]
    which means that $\nabla f_i(\ultarg) = 0$ and thus $\sigma_f^2 = 0$, concluding the proof.

    \paragraph{Lower bound for \gd.} We now show that for the same set of learning rates, 
    the stochastic version of ridge gradient descent on the retain set cannot converge 
    faster than $1/t^2$. 
    Key to our analysis will be that, for \gd, the gradient disagreement $\sigma_f^2$ 
    is non-zero, so long as the dataset is non-degenerate. 
    In particular, 
    \begin{align*}
        \sigma^2_f &= \mathbb{E}_i\left[
            \left\|\nabla f_i(\ultarg) \right\|^2
        \right] \\
        &\geq \frac{1}{n} \max_i \left\|\nabla f_i(\ultarg) \right\|^2 \\
        &= \frac{4}{n} \max_i \left\|(\bm{x_i}^\top \ultarg - y_i) \bm{x_i} + \lambda \ultarg \right\|^2.
    \end{align*}
    To see that this is strictly positive, we can proceed by contradiction.
    Suppose that $\sigma^2_f = 0$.
    Observe that if $\bm{x_i}^\top \ultarg = y_i$ for 
    any $i \in [n]$, then the corresponding $\|\nabla f_i(\ultarg)\|^2 = 
    4\lambda^2\|\ultarg\|^2$, and so $\sigma^2_f > 0$. 
    Thus, $\bm{x_i}^\top \ultarg \neq y_i$ for all $i$.
    In this case, however, we must have that 
    \[
        (\bm{x_i}^\top \ultarg - y_i) \bm{x_i} = -\lambda \ultarg,
    \]
    meaning that $\ultarg$ is parallel to $\bm{x_i}$. If $\text{rank}(X) > 1$, 
    this is a contradiction and so $\sigma^2_f > 0$.

    We can now continue with the rest of the proof.
    We start with some algebraic manipulation of the 
    gradient update. In particular, for a random $i \in [n]$,

    \begin{align*}
        \theta_t - \ultarg &= \theta_{t-1} - \ultarg - \eta \nabla f_i(\theta_{t-1}) \\
        \|\theta_t - \ultarg\|^2 &= \|\theta_{t-1} - \ultarg\|^2 
                                - 2\eta \langle\theta_{t-1} - \ultarg, \nabla f_i(\theta_{t-1}) \rangle 
                                + \eta^2 \|\nabla f_i(\theta_{t-1})\|^2.
    \end{align*}
    Taking an expectation conditioned on $\theta_{t-1}$,
    \begin{align*}
        \bbE{ \|\theta_t - \ultarg\|^2 } &= \|\theta_{t-1} - \ultarg\|^2
                                - 2\eta \langle\theta_{t-1} - \ultarg, \nabla f(\theta_{t-1}) \rangle 
                                + \eta^2 \bbE{\|\nabla f_i(\theta_{t-1})\|^2}.
    \end{align*}
    Now, we treat the second and third terms separately. In particular, for the second term,
    \begin{align*}
        2\eta \langle\theta_{t-1} - \ultarg, \nabla f(\theta_{t-1}) \rangle 
        &\leq 2\eta \|\theta_{t-1} - \ultarg\| \|\nabla f(\theta_{t-1}) \| 
            &\text{(Cauchy-Schwarz)}\\
        &= 2\eta \|\theta_{t-1} - \ultarg\| \|\nabla f(\theta_{t-1}) - \nabla f(\ultarg) \| 
            &\text{(Gradient at optimum is zero)}\\
        &\leq 2\eta\beta_{GD} \|\theta_{t-1} - \ultarg\|^2. 
            &\text{(Smoothness)}
    \end{align*}
    For the third term, we use the identity $\|u - v\|^2 \leq 2\|u\|^2 + 2\|v\|^2$,
    which we can rearrange to be $\|u\|^2 \geq \frac{1}{2}\|u - v\|^2 - \|v\|^2$.
    Letting $u = \nabla f_i(\theta_{t-1})$ and 
    $v = \nabla f_i(\theta_{t-1}) - \nabla f_i(\ultarg)$,
    \begin{align*}
        \eta^2 \bbE{\|\nabla f_i(\theta_{t-1})\|^2} 
        &\geq \frac{\eta^2}{2}\bbE{\left\|\nabla f_i(\ultarg)\right\|^2} 
        - \eta^2 \bbE{\left\|\nabla f_i(\theta_{t-1}) - \nabla f_i(\ultarg)\right\|^2} \\
        &\geq \frac{\eta^2}{2}\bbE{\left\|\nabla f_i(\ultarg)\right\|^2} 
        - \eta^2\beta_{GD}^2 \left\|\theta_{t-1} - \ultarg\right\|^2 \\
        &= \frac{\eta^2\sigma^2_f}{2} - \eta^2\beta_{GD}^2 \left\|\theta_{t-1} - \ultarg\right\|^2.
    \end{align*}
    Combining everything so far, 
    \begin{align*}
        \bbE{ \|\theta_t - \ultarg\|^2 } 
        &\geq \left(1 - 2\eta\beta_{GD} -\eta^2\beta_{GD}^2\right)\|\theta_{t-1} - \ultarg\|^2 
                + \frac{\eta^2\sigma^2_f}{2}
    \end{align*}
    Taking an expectation with respect to previous iterates yields
    \begin{align*}
        \bbE{ \|\theta_t - \ultarg\|^2 } 
        &\geq \left(1 - 2\eta\beta_{GD} -\eta^2\beta_{GD}^2\right)^t \|\theta_{0} - \ultarg\|^2 
                + \frac{\eta^2\sigma^2_f}{2} \cdot \sum_{\tau=0}^{t-1} 
                \left(1 - 2\eta\beta_{GD} -\eta^2\beta_{GD}^2\right)^\tau \\
        &\geq \left(1 - 2\eta\beta_{GD} -\eta^2\beta_{GD}^2\right)^t \|\theta_{0} - \ultarg\|^2 
                + \frac{\eta^2\sigma^2_f}{2} \\
        &\geq \left(1 - \frac{12}{5}\eta\beta_{GD}\right)^t \|\theta_{0} - \ultarg\|^2 
                + \frac{\eta^2\sigma^2_f}{2} \qquad \text{since $\eta \leq \frac{2}{5\beta_{GD}}$}.
    \end{align*}
    For ease of notation, let $C = \frac{12}{5}\beta_{GD}$. 
    Note that $\log(1-x) \geq \frac{x^2 - 2x}{1-x}$ for $x < 1$, and so
    \begin{align}
        \label{eq:error_thm}
        \bbE{ \|\theta_t - \ultarg\|^2 } 
        &\geq \exp\left(
            t\cdot (-C\eta) \frac{2 - C\eta}{1 - C\eta}
        \right) \|\theta_{0} - \ultarg\|^2 + \frac{\eta^2\sigma^2_f}{2}.
    \end{align}
    We now derive the learning rate $\eta$ that minimizes (\ref{eq:error_thm}), and show that at this optimal learning rate  (\ref{eq:error_thm}) $= \Omega(1/t^2)$. Note we can see this by inspection even without the formal derivation, because for (\ref{eq:error_thm}) to be $O(\frac{1}{t^2})$ we need $\eta = O(\frac{1}{t})$ so that the right hand term is $O(1/t^2)$, which forces the first term $\exp\left(t\cdot (-C\eta) \frac{2 - C\eta}{1 - C\eta}\right) = O(1)$. Now, to minimize the right hand side above with respect to $\eta$, we take the derivative and set to zero (note that at the extreme points $\eta = 0$ and 
    $\eta = \frac{2}{5\beta_{GD}}$ we are left with a constant amount of error).
    The result of this calculation is the fixed learning rate that optimizes
    the error at time $t$.
    Again for ease of notation, let $g(\eta) = \frac{2 - C\eta}{1 - C\eta}$, 
    so that
    \begin{align*}
        0 &= \frac{d}{d\eta} \left[
            \exp\left( t\cdot (-C\eta) g(\eta)
            \right) \|\theta_{0} - \ultarg\|^2 + \frac{\eta^2\sigma^2_f}{2}
        \right]  \\
        &= -Ct\left(g(\eta) + \eta g'(\eta)\right) \exp\left( t\cdot (-C\eta) g(\eta)
        \right) \|\theta_{0} - \ultarg\|^2 + \eta\sigma^2_f \\
        \eta\sigma^2_f &= Ct\left(g(\eta) + \eta g'(\eta)\right) \exp\left( t\cdot (-C\eta) g(\eta)
        \right) \|\theta_{0} - \ultarg\|^2  \\
        \log(\eta\sigma_f^2) &= \log(t) + \log(C(g(\eta) + \eta g'(\eta))\|\theta_0 - \ultarg\|) 
        - Ct\eta \cdot g(\eta)  \\
        Ct\eta \cdot g(\eta) - \log(t)
        &= \log(C(g(\eta) + \eta g'(\eta))\|\theta_0 - \ultarg\|) - \log(\eta\sigma_f^2)  \\
        Ct\eta \cdot g(\eta) 
        &\geq \log(C(g(\eta) + \eta g'(\eta))\|\theta_0 - \ultarg\|) - \log(\eta\sigma_f^2)  \\
        &\geq \log\left(\frac{C(g(\eta) + \eta g'(\eta))\|\theta_0 - \ultarg\|}{\sigma_f^2}\right) 
        + \log(1/\eta)  \\
        t &\geq \frac{\log\left(\frac{C(g(\eta) + \eta g'(\eta))\|\theta_0 - \ultarg\|}{\sigma_f^2}\right) 
        + \log(1/\eta)}{C\eta \cdot g(\eta)} 
    \end{align*}
    Now, by definition of $g(\eta)$, we have that for $\eta \in (0, \frac{2}{5\beta_{GD}})$, 
    $g(\eta) \leq 26$ and $g(\eta) + \eta \cdot g'(\eta) \geq 2$. Thus:
    \begin{align*}
        t &\geq \frac{\log\left(\frac{2C\|\theta_0 - \ultarg\|}{\sigma_f^2}\right) 
        + \log(1/\eta)}{26 C\eta} \\
        26Ct &\geq \frac{\log\left(\frac{2C\|\theta_0 - \ultarg\|}{\sigma_f^2}\right) 
        + \log(1/\eta)}{\eta}
    \end{align*}
    Using the fact that $\log(1/x) \geq 1 - x$ for $x \in (0, 1)$ yields:
    \begin{align*}
        26Ct &\geq \frac{\log\left(\frac{2C\|\theta_0 - \ultarg\|}{\sigma_f^2}\right) 
        + 1}{\eta} - 1 \\
        \eta &\geq
        \frac{\log\left(\frac{2C\|\theta_0 - \ultarg\|}{\sigma_f^2}\right) + 1}{26Ct + 1}.
    \end{align*}
    Plugging this result into \eqref{eq:error_thm} yields the desired $\Omega(1/t^2)$ 
    lower bound.
\end{proof}

\subsection{Experiment details}
\label{sec:linear_exp_details}
\paragraph{Details of unlearning algorithms}

We consider the various iterative algorithms for unlearning starting from $\theta_{\text{full}}$. For all of them, we consider their full-batch as well as the stochastic version. For the stochastic versions, we use a mini-batch size of $5$. We search for the learning rate from $\{10, \frac{10}{2}, \frac{10}{2^2}, \cdots, \frac{10}{2^{20}}\}$. We describe the algorithms below:
\begin{enumerate}
    \item \textbf{Ridge Gradient Descent (\gd)}: This involves minimizing the ridge regression objective with the retain set points $(X_{\text{retain}}, y_{\text{retain}})$ using gradient descent. 
    
    \item \textbf{Ridge Gradient Descent + Ascent (\gd+GA)}: This method aims to incorporate forget set points in the ridge gradient descent updates. Each step involves moving in a direction that is a linear combination of the gradient descent step on the retain set and the gradient ascent step on the forget set. That is, we set 
    \[
    \theta_{t} = \theta_{t-1} - \eta (\text{grad}_\text{retain}(\theta_{t-1}) - \alpha * {\text{grad}_\text{forget}}(\theta_{t-1})).
    \]
    Here, $\text{grad}_\text{retain}(\theta) = 2X_{\text{retain}}^T(X_{\text{retain}} \theta - y_{\text{retain}}) + 2\lambda \theta$ and $\text{grad}_\text{forget}(\theta) = 2X_{\text{forget}}^T(X_{\text{forget}} \theta - y_{\text{forget}})$. We do a hyperparameter search for $\alpha$ in $\{0.01, 0.1, 1, 10 \}$. 
    
    In the stochastic setting, in each update step, we draw minibatch  points uniformly at random from the full dataset, and calculate the ascent step term only if the drawn points include points from the forget set.
    
    \item \textbf{Oracle Matching (OM)}: Here, we assume oracle access to predictions made using the optimal model $\theta_{*}$. This method involves using gradient descent to minimize the  the squared error from the oracle predictions on the full dataset: 
    $||X_{\text{full}}\theta  - X_{\text{full}}\theta_{*}||_2^2$. We include a full algorithm of this in Algorithm \ref{algo:oracle_matching}.

    \item \textbf{Oracle Matching on retain set (OM retain set)}: This involves using gradient descent to minimize the squared error from the oracle predictions only on the retain set points: $||X_{\text{retain}}\theta  - X_{\text{retain}}\theta_{*}||_2^2$.
\end{enumerate}

\paragraph{Slow convergence with stochastic gradient descent.}

In Section \ref{sec:linear_models}, we saw that in the stochastic setting, OM converges much faster than \gd, even when $\lambda$ is large. Here, we dig deeper into the large $\lambda$ experiment considered in Section \ref{sec:linear_models} to understand this. We observe that stochastic \gd remains stable only at small learning rates with large $\lambda$, which results in slower progress. Specifically, while the optimal learning rate for full-batch \gd is similar in both large and small $\lambda$ regimes, for stochastic \gd, it is about 100 times smaller in the large $\lambda$ regime. Using a higher learning rate for stochastic \gd in the large $\lambda$ regime leads to instability and non-convergence, an issue not seen with stochastic OM.

In each update step, stochastic OM adjusts the model parameters only within the span of the random subset of points used for that update. On the other hand, stochastic \gd decays the model parameters in other directions due to the regularization term. Although we want the parameters to decay in the subspace orthogonal to the span of the retain set, stochastic \gd also decays the parameters in directions spanned by the retain set points that are not part of the current update set. As a result, using a high learning rate for stochastic \gd in the large $\lambda$ regime disrupts parameters in the span of the retain set. Therefore, while increasing $\lambda$ improves full-batch \gd's convergence speed and could potentially make it faster than OM, this advantage does not apply to stochastic \gd, which has to use a much smaller learning rate.

In Figure \ref{fig:lr_sgd_large_lambda_ret}, we illustrate how stochastic \gd at high learning rates disrupts the model parameters within the span of the retain set. We plot the progression of relative squared distance to the optimal unlearned parameter within the span of retain set points, $\frac{||P_{\text{retain}}(\theta_t - \theta_{})||_2^2}{||\theta_\text{full} - \theta_{*}||_2^2}$. Here $\theta_t$ is the iterate at time $t$, $\theta_{\text{full}}$ is the model trained on the full dataset, $\theta_*$ is the optimal unlearned model, and $P_{\text{retain}}$ is the projection matrix onto the retain set span. We show this for iterates with the optimal learning rate, as well as for learning rates 4 times faster and 4 times slower. As the learning rate increases, the iterates tend to diverge in the span of the retain set.

In Figure \ref{fig:lr_sgd_large_lambda_ret_orth}, we plot the progression of relative squared distance to the optimal unlearned parameter orthogonal in the subspace orthogonal to the retain set points, $\frac{||P_{\text{orth-retain}}(\theta_t - \theta_{*})||_2^2}{||\theta_\text{full} - \theta_{*}||_2^2}$, where $P_{\text{orth-retain}}$ is the projection matrix for the orthogonal subspace.  Here, we observe that increasing the learning rate beyond the optimal rate leads to faster convergence. Thus, while increasing the learning rate beyond the optimal value accelerates convergence in the subspace orthogonal to the retain points, it harms progress in the span of the retain set points. Therefore, in the large $\lambda$ regime, stochastic \gd must operate at small learning rates to avoid disrupting the model parameters in the span of the retain set.

\begin{figure}[H]
 \centering
    \begin{tabular}{cc}
        \begin{subfigure}[b]{0.45\textwidth}
         \captionsetup{skip=0pt} %
            \includegraphics[width=\textwidth]{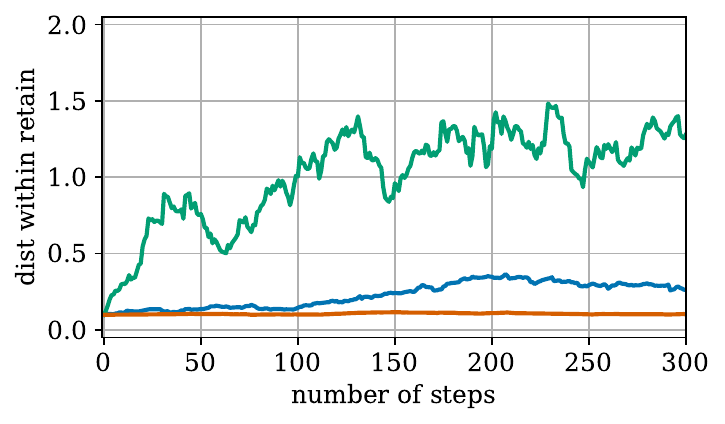}
            \caption{}
            \label{fig:lr_sgd_large_lambda_ret}
        \end{subfigure} &
         \begin{subfigure}[b]{0.45\textwidth}
         \captionsetup{skip=0pt} %
            \includegraphics[width=\textwidth]{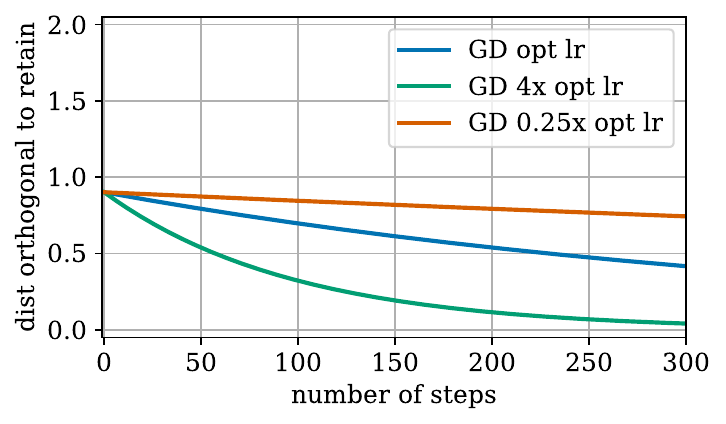}
            \caption{}
    \label{fig:lr_sgd_large_lambda_ret_orth}
        \end{subfigure}\\       
    \end{tabular}
    \caption{Stochastic Gradient Descent performance in the large $\lambda$ regime with varying learning rates. (a) The y-axis shows the relative squared distance to the optimal unlearned parameter within the span of retain set points, $\frac{||P_{\text{retain}}(\theta_t - \theta_{})||_2^2}{||\theta_\text{full} - \theta_{*}||_2^2}$, where $\theta_t$ is the iterate at time $t$, $\theta_{\text{full}}$ is the model trained on the full dataset, $\theta_*$ is the optimal unlearned model, and $P_{\text{retain}}$ is the projection matrix onto the retain set span. Larger learning rates lead to divergence within this span. (b) The y-axis shows the relative squared distance to the optimal unlearned parameter in the subspace orthogonal to the retain set points, $\frac{||P_{\text{orth-retain}}(\theta_t - \theta_{*})||_2^2}{||\theta_\text{full} - \theta_{*}||_2^2}$, where $P_{\text{orth-retain}}$ is the projection matrix for the orthogonal subspace. Larger learning rates result in faster convergence in this subspace.}
    \label{fig:lr_sgd_large_lambda}
\end{figure}

\subsection{Additional experiments}
\label{sec:linear_exp_examples}

In \Cref{sec:linear_models}, we discussed an example with linear models that highlighted the qualitative differences between oracle matching and other unlearning methods. Here, we show the comparison for another example with different covariance structure. 
We draw $100$ training points $(x_i, y_i)$ where $x_i$ are drawn i.i.d. from $N(0, \Sigma)$ in $400$ dimensions where covariance matrix $\Sigma = diag(1, 1/2, 1/3, \cdots, 1/400)$ ( 
$diag(.)$ represents a diagonal matrix with the specified entries on the diagonal). $y_i = \theta^T x_i + \epsilon_i$, where $\theta$ is drawn from $N(0, I)$ and $\epsilon_i$ is drawn from $N(0, 1/4)$. These $100$ points form $(X_\text{full}, y_\text{full})$. We fit these points to minimize the ridge regression objective with $\lambda = 1/4$ (for the small $\lambda$ case) or $\lambda = 5$ (for the large $\lambda$ case), to obtain the  model $\theta_\text{full}$. Here $\lambda = 1/4$ is the $\lambda$ value that minimizes the expected squared prediction error. We want to unlearn $5$ training points chosen uniformly at random. We show the performance of various methods (in both the stochastic and full-batch setting with small and large $\lambda$) in Figure \ref{fig:skewed_cov_linear_model}. Even here, we obtain the same qualitative patterns as in Figure \ref{fig:skewed_cov_linear_model}.

\begin{figure}[H]
 \centering
    \includegraphics[width=0.8\textwidth]{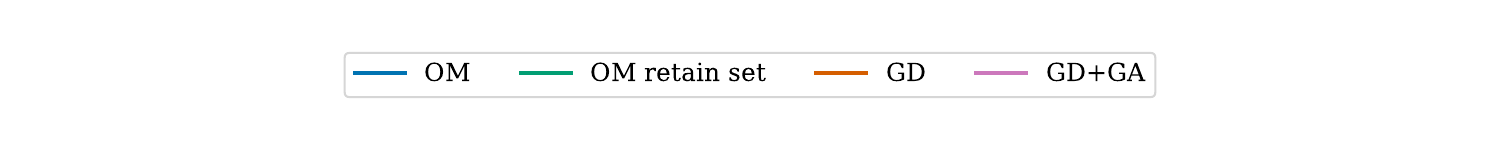} %
    \begin{tabular}{cc}
        
        \begin{subfigure}[b]{0.45\textwidth}
         \captionsetup{skip=0pt} %
            \includegraphics[width=\textwidth]{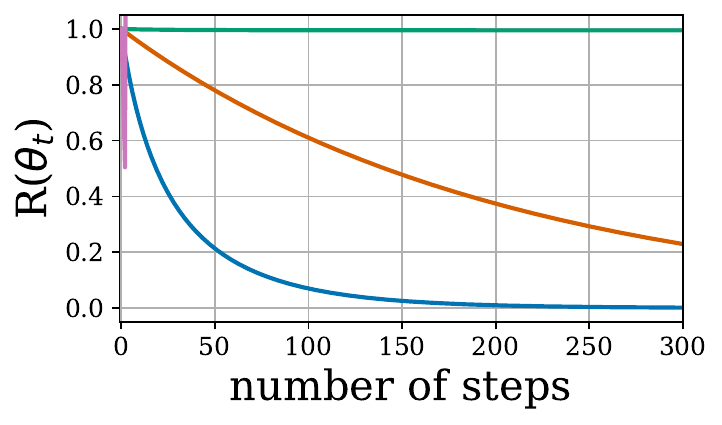}
            \caption{small $\lambda$, full batch}
        \end{subfigure} &
         \begin{subfigure}[b]{0.45\textwidth}
         \captionsetup{skip=0pt} %
            \includegraphics[width=\textwidth]{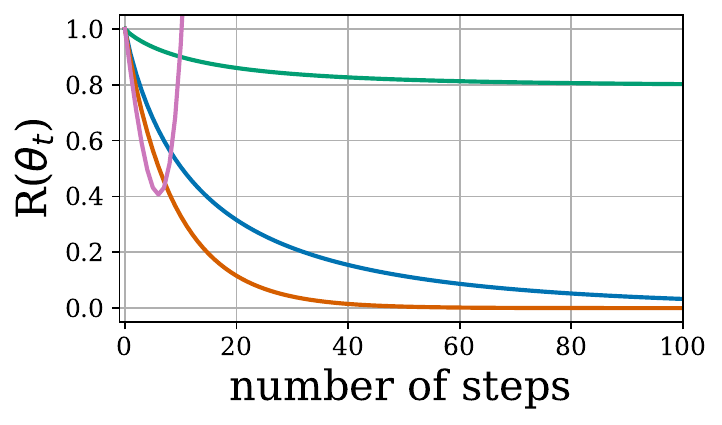}
            \caption{large $\lambda$, full batch}
        \end{subfigure}\\

        \begin{subfigure}[b]{0.45\textwidth}
         \captionsetup{skip=0pt} %
            \includegraphics[width=\textwidth]{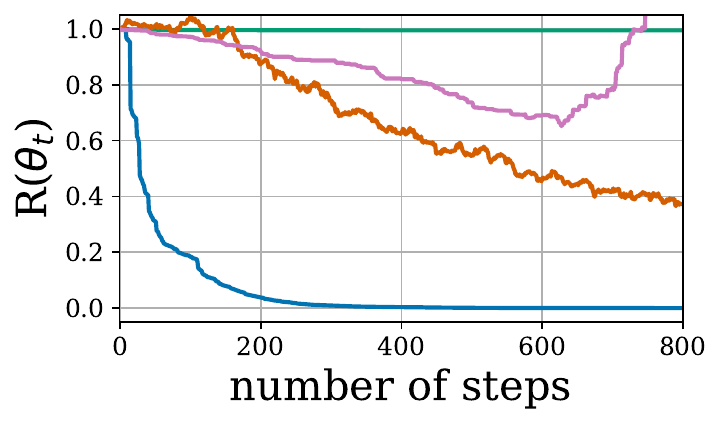}
            \caption{small $\lambda$, stochastic}
        \end{subfigure} &
        \begin{subfigure}[b]{0.45\textwidth}
         \captionsetup{skip=0pt} %
            \includegraphics[width=\textwidth]{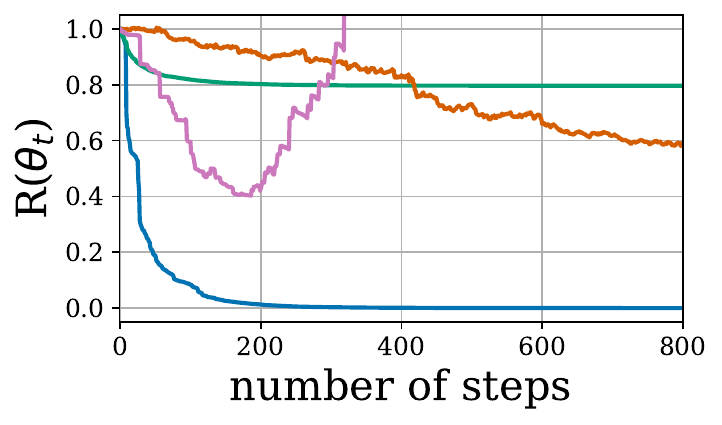}
            \caption{large $\lambda$, stochastic}
        \end{subfigure} \\
        
    \end{tabular}
    \caption{Another example comparing the performance of various unlearning methods with linear models.}
    \label{fig:skewed_cov_linear_model}
\end{figure}

\newpage

\section{Unlearning evaluation}
\label{app:eval_details}

\subsection{KL Divergence of Margins (\klom) } \label{appendix:klom}

Below we formally define the \klom evaluation, which computes the distance between the distribution of outputs for unlearned models and re-trained models.  For output, we use the classification margin. We also include a visual representation of our algorithm in Figure \ref{fig:klom-diagram}.

\begin{algorithm}
\caption{\klom} \label{alg:klom}
    \begin{algorithmic}[1]
    \State  \textbf{Input} Number of models $N$, 
    dataset $D$, forget set  $F \subseteq D$, 
    retain set $R\subseteq D$ (such that $D = F \cup R$), and a validation dataset $V$, training algorithm $A$, unlearning 
    algorithm $U$, margin function $\phi$, and histogram function $H(S) $ .
    \State Train \(N\) models (\textit{Oracles}) on the entire dataset, excluding the forget set. $\Theta^o = \{A(D\setminus F) \}$, $|\Theta^o| = N$.
    \State Train \(N\) models (\textit{Unlearned-models}) on the  entire dataset, then for each model unlearn forget set $F$, $\Theta^f = \{U(A(D), F) \}$, $|\Theta^f| = N$.
    \State Initialize a vector of results with all zeros $\vec{r}$. 
    \For{each point \(x\) in \{Forget, Retain, Validation\} set}
        \State Compute the margins for each oracle $M_o = \{\phi(\theta^o_i(x)) | \theta^o_i \in \Theta^o \}$.
        \State Compute the margins for each unlearned-model $M_f = \{\phi(\theta^f_i(x)) | \theta^f_i \in \Theta^f \}$
        \State Assign $\vec{r}[x] = KL(Hist(M_o), Hist(M_f))$
    \EndFor
    \State return $\vec{r}$
    \end{algorithmic}
\end{algorithm}

 \noindent Note that in order to approximate the KL divergence, we compute a histogram $H(S)$ that takes a set of real numbers and returns an empirical probability distribution by truncating and binning samples from $S$.

\begin{figure}[H]
    \centering
    \includegraphics[width=0.5\linewidth]{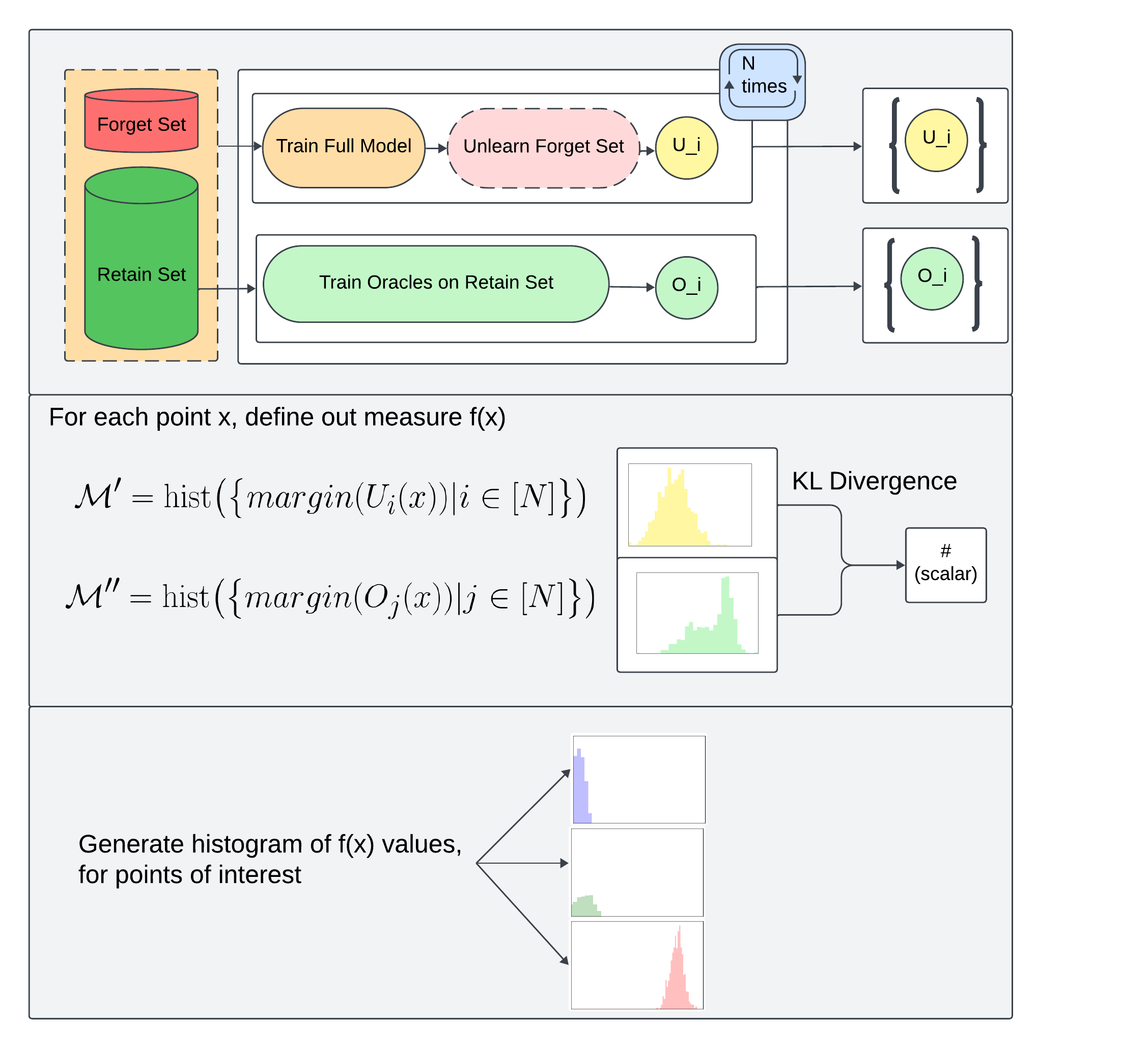}
    \caption{Visual diagram of \klom }
    \label{fig:klom-diagram}
\end{figure}

In practice, we compute \klom using $N=100$ oracles compared to $N=100 $ unlearned models. For each point we evaluate we clip the $N$ margins to a range of $[-100,100]$ to exclude outliers (some methods, like SCRUB, result in extremely large margins) into 20 bins. Then the KL-divergence is computed between the binned histogram of the oracles and the binned histogram of the unlearned models. For any region of the support where one histogram has no support, and the other does, the bin with no support is set to a non-zero probability of $\epsilon= 10^{-5}$. We note, that for this reason, \klom scores are artificially rescaled and capped at value around $\approx 12$. This cap can be changed by changing the number of bins or the minimum value $\epsilon$. 

\subsection{\ulira} \label{appendix:ulira}

At a high-level \ulira measures the distinguishability of predictions of an
unlearned model from that of retrained models based on adapting membership inference attacks. %
Implementing the \Cref{algo:ulira}, as written, would be computationally infeasible in most cases, as it involves unlearning and retraining $T$ times for each point $(x,y)$. In practice, \citet{hayes2024inexact} computes \ulira using the method \Cref{alg:eff_ulira}.

Omitting a few details, the main idea is as follows: consider an unlearning
    algorithm $\mathcal{U}$ and a training algorithm $\algo$.
    Start from a model $\theta_0$ trained on a random subset $S \sim \mathcal{P}^n$, and unlearn one specific (random) forget set, to produce $\theta_{F}$. Next, construct a collection of shadow models by that producing many unlearned models from random training sets and random forget sets. Lastly for a collection of points in the retain set and in the forget set ($\{x| x \in S_F \cup S \}$) compare how $\theta_{F}$ compares to the subset of shadow models that never-saw-$x$ distribution of margins for models that unlearned-$x$. %
    Now if $\mathcal{U}$ perfectly unlearns $S_F$, then $\theta_0$ no longer depends on $S_F$, and so even \emph{conditioned} on $\theta_0$ the marginal distribution of $x \in S_F|\theta_0$ is still $\mathcal{P}$; the same distribution as any $x \in S_V$. Operationalizing this intuition, \ulira with probability $\frac{1}{2}$ draws either $x \in S_F$ or $x \in S_V$, and measures the output $y = f_x(\theta_0)$. An (optimal) adversary observes $y$, and tries to guess
    whether the corresponding $x$ was an unlearned point or a validation point, e.g. whether $x \in S_F$ or $S_V$. More generally, if $\mathcal{U}$ is an $(\epsilon, \delta)$-unlearning
    algorithm, the two distributions $y|x\in S_F, \theta_0$ and $y|x\in S_V, \theta_0$ would be
    $(\epsilon, \delta)$-indistinguishable by post-processing, and so even the optimal adversary
    couldn't have accuracy greater than $\frac{1}{2}e^{\epsilon} + \delta$.
    The optimal adversary can be implemented by training models 
    with/without $x$, 
    unlearning $S_F$, and then measuring the output $y$. 
    For a more detailed description of \ulira and overview of similar MIA-based approaches, we refer the reader to
    \cite{hayes2024inexact} (Section 4.2), 
    and we include the pseudocode for the computationally efficient
    version of this evaluation \EfficientULIRA in Appendix \ref{appendix:ulira}.

\begin{algorithm} 
    \caption{U-LiRA (LiRA adapted for machine unlearning) \citep{hayes2024inexact}}\label{algo:ulira}
    \textbf{Args:} model parameters to evaluate $\theta^*$, learning algorithm $A$, unlearning algorithm $U$, number of shadow models $T$, example $(x, y)$, logit function $\phi$, function that returns probabilities $f(\cdot, \theta)$ given model parameters $\theta$. \\
    \textbf{Observations:} $O \leftarrow \{\}, \hat{O} \leftarrow \{\}$ \\
    \textbf{while} $t \leq T$ \textbf{do}
    \begin{itemize}
        \item $D \leftarrow$ sample a dataset that includes $(x, y)$
        \item $\theta^0 \leftarrow A(D)$ \textit{train a model}
        \item $\theta' \leftarrow U(\theta^0, (x, y))$ \textit{unlearn $(x, y)$}
        \item $\theta'' \leftarrow A(D\setminus(x, y))$ \textit{retrain without $(x, y)$}
        \item $O[t] \leftarrow \phi(f(x; \theta'))$
        \item $\hat{O}[t] \leftarrow \phi(f(x; \theta''))$
    \end{itemize}
    \textbf{end while} \\
    $\mu, \sigma \leftarrow \textit{fit Gaussian}(O)$ \\
    $\hat{\mu}, \hat{\sigma} \leftarrow \textit{fit Gaussian}(\hat{O})$ \\
    $o \leftarrow \phi(f(x, \theta^*))$ \\
    $p_{\text{member}} \leftarrow \frac{N(o;\mu,\sigma^2)}{N(o;\mu,\sigma^2) + N(o;\hat{\mu},\hat{\sigma}^2)}$ \\
    \textbf{if} $p_{\text{member}} \geq \frac{1}{2}$ \textbf{then}
    \begin{itemize}
        \item \textbf{return} Predict $(x, y)$ is a member of training
    \end{itemize}
    \textbf{else}
    \begin{itemize}
        \item \textbf{return} Predict $(x, y)$ is not a member of training
    \end{itemize}
\end{algorithm}

\begin{algorithm}
\caption{Sub algorithm: Membership Prediction}
\label{alg:membership}
    \begin{algorithmic}[1]
    \State \textbf{Input}: Sets $A,B$ of real-valued numbers, and point $x\in \mathbb{R}$.
    \State $(\mu, \sigma \leftarrow \textrm{fit Gaussian}(A)$
    \State $(\hat{\mu}, \hat{\sigma} \leftarrow \textrm{fit Gaussian}(B)$
    \State \(p_{\textrm{member}} \leftarrow \frac{\mathcal{N}(x; \mu, \sigma^2)}{\mathcal{N}(x; \mu, \sigma^2) + \mathcal{N}(x; \hat{\mu}, \hat{\sigma}^2)}\)
    \If{\(p_{\text{member}} > \frac{1}{2}\)}
        \State \textbf{return} Predict \((x, y)\) is a \textit{member} of set $A$
    \Else
        \State \textbf{return} Predict \((x, y)\) is \textit{not a member} of set $A$
    \EndIf
    \end{algorithmic}
\end{algorithm}

\begin{algorithm}
\caption{\EfficientULIRA }
\label{alg:eff_ulira}
\begin{algorithmic}[1]
\State \textbf{Input} Number of base models N, set of forgettable points $F$ (default: all points of class 5), Number of random forget sets per base model $n_f$,  size of each forget set $m$ (default 200)

\State Train $N$ base models on random 50\% subsets of the dataset

\For{each base model \(\mathcal{\theta}_i\)}
    \State construct $n_f$ random forget sets of size $m$, denoted $F_{i,j}$.
    \For{each random forget set \(F_{ij}\)}
        \State Unlearn forget set \(F_{ij}\)
    \EndFor
\EndFor
\State Split the $n_f \cdot N$ unlearned models into two sets, Shadow models \(S\) and Target models \(T\)
\State initialize accuracy vector $\vec{a}\in \mathbb{R}^{|T|}$, with all 0's.
\For{each target model \(\mathcal{\theta_t} \in T\)}
    \State Construct $D_f$ from $m$ from the $m$ points that $\theta_t$ unlearned
    \State Construct $D_v$ from $m$ from the $m$ points that $\theta_t$ was not trained on.
    \State let $D= D_f \cup D_v$ 
    \State Let $c=0$
    \For{each point \(x \in D\)}
        \State Let \(S_A\) be the set of shadow models that were trained on \(x\).
        \State Let \(S_B\) be the set of shadow models that unlearned \(x\)
        \State Construct sets $A = \{\theta'(x)| \theta' \in S_A \}$, $B = \{\theta'(x)| \theta' \in S_B \}$
        \State Run sub-algorithm Membership Prediction \ref{alg:membership} with inputs $(A,B,x)$, returning $l \in \{1,0\}$
        \If ($l=1$ and $x \in D_f$) OR   ($l=0$ and $x \in D_r$)
            \State $c +=1$
        \EndIf
        \State 
    \EndFor
    \State Average the model accuracy across all predictions in $D$, $a_t = c / (2m)$
\EndFor
\State Average the accuracy-per-model over all the target models Return $\textrm{mean}(\vec{a})$
\end{algorithmic}
\end{algorithm}

In the original \ulira paper \citep{hayes2024inexact}, they report results for Efficient \ulira for $N = 256$, forgettable points $F$ all points of class $5$, $n_f=40$ random forget sets per base model, and each forget set $m =20$. The $N$ base models that they use are trained on ResNet-18 for 100 epochs and so are highly overparameterized.

For our evaluation, \ulira paper \citep{hayes2024inexact}, we run Efficient \ulira for N= 50, $n_f=40$ random forget sets per base model, and then we vary the training setting, the forget size $n_f$, and total set of forgettable points $F$. Specifically, we evaluate three settings:
\begin{enumerate}
    \item ResNet-18 trained for 100 epochs, evaluated on forget sets of size 200, with the forgettable points $F$ being 1000 points in class 5.
    \item ResNet-9 trained for 25 epochs, evaluated on forget sets of size 200, with the forgettable points $F$ being 1000 points in class 5.
    \item ResNet-9 trained for 25 epochs, evaluated on forget sets of size 50, with the forgettable points $F$ being 500 random points in the dataset.
\end{enumerate}

\subsection{Comparing \ulira to \klom} \label{appendix:ulira_vs_klom}

We elaborate on why \klom is a better metric than existing evaluations, including \ulira. The advantages of \klom are:

\begin{enumerate}
    \item \klom requires fewer unlearning trials (on the order of 100) than \ulira (which is generally on the order of 2000).
    \item \klom returns a distribution of differences, rather than a binary assignment of if one particular model was more like an unlearned model or a retrained model. This is valuable because it informs how a method unlearns individual points (e.g. is bad on average, or just bad on specific points)
    \item \klom does not assume the margins can be fit well by a Gaussian. Empirically, for \ulira, we find this is a reasonable but imperfect assumption, and it is currently unclear how much error is introduced by this assumption.
    \item \klom has the capacity to look at an unlearning algorithm's ability to handle coherent sets of points, not just random subsets of a set of forgettable points.
    \item \ulira does not capture closeness to unlearned model, and thus one can force \ulira score to go down (implying better unlearning) by having the unlearning method destroy the original model, thus \ulira scores must be traded-off against an accuracy drop. \klom measures distance directly, and thus unlearning can be evaluated with a single measure. We expand on this point below. And is illustrated by figure \ref{fig:ulira_vs_acc_plot}.
\end{enumerate}

\begin{figure} 
    \centering
    \includegraphics[width=0.65\linewidth]{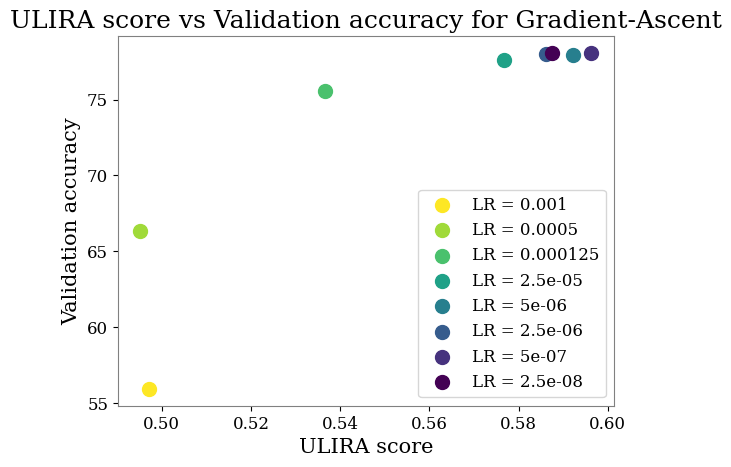}
    \caption{\ulira accuracy vs Validation accuracy of the unlearned model, for Gradient Ascent method, for a model trained on CIFAR-10. Unlearning occurs on a random forget sets of size 200 in class 5. Observe that it is possible to achieve nearly perfect \ulira accuracy (50\%) by simply dropping the validation accuracy of the model by 20\%. With Gradient Ascent specifically, this is achieved by increasing the learning rate beyond the optimal learning rate, causing the gradient ascent updates to be too significant.}
  \label{fig:ulira_vs_acc_plot}
\end{figure}

\paragraph{A toy example where \ULIRA fails.}
Consider your unlearning method returns a constant function $f$ (e.g. such that that the margin is always some constant, e.g. $7$). In such a situation, the distribution of unlearned models will look radically different from the distribution of models that were fully retrained; in theory, a good unlearning measure should return that this is a bad unlearning method. However, \ulira will actually return a nearly perfect unlearning score ($\approx 50\%$). In \ulira, one does 2 sets of likelihood ratio tests, first on points the unlearned model unlearned, then on points the unlearned model never saw.
The first set of likelihood ratio tests, \ulira will get 100\% (or nearly), because at margin is constant ($7$),  and so it will be at the unlearned-models peak; thus we get 100\% accuracy on these points. 
Now, on set 2, \ulira will get 0\% (or nearly), because \ulira compute the likelihood ratio at the margin of the unlearned model's prediction, which in this case is a constant ($7$), thus it will still be at the unlearned-models peak. Thus, \ulira will not predict that anything is ``not-in-the-training-set,'' achieving 0\% on this set. 
However, the problem is that 100 +0 /2 = 50\%, which appears to be a perfect unlearning score, but in reality is just a classifier that thinks everything is from an unlearned model.

\paragraph{Advantages of \ulira:}
\begin{enumerate}
    \item It is worth noting that our measure requires binning, and some assumptions there. \ulira does not (because it makes the Gaussian approximation assumption)
    \item \ULIRA reports an accuracy for one particular model, and then averages that across multiple unlearning models. This correctly captures individual model performance, rather than aggregate unlearning algorithm performance.
\end{enumerate}

\subsection{Sensitivity of unlearning to models and forget sets} \label{appendix:forget_set_sensitivity}

\begin{figure}[!hb]
    \centering
    \includegraphics[width=0.85\linewidth]{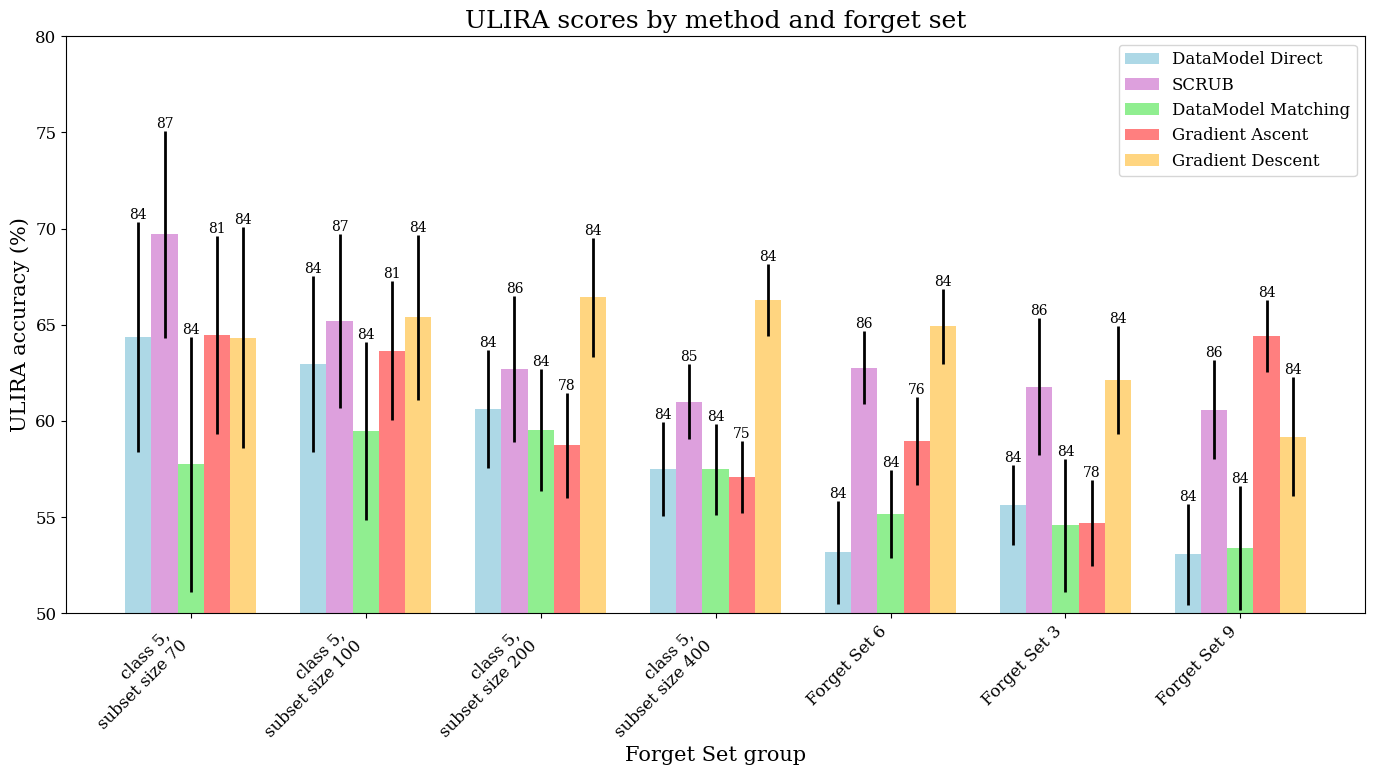}%
    \caption{Results for \EfficientULIRA run with parameters specified in \ref{appendix:ulira}. The validation accuracies are presented above the bar for the each method-forget set pairing. Observe that Gradient Ascent method is able to achieve low \ulira scores, but does so at the cost of the validation accuracy of the unlearned model.}
            \label{fig:ULIRA-results-all}
\end{figure}

\newpage
\section{Additional results}

\subsection{Hyperparameter sensitivity of unlearning algorithms} \label{appendix:hyperparam-sensitivity}
We compare the stability of SCRUB, a gradient-ascent based algorithm, and Oracle Matching.
In particular, we analyze their sensitivity to varying the learning rate in Figure \ref{fig:learning_rate_sensitivity}, where  \klom scores are plotted as a function of learning rate.\footnote{We show the 85th percentile for KLOM scores, because SCRUB generally performs significantly worse than Oracle Matching, and we found that for higher percentiles, the effect of learning rate was weaker because SCRUB never achieved particularly good KLOM scores in the first place.}

\begin{figure}[H]
    \centering
    \includegraphics[width=1.\linewidth]{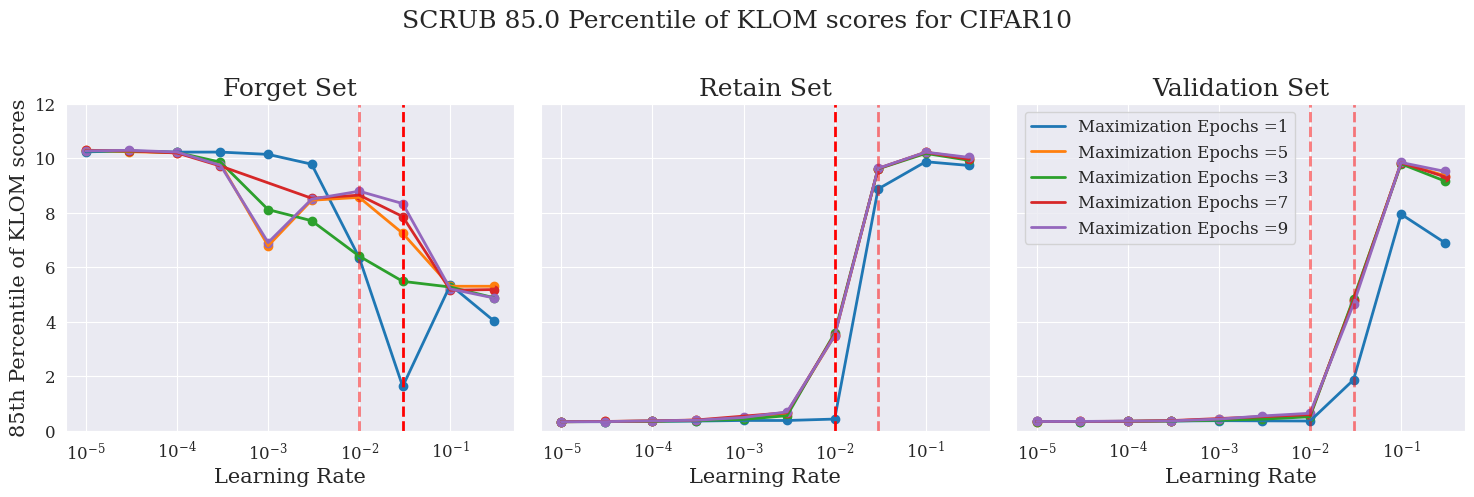}
    \includegraphics[width=1.\linewidth]{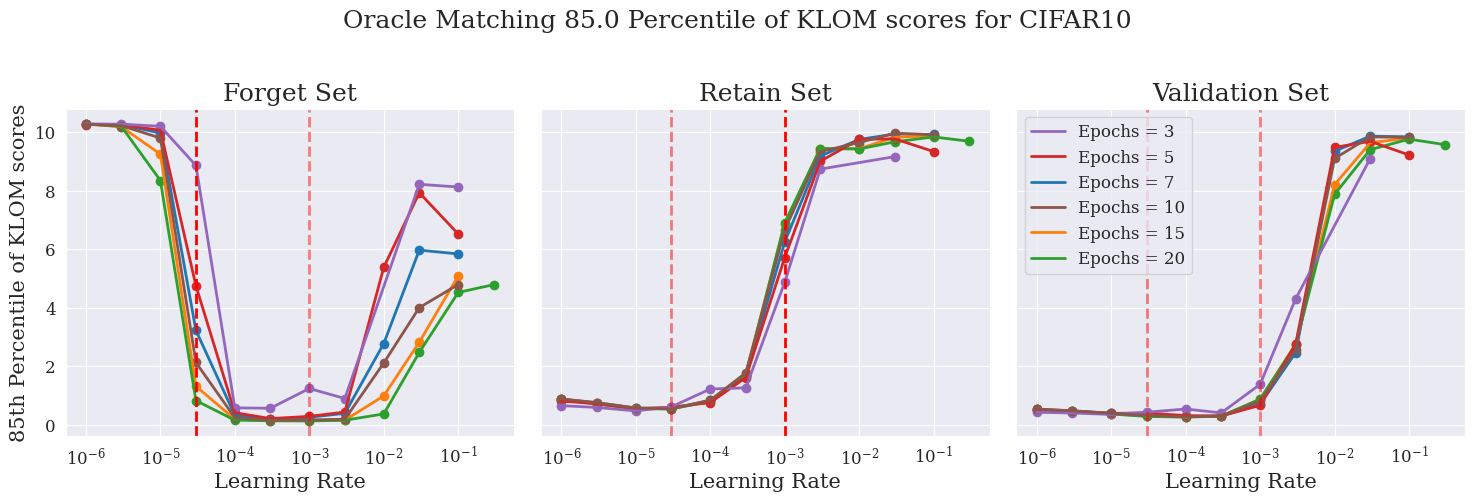}    

    \caption{\small{KLOM sensitivity to learning rate. SCRUB is trained for 10 epochs, with the first ``Maximization Epochs'' used for Gradient Ascent; Oracle Matching is trained for `Epochs', both demarcated in the legend. Both algorithms are evaluated on Forget Set 5, which is a non-random forget set of size 100. The red lines denotes the learning rate at which the KLOM score drops below 6 (an arbitrarily chosen threshold which we deemed as the largest value approaching a ``reasonable'' unlearning). Left red line indicates the largest forget that achieves a reasonable unlearning for Forget points, and the right red line indicates the smallest learning rate that achieves a reasonable unlearning for Validation and Retain points. The line is darker if that subset is the limiting factor for the thresholding.
    }}
    \label{fig:learning_rate_sensitivity}
\end{figure}

The key takeaway is that the range of suitable learning rates for SCRUB is much smaller (less than an order of magnitude) than that for Oracle Matching (which also performs significantly better). 
Note that because this is a log-scale, the range for the SCRUB learning rates is larger in absolute magnitude; however, the {\em relative} range of stable learning rates (between two dotted red lines) is much larger for Oracle Matching.

\subsection{Per-sample unlearning over time} \label{appendix:unlearning-over-training}

In \Cref{fig:unlearning-over-time-detailed}, we present the rate of unlearning individual data points as a function of epochs in the unlearning algorithm for both Oracle Matching and Datamodel matching, which we contrast with SCRUB (cf. Figure \ref{fig:individual_point_unlearning__scrub}).

\begin{figure}[!hb]
    \centering
    \includegraphics[width=0.95\linewidth]{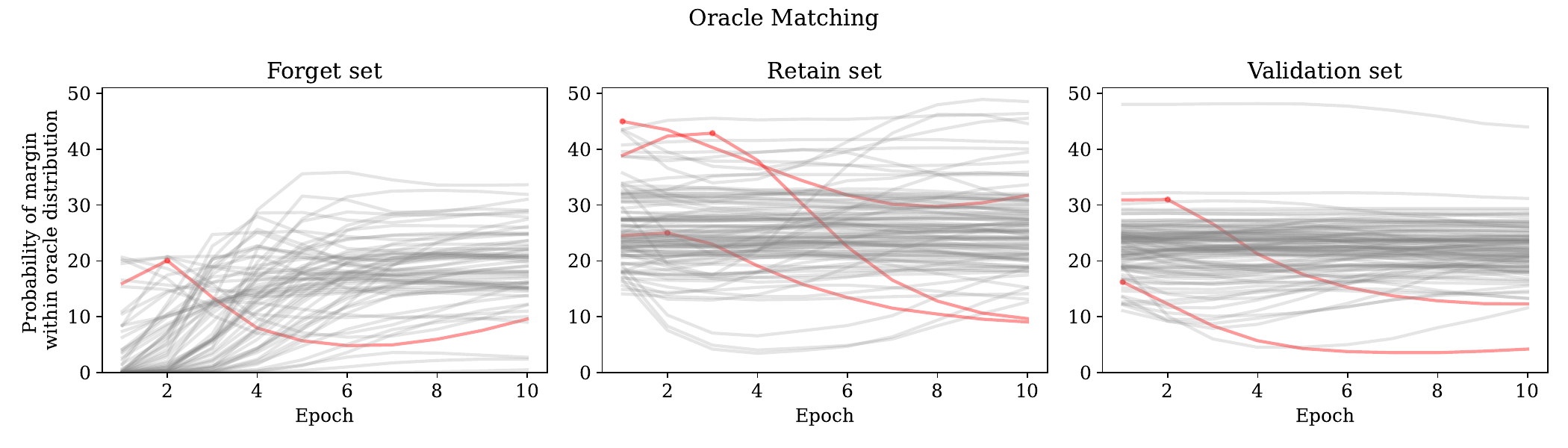}
    \includegraphics[width=0.95\linewidth]{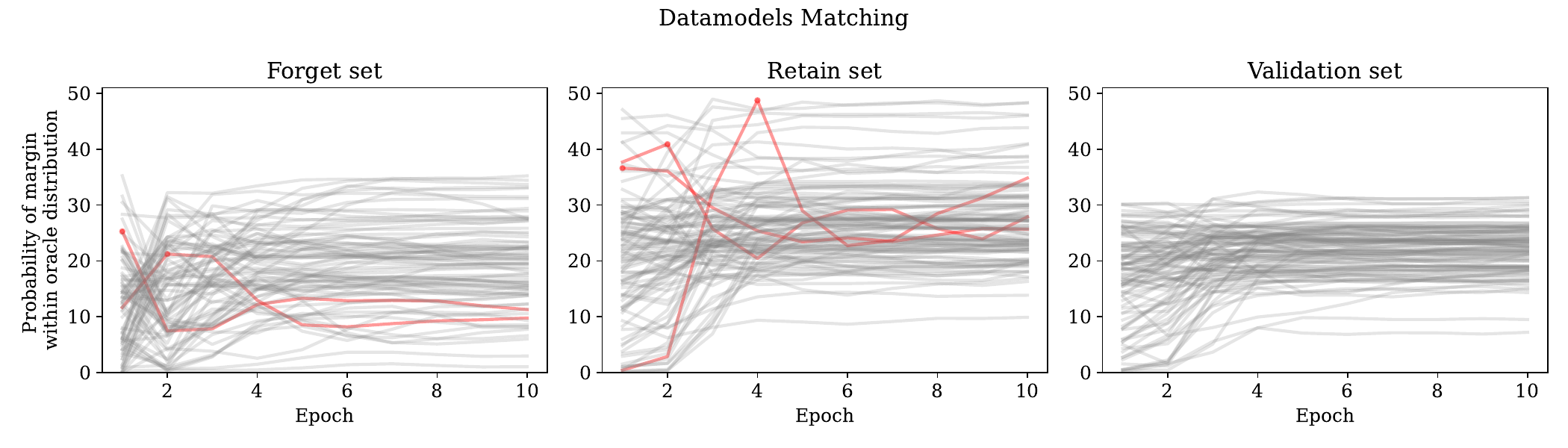}
    
        \caption{Plotting unlearning quality over time for Oracle Matching and Datamodels matching. We observe that different points unlearn at different rates. However, unlike for gradient ascent based methods like SCRUB, for gradient descent based methods like Datamodels matching and oracle matching, one a data point is unlearned, it tends to remain unlearned. The red lines highlight examples whose final quality of unlearning is more than 10\% worse than their maximum unlearning quality.}
    \label{fig:unlearning-over-time-detailed}
\end{figure}

\clearpage

\clearpage

\subsection{Full \klom evaluation} \label{appendix:klom_over_training_time}

We present the full \klom vs. time evaluations for all forget sets in \Cref{fig:klom-full,fig:klom-full-living}.
The pareto frontier (1st column) is computed separately for each forget set. This is because we observe that unlearning is quite model-specific and forget-specific; as such, it is reasonable that a practitioner should have a different set of hyperparameters for each unlearning depending on the forget set. Thus, having the pareto-frontier be specified per-forget-set gives each unlearning method the best opportunity to perform in that setting.
Observe that Gradient Descent tends to perform very well for the retain and validation sets, but leaves the forget set relatively unchanged. It is possible that unlearning the forget can be more effective with heavier use of regularization such as weight-decay. %

\begin{figure}[!h]
    \centering
    \includegraphics[width=0.8\linewidth]{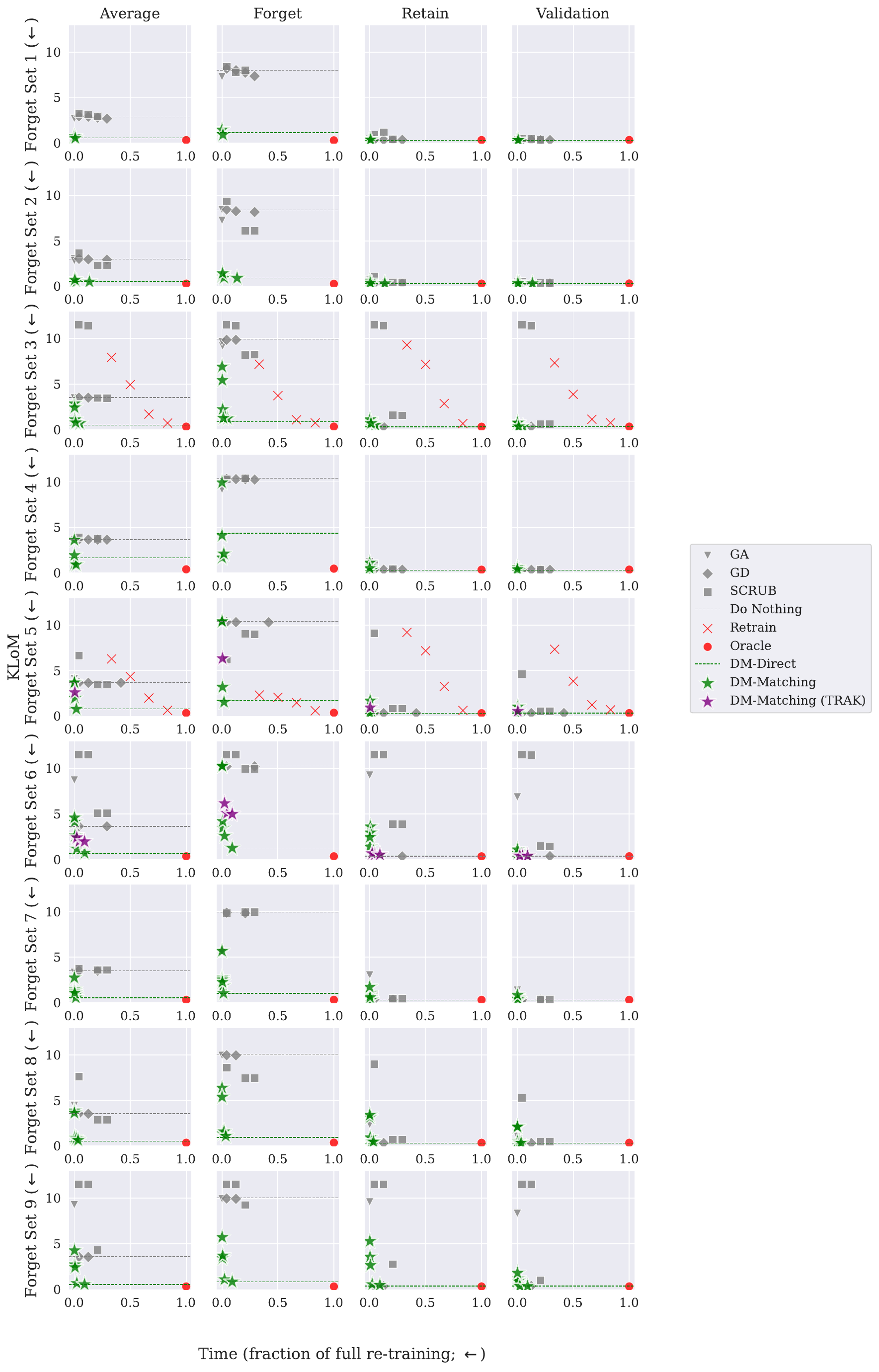}
    \caption{\klom~results for all forget sets 1-9 on CIFAR-10. The pareto frontier for each method is computed based on the 1st column (Average).}
    \label{fig:klom-full}
\end{figure}

\begin{figure}[!h]
    \centering
    \includegraphics[width=0.9\linewidth]{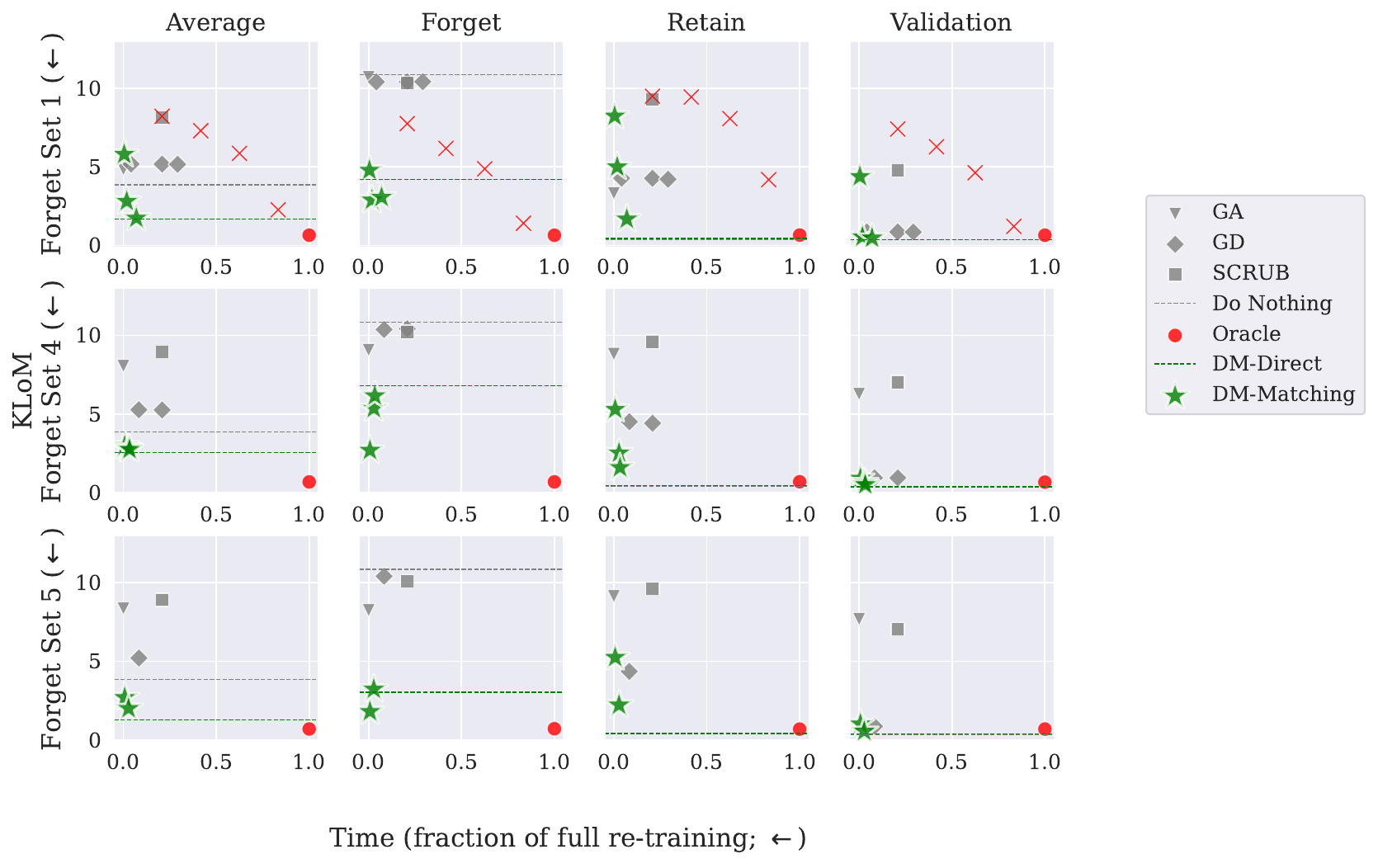}
    \caption{\klom~results for all forget sets 1-3 on Living-17. The pareto frontier for each method is computed based on the 1st column (Average).}
    \label{fig:klom-full-living}
\end{figure}

\subsection{Model accuracy after unlearning} \label{appendix:model-accuracy}

For some measures of unlearning it is possible to achieve high scores, while at the same time harming overall performance on the target task. Our measure of success, \klom, does not suffer from this problem, as a machine unlearning algorithm is measured on its ability to match the outputs of the oracle model. Hence, if an unlearning algorithm achieves a low \klom score that the oracle model has high accuracy, then the unlearned model must also have high accuracy by definition.  However, for completeness, we include a table of the model accuracies for each unlearning method (evaluated using the hyperparameters that achieves the lowest \klom~99th-percentile score on the forget set).

\begin{table}[ht]
\centering
\begin{tabular}{|c|c|c|c|c|c|c|}
\hline
    \textbf{Index} & \textbf{GA} & \textbf{GD} & \textbf{Scrub} & \textbf{Retrain an Oracle} & \textbf{DM-Direct} & \textbf{DM-Matching} \\
    \hline
    1 & 88.34 & 88.44 & 88.29 & 88.48 & 89.97 & 89.90 \\
    2 & 88.48 & 88.49 & 88.36 & 88.51 & 89.98 & 89.70 \\
    3 & 88.25 & 88.48 & 87.55 & 88.25 & 89.96 & 89.61 \\
    4 & 87.85 & 88.44 & 88.26 & 88.53 & 89.95 & 89.46 \\
    5 & 88.07 & 88.48 & 81.99 & 88.51 & 89.93 & 89.86 \\
    6 & 84.20 & 88.47 & 85.65 & 88.42 & 89.86 & 89.47 \\
    7 & 87.98 & 88.44 & 88.26 & 88.49 & 89.98 & 89.88 \\
    8 & 88.30 & 88.49 & 81.62 & 88.43 & 90.03 & 89.93 \\
    9 & 10.92 & 88.48 & 85.36 & 88.46 & 89.99 & 89.89 \\
    \hline
    \end{tabular}
    \caption{Model Accuracy Results}
    \label{tab:model-accuracy}
\end{table}

\end{document}